\documentclass{article}

\usepackage[numbers]{natbib}

\usepackage[final]{neurips_2022}

\usepackage[utf8]{inputenc} 
\usepackage[T1]{fontenc}    
\usepackage{hyperref}       
\usepackage{url}            
\usepackage{booktabs}       
\usepackage{amsfonts}       
\usepackage{nicefrac}       
\usepackage{microtype}      
\usepackage{xcolor}         
\usepackage{amsmath}
\usepackage{graphicx}
\usepackage{amssymb}
\usepackage{mathtools}
\usepackage{amsthm}
\usepackage{bbm}
\usepackage{wrapfig}
\usepackage{comment}
\usepackage{caption}
\usepackage{subcaption}
\theoremstyle{plain}
\newtheorem{theorem}{Theorem}[section]
\newtheorem{proposition}[theorem]{Proposition}
\newtheorem{lemma}[theorem]{Lemma}
\newtheorem{corollary}[theorem]{Corollary}
\theoremstyle{definition}
\newtheorem{definition}[theorem]{Definition}

\newtheorem{remark}[theorem]{Remark}

\usepackage[Symbolsmallscale]{upgreek}

\newcommand{\Char}{\mathbbm{1}}

\newcommand{\cG}{\mathcal{G}}

\newcommand{\cN}{\mathcal{N}}
\newcommand{\cE}{\mathcal{E}}
\newcommand{\cS}{\mathcal{S}}
\newcommand{\cP}{\mathcal{P}}
\newcommand{\cV}{\mathcal{V}}
\newcommand{\cX}{\mathcal{X}}

\newcommand{\N}{\mathbb{N}}
\newcommand{\Z}{\mathbb{Z}}

\newcommand{\R}{\mathbb{R}}
\newcommand{\E}{\mathbb{E}}

\newcommand\va{\mathbf{a}}
\newcommand\vb{\mathbf{b}}

\newcommand\vd{\mathbf{d}}
\newcommand\ve{\mathbf{e}}
\newcommand\vf{\mathbf{f}}
\newcommand\vg{\mathbf{g}}

\newcommand\vu{\mathbf{u}}
\newcommand\vv{\mathbf{v}}
\newcommand\vw{\mathbf{w}}
\newcommand\vx{\mathbf{x}}

\newcommand\vF{\mathbf{F}}
\newcommand\vzero{\mathbf{0}}
\newcommand\vone{\mathbf{1}}

\newcommand\mA{\mathbf{A}}
\newcommand\mB{\mathbf{B}}
\newcommand\mC{\mathbf{C}}
\newcommand\mD{\mathbf{D}}

\newcommand\mM{\mathbf{M}}
\newcommand\mN{\mathbf{N}}

\newcommand\mI{\mathbf{I}}
\newcommand\mJ{\mathbf{J}}
\newcommand\mK{\mathbf{K}}
\newcommand\mQ{\mathbf{Q}}
\newcommand\mP{\mathbf{P}}

\newcommand\mV{\mathbf{V}}

\newcommand\mU{\mathbf{U}}

\newcommand\mSigma{\boldsymbol{\Sigma}}

\usepackage{mathtools}

\newcommand\vpi{\boldsymbol{\pi}}

\title{Efficiency Ordering of Stochastic Gradient Descent}

\author{%
  Jie Hu$^1$\thanks{Equal contributors.\vspace{-5mm}} ~~~~~~~ Vishwaraj Doshi$^2$$^*$ ~~~~~~~ Do Young Eun$^1$\\
  $^1$ Department of Electrical and Computer Engineering, North Carolina State University\\
  $^2$ Data Science and Advanced Analytics, IQVIA\\
  \texttt{jhu29@ncsu.edu} ~~~~ \texttt{vishwaraj.doshi@iqvia.com} ~~~~ \texttt{dyeun@ncsu.edu} \\
}

\begin{document}

\maketitle

\begin{abstract}
We consider the stochastic gradient descent (SGD) algorithm driven by a general stochastic sequence, including \textit{i.i.d} noise and random walk on an arbitrary graph, among others; and analyze it in the asymptotic sense. Specifically, we employ the notion of `efficiency ordering', a well-analyzed tool for comparing the performance of Markov Chain Monte Carlo (MCMC) samplers, for SGD algorithms in the form of Loewner ordering of covariance matrices associated with the scaled iterate errors in the long term. Using this ordering, we show that input sequences that are more efficient for MCMC sampling also lead to smaller covariance of the errors for SGD algorithms in the limit. This also suggests that an arbitrarily weighted MSE of SGD iterates in the limit becomes smaller when driven by more efficient chains. Our finding is of particular interest in applications such as decentralized optimization and swarm learning, where SGD is implemented in a random walk fashion on the underlying communication graph for cost issues and/or data privacy. We demonstrate how certain non-Markovian processes, for which typical mixing-time based non-asymptotic bounds are intractable, can outperform their Markovian counterparts in the sense of efficiency ordering for SGD. We show the utility of our method by applying it to gradient descent with shuffling and mini-batch gradient descent, reaffirming key results from existing literature under a unified framework. Empirically, we also observe efficiency ordering for variants of SGD such as accelerated SGD and Adam, open up the possibility of extending our notion of efficiency ordering to a broader family of stochastic optimization algorithms.
\end{abstract}

\section{Introduction}
\vspace{-1mm}
Stochastic gradient descent (SGD) is widely used in machine learning, signal processing and other engineering fields to solve the optimization problem
\begin{equation}\label{eqn:objective_function}
    \theta^* = \arg\min_{\theta \in \Theta} \left\{f(\theta) \triangleq \frac{1}{n}\sum_{i=1}^n F(\theta,i)\right\},
\end{equation} 
where $\Theta \subset \R^d$ is some closed and convex set, and $F(\cdot,i):\R^d \to \R$ for $i \in [n]\triangleq \{1,\cdots,n\}$ are smooth functions on $\Theta$, not necessarily convex, such that their summation $f:\R^d \to \R$ exhibits a minimizer $\theta^* \in \Theta$ satisfying $\nabla \! f(\theta^*) \!=\! 0$. The update rule of the iterative SGD scheme is of the form
\begin{equation}\label{eqn:update_rule}
    \theta_{t+1} = \text{Proj}_{\Theta}\left(\theta_t -\gamma_{t+1} \nabla_\theta F\left(\theta_t, X_{t+1}\right)\right), 
\end{equation}
where $\gamma_t$ is the step size that can be constant or diminishing as $t\!\to\!\infty$, $\text{Proj}_{\Theta}$ is a projection operator onto the constraint set $\Theta$, and $\{X_t\}_{t \geq 0}$ is some sequence taking values in $[n]$. This sequence is often generated in a stochastic manner, and samples can be drawn from temporally independent and identically distributed \textit{(i.i.d)} random variables that are either uniformly distributed over $[n]$ \citep{robbins1951stochastic,nemirovskij1983problem,bottou2010large}, or leverage importance sampling techniques for variance reduction \citep{namkoong2017adaptive,borsos2019online,el2020adaptive}. $\{X_t\}_{t\geq 0}$ can also be constructed by repeatedly shuffling over all possible states without repetition,\footnote{One complete pass over the entire set $[n]$ is typically called an epoch. Shuffling can refer to passing over $[n]$ in the same order for every epoch (single shuffling), or in a random order (random shuffling).} leading to faster convergence than stochastic counterparts drawing \textit{i.i.d} samples from $[n]$ \citep{safran2020good,ahn2020sgdshuffling,gurbuzbalaban2021random,yun21can}.

\textbf{Random Walk Stochastic Gradient Descent (RWSGD):}
Some applications observe restricted access to the state space, such as decentralized optimization  \citep{scaman2017optimal,xin2020general,mao2020walkman}, where communication occurs between nodes in a network to collaboratively solve the optimization problem \eqref{eqn:objective_function}. For instance, disease classification in confidential clinical swarm learning \citep{warnat2021swarm} considers peer-to-peer networks due to the highly private nature of medical data. In such a setting, the random sequence $\{X_t\}_{t\geq 0}$ is usually realized as a Markov chain on a general graph $\cG(\cV,\cE)$ that only samples local gradients of the nodes in $\cV \triangleq [n]$ and traverses the network via edges connecting them without divulging the update history or its own gradient. The randomness of the communication path ensures that the compromised node can not easily leak the data of its neighbors \citep{mao2020walkman}. 

Apart from the privacy concern, such dynamics are also employed in swarm learning/optimization in robotics \citep{brambilla2013swarm} and wireless sensor networks \citep{lesser2003distributed} due to their low communication cost and asynchronous nature. The need for data privacy and demand for communication-efficient algorithms for decentralized optimization has spurred the study of RWSGD algorithms in recent years \citep{ram2009incremental,johansson2010randomized,sun2018markov}, with the underlying Markov chain in the form of Metropolis-Hasting random walk (MHRW) \citep{metropolis1953equation}.

\textbf{Common analytical approach - Finite time bounds based on mixing time:} Most of the existing works analyzing iteration \eqref{eqn:update_rule} provide so-called finite-time upper bounds on expected error in either the objective function $\E[f(\tilde \theta_t)-f(\theta^*)]$, 
where $\tilde \theta_t$ is some weighted average of the iterates, 
or its gradient $\E[|\!|\nabla f(\theta_t)|\!|_2^2]$; and are used to infer the convergence rate of the iterate sequence \citep{ram2009incremental,duchi2012ergodic,sun2018markov}.
For diminishing step sizes $\gamma_t = t^{-\alpha}$ with $\alpha \in (0.5,1)$,\footnote{We only need the step size to be $O(t^{-\alpha})$, but we omit the $O(\cdot)$ notation for simplicity. We also consider a slightly more general case, allowing for $\alpha = 1$ as well.} the upper bound on $\E[ \| \nabla f(\theta_t) \|_2^2]$ reads as 
\begin{equation}\label{eqn:slem_bound}
    \E[ \| \nabla f(\theta_t) \|_2^2] \leq O\left(\frac{\max\left\{M,1/\log(1/\beta)\right\}}{t^{1-\alpha}}\right),
\end{equation}
and a similar form for  $\E[f(\tilde \theta_t)-f(\theta^*)]$ as well \citep{sun2018markov,doan2020finite,ayache2021private}. Here, $\beta \in (0,1)$ is the second largest eigenvalue modulus (SLEM) of the underlying Markov chain's transition matrix and is related to its mixing time property, since smaller SLEM leads to faster mixing of the Markov chain \citep{bremaud2013markov,levin2017markov}. On the other hand, $M>0$ is usually a quantity proportional to the local gradients evaluated at the minimizer, or their upper bound. Both the gradient information and the mixing time play a key role in quantifying the convergence rate derived from this upper bound, and the mixing time is especially important since it hints that convergence rate of the SGD algorithm can potentially be accelerated using faster mixing Markov chains for the input driving sequence. It has also been noted that the inherent correlation of the underlying random walk has to be addressed in any analysis concerning Markov-chain-driven gradient descent \citep{sun2018markov}. The mixing time technique, by capturing the rate at which the chain converges to its stationary distribution \citep{bremaud2013markov,levin2017markov}, is one way of doing so. 

\textbf{Alternative approach - Asymptotic analysis and efficiency ordering:} In addition to the aforementioned mixing time, another widely used metric for characterizing the second order properties of Markov chains is the asymptotic variance (AV). For any scalar valued function $g:[n]\to\R$, the estimator $\hat \mu_t(g) \triangleq \frac{1}{t}\sum_{i=1}^t g(X_i)$, associated with an irreducible Markov chain $\{X_t\}_{t\geq0}$ with stationary distribution $\vpi$, is the average of the samples of $g(\cdot)$ obtained along the chain's sample path up to time $t > 0$. The AV of the Markov chain, denoted by $\sigma^2_X(g)$, is then defined as the the limiting variance of the estimator; that is, 
\begin{equation}\label{eqn:AV scalar}
    \sigma^2_X(g) \triangleq \lim_{t\to\infty}t\cdot\text{Var}(\hat \mu_t(g)).
\end{equation}
For all functions $g(\cdot)$ satisfying $\E_{\vpi}(g^2) < \infty$, the AV is associated with the Central Limit Theorem (CLT) for any Markovian kernel on a finite state space, as the variance of the normally distributed estimates in the limit \citep{roberts2004general,jones2004markov,bremaud2013markov}. More formally, we have
\begin{equation}\label{eqn:CLT scalar mcmc}
    \sqrt{t}\cdot\left[\hat \mu_t (g) - \E_{\vpi}(g)\right] \xrightarrow[t \to \infty]{dist} \cN(0,\sigma^2_X(g)).
\end{equation}
A smaller AV means that fewer samples are required \textit{post} mixing of the chain\footnote{Achieved by employing a burn-in period to get rid of the correlation with the initial state \citep{gjoka2011practical}.} in order to obtain a desired accuracy - in some sense quantifying the chain's \textit{efficiency}.

Both the AV and the mixing time of a Markov chain are very strongly related concepts\footnote{For reversible Markov chains, the AV can be written explicitly as an increasing function of \emph{every} eigenvalue of the transition matrix \citep{ bremaud2013markov}, while the mixing time is related to the SLEM as mentioned earlier.}. In fact, the AV has an upper bound in terms of the SLEM, which decreases as the SLEM gets smaller (chain mixes faster) \citep{mira2001ordering}. However, an ordering of the SLEM between two Markov chains does not imply an ordering of their AV, as we shall demonstrate later in Section \ref{section:application} for a special case. Both of these second-order properties therefore lead to different notions of optimality; and the comparison of two chains based on their AV leads to the concept of \textit{efficiency ordering} \citep{mira2001ordering}, where we say that a chain is more efficient than the other if it has a smaller AV, uniformly over all functions $g:[n]\to \R$.

\begin{wrapfigure}{r}{0.46\textwidth}
  \begin{center} \vspace{-7mm}
  \includegraphics[width = 0.46\textwidth]{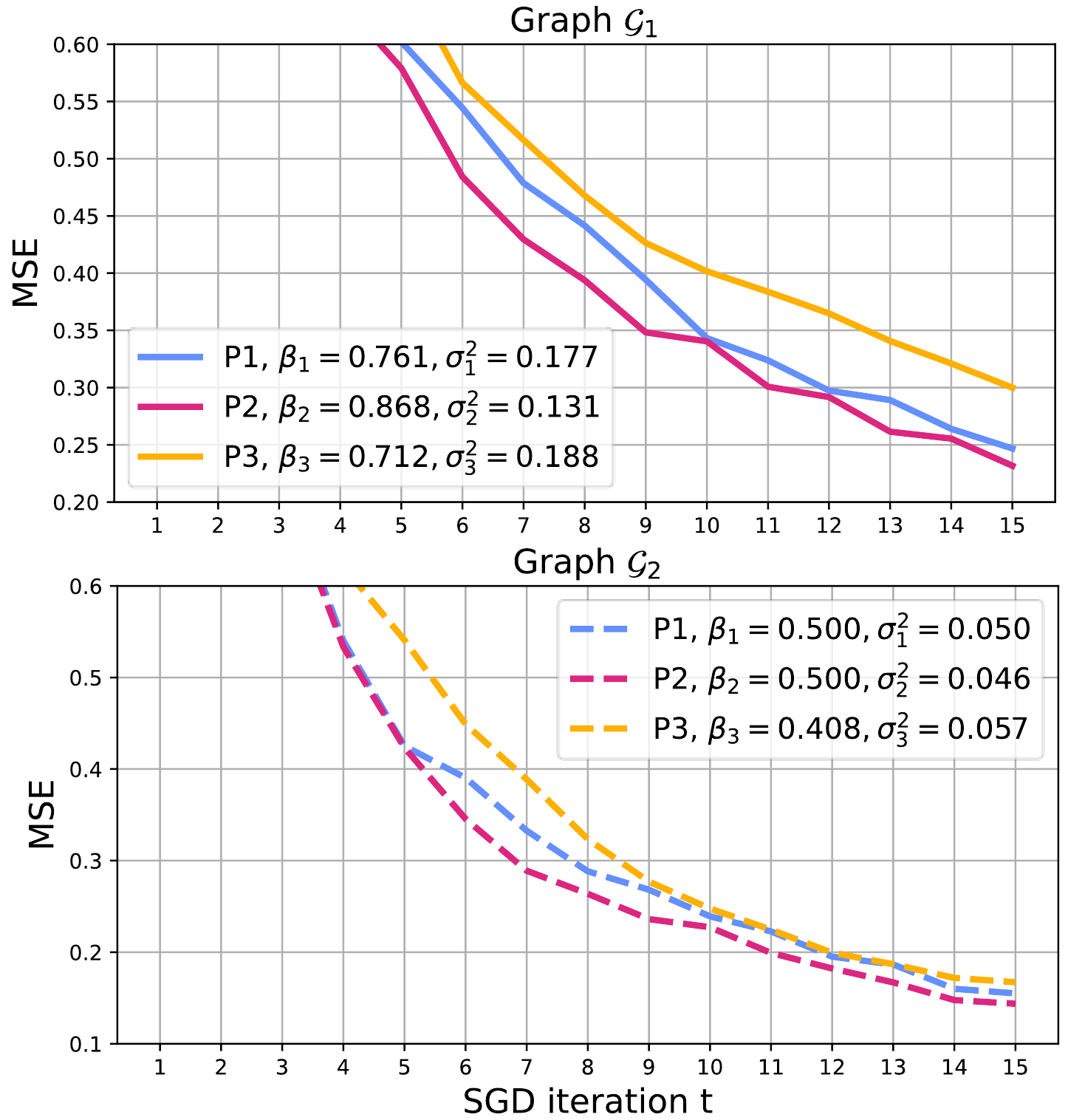}\vspace{-2mm}
  \end{center}
  \caption{Comparison of MHRW (P1), Modified-MHRW (P2) and FMMC (P3) as stochastic inputs for RWSGD on two different graphs $\cG_1$ and $\cG_2$.}\vspace{-3mm}\label{fig:fmmc}
\end{wrapfigure}
As mentioned earlier, the common intuition asserted by finite time bounds such as \eqref{eqn:slem_bound} is that Markov chains with smaller SLEM lead to faster convergence of the SGD iteration \eqref{eqn:update_rule} to the minimizer \citep{sun2018markov,ayache2021private}. We put this logic to test by simulating the RWSGD algorithm with three different reversible Markov chains (w.r.t. uniform stationary distribution) as the stochastic inputs - the MHRW, a modification of MHRW, which is also shown in Appendix \ref{app:simulation} to be more \textit{efficient} than MHRW, and the so-called 'fastest mixing Markov chain' (FMMC) as defined in \citep{boyd2004fastest} as the Markov chain obtained by minimizing the SLEM over the entire class of reversible chains for a given graph topology. We employ RWSGD to minimize a quadratic objective function for two underlying graphs. The exact details of the setup are deferred to Appendix \ref{app:simulation}, and our numerical results in Figure \ref{fig:fmmc} show that even though the FMMC is theoretically guaranteed to have the smallest SLEM ($\beta_i$ for $i \!\in\! \{1,2,3\}$) of the three reversible chains simulated, it is the worst performing one with largest mean square error (MSE). Although MHRW and Modified-MHRW share the same SLEM in the lower plot of Figure \ref{fig:fmmc}, they still have performance differences. This is contradictory to the intuition derived from \eqref{eqn:slem_bound}, and could be attributed to the finite time results providing upper bounds for \textit{all} times $t>0$, which may therefore not necessarily be tight. On the other hand, the performance of the chains seem to be ordered according to their AV ($\sigma_i^2$ for $i \!\in\! \{1,2,3\}$) evaluated for a test function. This lends credence to developing techniques based on AV, for judging the performance of different stochastic inputs for SGD, as possible alternatives to using SLEM as the sole performance metric.

The asymptotic variance also appears in the CLT for stochastic approximation (SA) algorithms \citep{benveniste2012adaptive,chen2020explicit}, though this time not directly as the variance in the limit, but as a component of the limiting covariance matrix of the scaled iterate errors. Recent works \citep{chen2020explicit, mou2020linear} point out that the covariance matrix itself is of special interest, and typically contains more information than the non-asymptotic MSE bounds \citep{mou2020linear}. In the sense of SGD algorithms, we will show in Section \ref{Section:Results random walk} that it embeds explicit information of the exact vector-valued gradient evaluated at the optimizer as well as the entire spectrum of the transition matrix; as opposed to the upper bound $M$ of the gradients and only the second largest eigenvalue modulus commonly found in mixing time based non-asymptotic bounds. It has been suggested \citep{chen2021accelerating,devraj2021qlearning}, and also proved for the special case of linear SA \citep{chen2020explicit}, that the covariance matrix emerging out of the CLT dominates as a leading term of the finite-time MSE bounds. This also holds true for finite-time bounds on weighted MSE for any preferred weight; the weighted MSE being utilized in fields such as wireless MIMO \citep{wang2013wireless} and process optimization \citep{gomes2013weighted}. Overall, while finite-time bounds have enjoyed great success in the literature, the potential for performance gains out of the asymptotic analysis of SGD algorithms have remained largely unexplored.

\textbf{Contributions:} We employ asymptotic analysis to propose a general framework that offers seamless connection between AV in the MCMC literature with efficiency ordering and covariance matrix in the SGD algorithms. Our framework can be used to design different random walk variants and also to systematically compare the existing sampling methods in the SGD iteration \eqref{eqn:update_rule} with diminishing step size, not just limited to random walks. In particular, we show that any two random walks following an efficiency ordering have their covariance matrices Loewner ordered, including \textit{non-Markovian} stochastic processes versus its Markovian counterpart, which defies any mixing-time (SLEM) based analysis.  Such ordering can be harnessed into improving the accuracy of SGD iterates, which implies a reduction in the weighted MSE with arbitrary weights. Moreover, via a specific augmentation of the state space, we are able to analyze SGD for both single and random shuffling and show the efficiency of shuffling over \textit{i.i.d} sampling for a set of objective functions that may not satisfy `Polyak-Łojasiewicz inequality'. We further extend such comparison to mini-batch SGD algorithms. Lastly, we present numerical results where the efficiency ordering via asymptotic analysis tends to hold over all time periods and input sequences with higher efficiency have smaller errors in SGD.

\section{Modeling setup} \label{Section:Model description}

\textbf{Basic notations:}
We use lower case, bold-faced letters to denote vectors ($\vv \in \R^d$), and use upper case, bold-faced letters to denote matrices ($\mM \in \R^{d\times d}$). $\|\cdot\|_2$ denotes the $l^2$ norm for vectors or $2$-norm for matrices. We use $\nabla \vf(\cdot)$ as Jacobian matrix of vector-valued function $\vf(\cdot)$, and $\nabla^2 g(\cdot)$ as Hessian matrix of scalar-valued function $g(\cdot)$. We let $\nabla_{\theta} g(\theta,X)$ be the gradient of the scalar-valued function $g(\theta,X)$ with respect to $\theta$ and omit the subscript $\theta$ for simplicity. Loewner ordering of matrices is denote by `$\leq_L$' such that $\mA \leq_L \mB \Longleftrightarrow \vx^T(\mA-\mB)\vx \leq 0$ for any $\vx \in \R^d$. The term $Tr(\mA)$ denotes the trace of matrix $\mA$, and let $\mathbbm{1}_{\{\cdot\}}$ be the indicator function. We write $\cN(0,\mV)$ to represent a multivariate Gaussian distribution with zero mean and covariance matrix $\mV$. For a connected and undirected graph $\cG(\cV,\cE)$ with node set $\cV$ and edge set $\cE$, we use $N(i)$ for the set of neighbors of node $i\in\cV$ and $\vd\triangleq [d_1,d_2,\cdots,d_n]^T$ for the degree vector where $d_i=|N(i)|$. 

\textbf{SGD algorithm with arbitrary input sequence:}
We consider random walks $\{X_t\}_{t\geq 0}$ for which the limit $\pi_{i} \triangleq \lim_{t \to \infty} \frac{1}{t}\sum_{k=1}^t \Char_{\{X_k=i\}}$
exists almost surely and is positive for all $i \in [n]$, with $\vpi = [\pi_i]_{i \in [n]}$ denoting the limiting or \textit{stationary} distribution. This is trivially satisfied via strong law of large numbers \citep{durrett2019probability} when $X_t$ for each $t>0$ are \textit{i.i.d} random variables with distribution $\vpi$ over $[n]$, and via the ergodic theorem \citep{bremaud2013markov} when $\{X_t\}_{t \geq 0}$ is an irreducible, aperiodic and positive recurrent (ergodic) Markov chain. Note however that this way of defining the stationary distribution $\vpi$ allows for the input sequence $\{X_t\}_{t\geq 0}$ to be more general, possibly being non-Markov on $[n]$. Then, we can use $\vpi$ to rewrite the objective in \eqref{eqn:objective_function} as 
\begin{equation}\label{eqn:general_objective_function}
    f(\theta) = \frac{1}{n}\!\sum_{i=1}^n F(\theta,i) = \E_{X \sim \vpi}\left[G(\theta,X)\right],
\end{equation}
where function $G(\theta,i) \triangleq \frac{1}{n\pi_i} F(\theta,i)$ for any $\theta\!\in\!\Theta,i \!\in\! [n]$. The generalized update rule then becomes
\vspace{-2mm}
\begin{equation}\label{eqn:update_rule_general}
    \theta_{t+1} = \text{Proj}_{\Theta}\left(\theta_t -\gamma_{t+1} \nabla G\left(\theta_t, X_{t+1}\right)\right).
\end{equation}
This change of notation allows us to consider input sequences having possibly non-uniform stationary distributions, and is a version of importance sampling for RWSGD schemes, as in \citep{ayache2021private}. For example, the iteration \eqref{eqn:update_rule_general} with the input sequence generated from a MHRW with uniform target distribution $\vpi = \vone/n$ will reduce down to \eqref{eqn:update_rule} with $G(\theta,i) = F(\theta,i)$ for all $\theta \in \Theta$, $i\in[n]$. If the input sequence is instead a simple random walk on a connected graph $\cG(\cV,\cE)$ with $\cV=[n]$, we have $\vpi \propto \vd$, and $G(\theta,i) = \frac{\vone^T\vd}{n d_i} F(\theta, i)$ for all $\theta \in \Theta$, $i\in\cV$.\footnote{In practice, knowing $\pi_i$ up to a multiplicative constant is enough to converge to the optimal point.} 

\textbf{Asymptotic covariance matrix.}
We now quickly review the multivariate CLT for Markov chains, since it is a natural way to introduce the \textit{asymptotic covariance} matrix, used heavily throughout the paper. For any finite, irreducible Markov chain $\{X_t\}_{t \geq 0}$ with stationary distribution $\vpi$, its \textit{estimator} is defined as $\hat \mu_t(\vg) \triangleq \frac{1}{t}\sum_{k = 1}^{t} \vg(X_k)$ for any vector-valued function $\vg:[n]\to\R^d$. Then, the ergodic theorem \citep{bremaud2013markov,HandbookMCMC} states that for any initial distribution and any $\vg(\cdot)$ such that $\E_{\vpi}(\vg)\textbf{} = \sum_{i \in [n]} \vg(i) \pi_i < \infty$, we have $\hat \mu_t(\vg) \xrightarrow[t \to \infty]{a.s.} \E_{\vpi}(\vg)$.
Similarly to the asymptotic variance $\sigma^2_X(g)$ for a scalar-valued function $g(\cdot)$, we can also define the \textit{asymptotic covariance} matrix $\mSigma_X(\vg)$ for vector-valued function $\vg(\cdot)$,
\begin{equation}\label{eqn:asymptotic covariance matrix}
    \mSigma_X(\vg) \triangleq \lim_{t\to\infty}t\cdot\text{Var}(\hat \mu_t(\vg)) \\
    = \lim_{t\to\infty}\frac{1}{t}\cdot\E\left\{\Delta_t\Delta_t^T\right\},
\end{equation}
where $\Delta_t \triangleq \sum_{s=1}^t(\vg(X_s) - \E_{\vpi}(\vg))$. The associated multivariate CLT is then given as follows.
\begin{theorem}[Chapter 1 \citep{HandbookMCMC}]\label{Thm:CLT multivariate mcmc}
For any function $\vg:[n]\to\R^d$ that satisfies $\E_{\vpi}(\vg^2)<\infty$, we have
\begin{equation*}
    \pushQED{\qed}\sqrt{t}\cdot\left[\hat \mu_t (\vg) - \E_{\vpi}(\vg)\right] \xrightarrow[t \to \infty]{dist} \cN(0,\mSigma_X(\vg)).\qedhere \popQED
\end{equation*}
\end{theorem}
In the next section, we will show how the the asymptotic covariance matrix $\mSigma_X(\cdot)$ also appears as part of the CLT result for SGD algorithms.

\section{Efficiency Ordering of SGD Algorithms}\label{Section:Results random walk}
In this section, we present our main result concerning the performance comparison of different SGD algorithms to solve \eqref{eqn:objective_function}. We first begin by stating our assumptions on the objective function and the stochastic input sequence, providing a CLT result for SGD algorithms, and analyzing the covariance matrix arising therein. We then introduce the notion of \textit{efficiency ordering} of Markov chains in the context of MCMC sampling, and form the connection with covariance matrices as our main result in Theorem \ref{Thm:main_result}. 

For the rest of this section we assume that the functions $F(\cdot,i)$ (possibly non-convex), the summands of the objective function in \eqref{eqn:objective_function}, and the input process $\{X_t\}_{t\geq 0}$ for the SGD iteration \eqref{eqn:update_rule_general} satisfy:

\begin{enumerate}
    \item[(A1)] The step size is given by $\gamma_t = t^{-\alpha}$ for $\alpha \in (1/2,1]$;
    
    \item[(A2)] There exists a unique minimizer $\theta^*$ in the interior of the compact set $\Theta$ with $\nabla f(\theta^*) = 0$, and matrix $\nabla^2 f(\theta^*)$ (resp.\! $\nabla^2 f(\theta^*)-\mI/2$) is positive definite for $\alpha \!\in\! (1/2,1)$ (resp.\! $\alpha \!=\! 1$);
    
    \item[(A3)] Gradients are bounded in the compact set $\Theta$, that is, $~\sup_{\theta \in \Theta}\sup_{i \in [n]} |\!|\nabla F(\theta,i)|\!|_2 < \infty$;
    
    \item[(A4)] For every $z\in[n],\theta\in\R^d$, the solution $\Tilde{F}(\theta,z)\in\R^d$  of the Poisson equation $\Tilde{F}(\theta,z) \!-\! \E[\Tilde{F}(\theta,X_{t+1})~|~X_t=z] = \nabla F(\theta, z) \!-\! \nabla f(\theta)$ exists, and $\sup_{\theta\in\Theta,z\in[n]} \|\Tilde{F}(\theta,z)\|_2<\infty$;
    
    \item[(A5)] The functions $F(\theta,i)$ are $L$-smooth for all $i \in [n]$, that is, $\forall \theta_1,\theta_2 \in \Theta, \forall i \in [n]$, we have $\|\nabla F(\theta_1, i) - \nabla F(\theta_2, i)|\!|_2 \leq L |\!|\theta_1 - \theta_2 \|_2$.
\end{enumerate}

We then have the following CLT result for SGD algorithms.

\begin{lemma} \label{Lemma:CLT_SGD}
For iterates $\{\theta_t\}_{t \geq 0}$ of the SGD algorithm \eqref{eqn:update_rule_general} satisfying (A1)--(A5), we have
\begin{equation}\label{eqn:error}
    \theta_t \xrightarrow[t \to \infty]{a.s.} \theta^*,~~~\text{and}~~~\left(\theta_t - \theta^*\right)/\sqrt{\gamma_t} \xrightarrow[t \to \infty]{Dist} \cN(0, \mV_{\!\!X}),
\end{equation}
where covariance matrix $\mV_{\!\!X}$ is the unique solution to the Lyapunov equation $\mSigma_X \!+\! \mK\mV_{\!\!X} \!+\! \mV_{\!\!X}\mK^T\!=\! \vzero$ when $\alpha \!\in\! (0.5,1)$ (resp. $\mSigma_X \!+\! \left(\mK\!+\!\frac{\mI}{2}\right)\!\!\mV_{\!\!X} \!+\! \mV_{\!\!X}\!\!\left(\mK\!+\!\frac{\mI}{2}\right)^T \!=\! \vzero)$ when $\alpha \!=\! 1$). Here, $\mSigma_X \!\triangleq\ \mSigma_X(\nabla G(\theta^*,\cdot))$ is the asymptotic covariance matrix\footnote{We slightly abuse the notation and shorten $\mSigma_X(\nabla G(\theta^*,\cdot))$, that is, the asymptotic covariance matrix evaluated at $\nabla G(\theta^*,\cdot)$), to $\mSigma_X$ for better readability. In this paper, $\mSigma_X(\nabla G(\theta^*,\cdot))$ and $\mSigma_X$ are equivalent.} as in \eqref{eqn:asymptotic covariance matrix}, and $\mK \triangleq \nabla^2 f(\theta^*)$.

Additionally, for the averaged iterates $\{ \bar \theta_t\}_{t\geq0}$ where $\bar{\theta}_t \triangleq \frac{1}{t}\sum_{i=0}^{t-1}\theta_t$, we have
\begin{equation}\label{eqn:error_pr_averaging}
    \bar{\theta}_t \xrightarrow[t \to \infty]{a.s.} \theta^*,~~~\text{and}~~~\sqrt{t}(\bar{\theta}_t-\theta^*) \xrightarrow[t \to \infty]{Dist} \cN(0, \mV'_{\!\!X}),
\end{equation}
where $\mV'_{\!\!X} = \mK^{-1}\mSigma_X(\mK^{-1})^T$ with the same matrices $\mK$ and $\mSigma_X$ as in the non-averaged case. \qed
\end{lemma}

\begin{remark} 
Lemma \ref{Lemma:CLT_SGD} is itself a special case of the more general CLT result for SA algorithms provided in Appendix \ref{Appendix:CLT_assumptions}, and as proved in Appendix \ref{app: proof of clt lemma}. \qed
\end{remark}

\begin{remark}
While (A2) may appear to be too strict at first, it can be relaxed to the setting of the objective function $f(\cdot)$ having multiple minimizers, by leveraging more general CLT results from SA literature, such as Theorem 2.1 in \citep{fort2015central}. However, this comes at a cost of cumbersome notation, requiring conditioning of iterates converging to one of the minimizers, potentially making the mathematical parts harder to follow. We also show in Appendix \ref{app:discussion_pl_pdm} that (A2) is no stricter than the Polyak-Łojasiewicz inequality -- a popularly adopted weak assumption in recent SGD literature studying non-convex objective functions \citep{karimi2016linear,mertikopoulos2020almost,wojtowytsch2021stochastic,yun21can}. \qed
\end{remark}

\begin{remark}
Assumptions (A3) and (A5) are widely seen in the RWSGD literature \citep{ram2009incremental,johansson2010randomized,sun2018markov}, while (A4) is automatically satisfied for any ergodic Markov chain (see \citep{meyn2012markov,chen2020explicit} for details), a common assumption for the stochastic noise sequence \citep{johansson2010randomized,duchi2012ergodic,ayache2021private}. The compactness in (A3) can also be relaxed, given assumptions on the objective function in \citep{karmakar2021stochastic}, such that the estimator $\theta_t$ generated by Markov-driven sequences can still be ‘locked in’ a compact set after a sufficiently long time.
\qed
\end{remark}

Lemma \ref{Lemma:CLT_SGD} implicitly indicates that the asymptotic convergence rate (in distribution) for $\theta_t - \theta^*$ (resp. $\bar \theta_t - \theta^*$) is $O(\sqrt{\gamma_t})$ (resp. $O(1/\sqrt{t})$). While this does not necessarily translate to $O(\sqrt{\gamma_t})$ convergence rate for $\E[\|\theta_t - \theta^*\|_2]$ ($O(1/\sqrt{t})$ for $\E[\|\bar \theta_t - \theta^*\|_2]$), it has been suggested \citep{chen2021accelerating,devraj2021qlearning}, and is in fact true for cases such as quadratic objective functions since they satisfy the linear stochastic approximation in \citep{chen2020explicit}, which is of the form
\begin{equation}\label{eqn:linear_SA}
    \theta_{t+1} = \theta_t - \gamma_{t+1}(\mA\theta_t - \vb(X_{t+1})),
\end{equation}
for which the connection between finite-time MSE and covariance matrix $\mV_{\!\!X}$ has been established \citep{chen2020explicit}. This is also true for arbitrarily weighted MSE, which can be obtained as a weighted sum of diagonal entries of the covariance matrices $\mV_{\!\!X}$ and $\mV'_{\!\!X}$.

In addition to the apparent connection to MSE, the covariance matrix plays a wider role in SGD performance. Given any vector of weights $\vw\in\R^d$, from Lemma \ref{Lemma:CLT_SGD} we also have that the weighted sum of errors $\vw^T(\theta_t-\theta^*)$ converges to zero almost surely, and that $\vw^T(\theta_t-\theta^*)/\sqrt{\gamma_t} \xrightarrow[t\to \infty]{Dist} \cN(0, \vw^T\mV_{\!\!X}\vw)$. This means that, for sufficiently large $t$, we can estimate\vspace{-1mm}
\begin{equation*}
    P\left(\frac{\vw^T(\theta_t-\theta^*)}{\sqrt{\gamma_t\vw^T\mV_{\!\!X}\vw}}>\alpha\right) \approx \frac{1}{2\pi}\int_{\alpha}^{\infty} e^{-x^2/2}dx,
\end{equation*}
such that, for instance, the $95\%$ confidence interval for $\vw^T\theta_t$ is approximately $\vw^T\theta^*  \pm  2\sqrt{\gamma_t\vw^T\mV_{\!\!X}\vw}$. In other words, smaller $\vw^T\mV_{\!\!X}\vw$ leads to narrower confidence interval and higher accuracy. The form $\vw^T\mV_{\!\!X}\vw$ for any vector $\vw\in\R^d$ naturally implies that Loewner ordering should come into play when concerning the performance of SGD algorithms.

To proceed, we first employ the widely used notion of \textit{efficiency ordering} of Markov chains. The efficiency of different chains is compared by ordering them using their respective AV as follows.

\begin{definition}[\textbf{Efficiency Ordering} \citep{mira2001ordering}]\label{def:efficiency_order}
For two random walks $\{X_t\}_{t\geq 0}$ and $\{Y_t\}_{t\geq 0}$ with the same stationary distribution $\vpi$, we say $\{X_t\}_{t\geq 0}$ is more \textit{efficient} than $\{Y_t\}_{t\geq 0}$, which we write as $X \geq_E Y$, if and only if $\sigma^2_X(g) \leq \sigma^2_Y(g)$ for any $g:[n] \to \R$.\qed
\end{definition}

We are now ready to state our main result. We first extend the efficiency ordering of Markov chains by proving the equivalence of comparing their scalar-valued AVs, to comparing their asymptotic covariance matrices via Loewner ordering. We then use this extension to show that more efficient inputs $\{X_t\}_{t\geq0}$ (as in Definition \ref{def:efficiency_order}) to the SGD algorithm lead to performance improvements in the form of smaller covariance matrices in the Loewner ordering sense.

\begin{theorem}\label{Thm:main_result}
Consider the SGD iteration \eqref{eqn:update_rule_general} with two random walks $\{X_t\}_{t \geq 0}$ and $\{Y_t\}_{t \geq 0}$ as input sequences, with the same stationary distribution $\vpi$, satisfying (A1)--(A5). Then,
\begin{itemize}
    \item[(i)] $X \!\geq_E\! Y$ if and only if $\mSigma_X(\vg) \!\leq_L\! \mSigma_Y(\vg)$ for any vector-valued function $\vg$;
    \item[(ii)] If $\mathbf{\Sigma}_X(\nabla G(\theta^*,\cdot)) \!\leq_L\! \mathbf{\Sigma}_Y(\nabla G(\theta^*,\cdot))$, then $\mV_{\!\!X} \!\leq_L\! \mV_{\!\!Y}$ $(\mV'_{\!\!X} \!\leq_L\! \mV'_{\!\!Y}$ for the case of averaged iterates$)$;
\end{itemize}
where function $\nabla G(\theta^*,\cdot): [n]\to \R^d$ is defined in the SGD iteration \eqref{eqn:update_rule_general}, $\mV_X$ and $\mV'_X$ (resp. $\mV_Y$ and $\mV'_Y$) are the covariance matrices from Lemma \ref{Lemma:CLT_SGD}, corresponding to $\{X_t\}_{t \geq 0}$ (resp. $\{Y_t\}_{t \geq 0}$) as the stochastic input sequence.\qed
\end{theorem}

Theorem \ref{Thm:main_result} enables us to provide a sense of \textit{efficiency ordering of SGD algorithms} which are driven by different stochastic inputs. Since this is achieved via Loewner ordering, it also leads to smaller confidence intervals in the long run as mentioned earlier, as well as potentially smaller MSE\footnote{The mean square error can be retrieved as the trace of the covariance matrix (weighted sum of its diagonal entries in case of weighted MSE). Loosely speaking, an iterate having a smaller covariance matrix in the Loewner ordering will then also have a smaller MSE (weighted MSE).} depending on the objective function.
\begin{remark}
In addition to the CLT result for SGD algorithms with diminishing step size described in Lemma \ref{Lemma:CLT_SGD}, we include in Appendix \ref{appendix:constant_step_size} similar results for constant step sizes and quadratic objective functions, where the statement of Theorem \ref{Thm:main_result} still holds. \qed
\end{remark}
\section{Applications: Towards More Efficient SGD}\label{section:application}

In this section, we present some SGD variants and compare them in terms of efficiency ordering of SGD. Specifically, we first show that a certain class of non-Markov random walks can provide a better input sequence than its Markovian counterpart. We then analyze shuffling-based gradient descent and compare it to the SGD with \textit{i.i.d} input in terms of efficiency ordering for SGD algorithm. We also extend our approach to a more general mini-batch version, the discussion for which is deferred to Appendix \ref{app:appendixH3}.

\textbf{High-Order Efficient Random Walk for SGD:} The simple random walk (SRW) is a popular Markov chain that has been extensively studied in the literature \citep{ross1996stochastic,rasti2009respondent,gjoka2011practical}. Several recent works have focused on the non-backtracking random walk (NBRW) on a connected undirected graph $\cG(\cV,\cE)$ in the MCMC literature, which is an extension of SRW with the same limiting distribution $\vpi = \vd/\vone^T\vd$ \citep{neal2004improving,alon2007non,lee2012beyond,kempton2016non,ben2018comparing}. Intuitively speaking, NBRW is a random walk that selects one of its neighbors uniformly at random \emph{except} the one it just came/transitioned from. Specifically, the NBRW $\{Y_t\}_{t\geq 0}$ is a second-order non-reversible Markov chain (i.e., it is non-Markov on $\cV = [n]$) with its transition probability given by
\begin{equation}\label{eqn:nbrw_trans_prob}
    P(Y_{t+1} = j|Y_{t} = i, Y_{t-1} = k)
    = \begin{cases} \frac{1}{d_i-1} \!\!& \text{if} ~ j\neq k, j\!\in\! N(i), d_i > 1, \\ 1 \!\!& \text{if} ~ d_i \!=\! 1, j\!\in\! N(i), \\ 0 \!\!& \text{otherwise.}\end{cases}
\end{equation}
Since the limiting distributions of NBRW and SRW are the same, NBRW can be used as the input for SGD iterations \eqref{eqn:update_rule_general} with the same re-weighted local functions $G(\theta^*,i)$ as that of SRW for all $i \in [n]$ whenever the applications call for random-walk type of inputs.
Let $\mSigma_Y(\nabla G(\theta^*,\cdot))$ be the asymptotic covariance matrix of this NBRW $\{Y_t\}_{t\geq 0}$, as defined in \eqref{eqn:asymptotic covariance matrix}. One of the main results in \citep{lee2012beyond} concerns the efficiency ordering of NBRW and SRW. They show that NBRW has a smaller AV, or equivalently, from our Theorem \ref{Thm:main_result} (i), a smaller asymptotic covariance in terms of Loewner ordering. Our next result forms the necessary connection between the asymptotic covariance matrix arising in the CLT result and $\mSigma_Y(\nabla G(\theta^*,\cdot))$. 
\begin{proposition}\label{cor:reweight_nonMarkov}
Consider the SGD iteration \eqref{eqn:update_rule_general} with two input sequences SRW $\{X_t\}_{t \geq 0}$ and NBRW $\{Y_t\}_{t \geq 0}$ respectively. Then, both the respective estimators $\theta^X_t, \theta^Y_t \xrightarrow[t \to \infty]{a.s.} \theta^*$, and $\mV_Y \leq_L \mV_X$, that is, NBRW is more efficient than SRW in the SGD algorithm.\qed
\end{proposition}
By augmenting the state space, we can represent NBRW as a Markov chain $Z_t = (Y_{t-1},Y_t)\in \cV \times \cV$, as was done in \citep{neal2004improving,lee2012beyond}. This transformation then allows us to build CLT for an SGD iteration with $\{Z_t\}_{t \geq 0}$ as the input. The subtlety here is to prove that the asymptotic covariance matrix arising out of the CLT with respected to the augmented process $\{Z_t\}_{t \geq 0}$ is indeed equal to $\mSigma_Y(\nabla G(\theta^*,\cdot))$. This is shown by cultivating the relationship between the stationary distribution of $\{Z_t\}_{t\geq 0}$ on the augmented state space $\cV \times \cV$ and $\{Y_t\}_{t\geq 0}$ on the node space $\cV$, as provided in \citep{neal2004improving}.

Thus, our Theorem \ref{Thm:main_result} together with the existing works on efficiency ordering of NBRW versus SRW in the MCMC literature \citep{neal2004improving,lee2012beyond} enable us to show that NBRW is a more efficient input sequence than SRW for the SGD iteration \eqref{eqn:update_rule_general}. Interestingly, it has been shown that non-backtracking walks mix faster when the underlying graph is $d-$regular \citep{alon2007non}. In this case, a faster convergence rate is also suggested by mixing time based non-asymptotic bounds prevalent in RWSGD literature. However, no such results concerning mixing time and SLEM exists for NBRW on a general graph. Thus, in the form of Proposition \ref{cor:reweight_nonMarkov}, we demonstrate the utility of our approach in settings where mixing time based comparisons are unavailable.

\textbf{Shuffling versus \textit{i.i.d} Input Sequence:} Shuffling-based methods have been widely used in machine learning applications \citep{bottou-slds-open-problem-2009}. They work by repeatedly passing over the entire state space $[n]$ without repetition, each complete pass forming an \textit{epoch}. \textit{Random shuffling} and \textit{single shuffling} are two versions therein and differ in the order in which they pass over $[n]$. Random shuffling, as the name suggests, makes the pass in a randomly chosen order in each epoch, while single shuffling maintains the same predetermined order (often randomly chosen once at the beginning) for all epochs.  Shuffling-based methods are known to show better empirical performance than \textit{i.i.d} input \citep{bottou2012stochastic}, although intense theoretical analysis for shuffling-based gradient descent has only emerged in recent years \citep{shamir2016without,haochen2019random,safran2020good,ahn2020sgdshuffling,gurbuzbalaban2021random}. In what follows, we use our results from Section \ref{Section:Results random walk} to compare shuffling-based gradient descent to SGD with \textit{i.i.d} input. To do so, we first analyze the asymptotic covariance matrix for shuffling-based methods.

\begin{lemma}\label{Lem:Shuffling finite time var}
Let the input process $\{X_t\}_{t\geq 0}$ be single or random shuffling. Then, for any vector-valued function $\vg: [n] \to \R^d$, $\mSigma_X(\vg) = \vzero$, where $\mSigma_X(\vg)$ is defined in \eqref{eqn:asymptotic covariance matrix}.\qed
\end{lemma}

For \textit{i.i.d} input sequence with distribution $\hat{\vpi}$, the asymptotic covariance from Lemma \ref{Lemma:CLT_SGD} reduces to
\begin{equation}\label{eqn:cov_matrix_iid}
    \mSigma_X(\nabla G(\theta^*,\cdot)) \triangleq \text{Var}_{X_0 \sim \hat{\vpi}}\left( \nabla G(\theta^*,X_0)\right)
\end{equation}
following its definition in  \eqref{eqn:asymptotic covariance matrix}, and thus, trivially, $\mSigma_X(\nabla G(\theta^*,\cdot)) \geq_L \vzero$. Lemma \ref{Lem:Shuffling finite time var} shows that shuffling-based methods are more efficient than \textit{i.i.d} input sequence due to a smaller asymptotic covariance matrix in Loewner ordering. Next, we show that they also outperform \textit{i.i.d} input when used for driving the input sequence of SGD algorithms. 
\begin{proposition}\label{Cor:Shuffling vs iid}
Consider the SGD iteration \eqref{eqn:update_rule_general} with stochastic inputs single/random shuffling $\{X_t\}_{t\geq 0}$ and \textit{i.i.d} sampling $\{Y_t\}_{t\geq 0}$, we have $\theta^X_t, \theta^Y_t \xrightarrow[t \to \infty]{a.s.} \theta^*$ and $\mV_X = \vzero \leq_L \mV_Y$.\qed
\end{proposition}\vspace{-2mm}

Though it may seem so at first, Proposition \ref{Cor:Shuffling vs iid} is not a simple application of Theorem \ref{Thm:main_result}, especially for random shuffling because it is hard to check if random shuffling, formulated as a time-inhomogeneous Markov chain, indeed satisfies (A4). To overcome this difficulty, in Appendix \ref{app: generalized_minibatch_shuffling_proof} we come up with a non-trivial augmentation to a much higher dimensional state space ($[n]^{n+1}$) to make random shuffling a time-homogeneous periodic Markov chain in order to show that both single shuffling and random shuffling satisfy (A4) and thus apply Theorem \ref{Thm:main_result}.

The case of shuffling versus \textit{i.i.d} inputs is an example of a setting where the sequence with larger SLEM is more efficient than one with smaller SLEM\footnote{The single shuffling when realized as a periodic Markov chain has SLEM $=1$ (transition matrix is unitary), while the \textit{i.i.d} input sequence has SLEM $=0$ (transition matrix is rank one).} as an input sequence to the SGD iteration \eqref{eqn:update_rule_general}. For quadratic objective functions that satisfy the linear SA iteration in \citep{chen2020explicit}, it also attains a faster convergence speed in terms of MSE than \textit{i.i.d} inputs to SGD algorithms. Although some recent works provide more informative finite-time error bounds on the MSE of the objective function for shuffling-based methods, by studying a special case of the matrix norm AM-GM inequality and proving faster convergence rate than \textit{i.i.d} inputs \citep{rajput2020closing,ahn2020sgdshuffling,gurbuzbalaban2021random}, our result is not a subset of theirs. To be precise, we show in Appendix \ref{app:discussion_pl_pdm} that our assumption (A2) on the objective function is no less general than their most general setting based on the Polyak-Łojasiewicz inequality.

\begin{remark}
Mini-batch gradient descent is another popular gradient descent variant and is widely used in the machine learning tools \citep{chollet2015keras,abadi2016tensorflow,paszke2019pytorch} to accelerate the learning process when compared to SGD. In Appendix \ref{app:appendixH3} we show how our framework can be applied to study min-batch based SGD algorithms, and include the asymptotic analysis on mini-batch gradient descent with shuffling. 

Besides mini-batch gradient descent, there are other SGD variants, e.g., momentum SGD, Nesterov accelerated SGD and ADAM, that have been studied in the SGD literature for non-asymptotic analysis \citep{Kingma2015,Reddi2018,assran2020convergence}. However, asymptotic analysis on the SGD variants is not well studied in the literature, with only very recent results on the CLT for \textit{i.i.d} input sequences \citep{lei2020variance,barakat2021convergence,barakat2021stochastic,li2022revisiting}. Asymptotic analysis and CLT for variants more general Markovian input sequences, which would be a prerequisite for Theorem \ref{Thm:main_result}, remains an open problem. We defer the discussion on related works to Appendix \ref{app:simulation_ext}, where we also empirically test the SGD variants and find that the efficiency ordering result still holds for these SGD variants - opening up the possibility for further theoretical analysis.\qed 
\end{remark}
\section{Numerical Experiments}\label{section:simulation}

In this section, we empirically validate our theoretical analysis.
We select two convex objective functions as follows.\vspace{-2mm}
\begin{equation}\label{eqn:obj}
   \tilde{f}(\theta) \!=\! \frac{1}{n}\!\sum_{i=1}^n \log(1\!+\!\text{exp}(-y_i\vx_i^T\theta)) \!+\! \frac{1}{2}\|\theta\|_2^2~,~~~ \hat{f}(\theta) \!=\! \frac{1}{n}\sum_{i=1}^n \theta^T(\va_i\va_i^T \!+\! \mD_i)\theta \!+\! \vb^T\theta.
\end{equation}
For $l_2$-regularized logistic regression $\tilde{f}(\theta)$, we choose the dataset CIFAR-10 \citep{krizhevsky2009learning} where $n$ is the total number of data points. Here, $\vx_i\in\R^{108}$ is the vector flattened from the cropped image $i$ with shape $(6, 6, 3)$, and $y_i\in\R$ is the label. For sum-of-non-convex functions $\hat{f}(\theta)$, which is based on the experiment setup in \citep{gower2019sgd,allen2016improved},
we generate random vectors $\va_i,\vb$ and matrices $\mD_i$ which ensure the invertibility of matrix $\sum_{i=1}^n \va_i\va_i^T$ and $\sum_{i=1}^n \mD_i = \mathbf{0}$ (details are deferred to Appendix \ref{subsection:additional_result}). For both experiments, we assign a data point to each node $i$ on the general graph `Dolphins' ($62$ nodes) \citep{nr}.  We set the step size in the SGD algorithm as $1/t^{0.9}$, and use MSE $\E[\|\theta_t-\theta^*\|_2^2]$ to measure the relative performance of different inputs. We also employ the scaled MSE $\E[\|\theta_t-\theta^*\|_2^2]/\gamma_t$ to empirically show its relationship to the CLT result \eqref{eqn:error}. Due to space constraints, additional simulation results which support our efficiency ordering result in the setting of large graphs, and non-convex functions are deferred to Appendix \ref{app:simulation_ext}. Therein, via numerical simulations, we also observe the efficiency ordering for other SGD variants such as Nesterov accelerated SGD and ADAM when comparing their iterations under efficiency ordered noise sequences.

In Figure \ref{fig:5} we compare NBRW and SRW as input sequences on the graph `Dolphins' for two objective functions in \eqref{eqn:obj}. We also compare uniform sampling, random shuffling and single shuffling, assuming that they can access any node on the graph in each iteration. We can see in Figure \ref{fig:1} and \ref{fig:3} that NBRW always falls below SRW throughout all time periods, which indicates that NBRW tends to have smaller MSE than SRW. Single and random shuffling are both better than uniform sampling in terms of smaller MSE. The oscillation of single shuffling comes from a predetermined fixed data sampling sequence, while random shuffling changes the permutation whenever traversing all nodes. Such oscillation is not visible in Figure \ref{fig:3} and Figure \ref{fig:4} because it is small on the current y-axis scale. The curves of single and random shuffling in Figure \ref{fig:2} and \ref{fig:4} fall below that of uniform sampling and still decrease in the linear rate because eventually their covariance matrices will be zero matrix, as indicated in Proposition \ref{Cor:Shuffling vs iid}. Figure \ref{fig:2} shows that the scaled MSEs of NBRW, SRW and uniform sampling approach some constants after some time, which is consistent with the CLT result \eqref{eqn:error}. The curves of NBRW are still below that of SRW, showing that the input with smaller scaled MSE tends to have higher efficiency, which supports Proposition \ref{cor:reweight_nonMarkov}. We can see from Figure \ref{fig:4} that the curves NBRW, SRW are still increasing andd they have not yet entered the regime where the covariance matrix becomes the main factor. On the other hand, uniform sampling and both shuffling methods are just entering this regime.
\vspace{-2mm}
\begin{figure}[!h]
    \centering
     \begin{subfigure}[b]{0.37\textwidth}
         \centering
         \includegraphics[width=\textwidth]{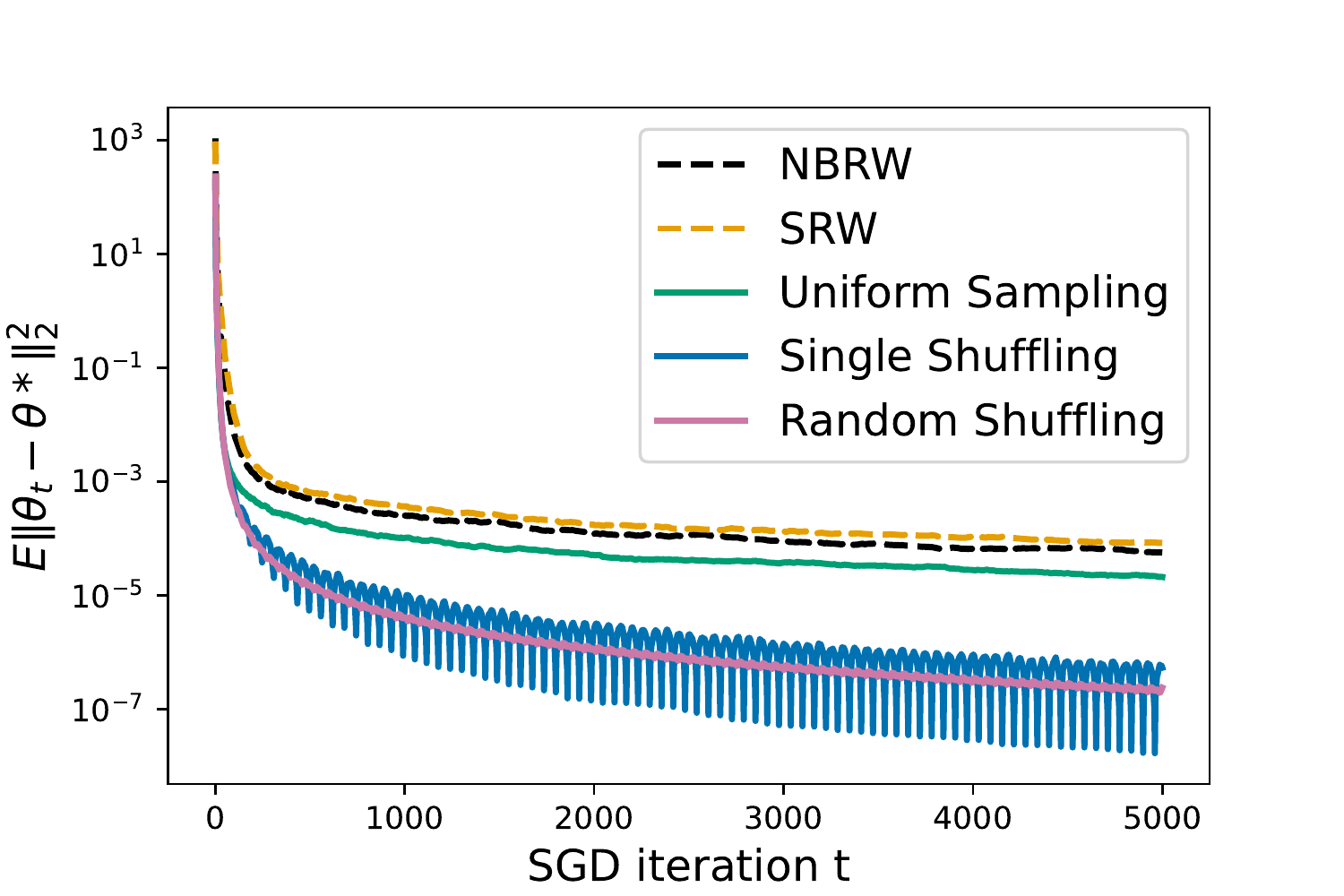}
         \caption{Logistic regression;
         (MSE)}
         \label{fig:1}
     \end{subfigure}\quad
     \begin{subfigure}[b]{0.37\textwidth}
         \centering
         \includegraphics[width=\textwidth]{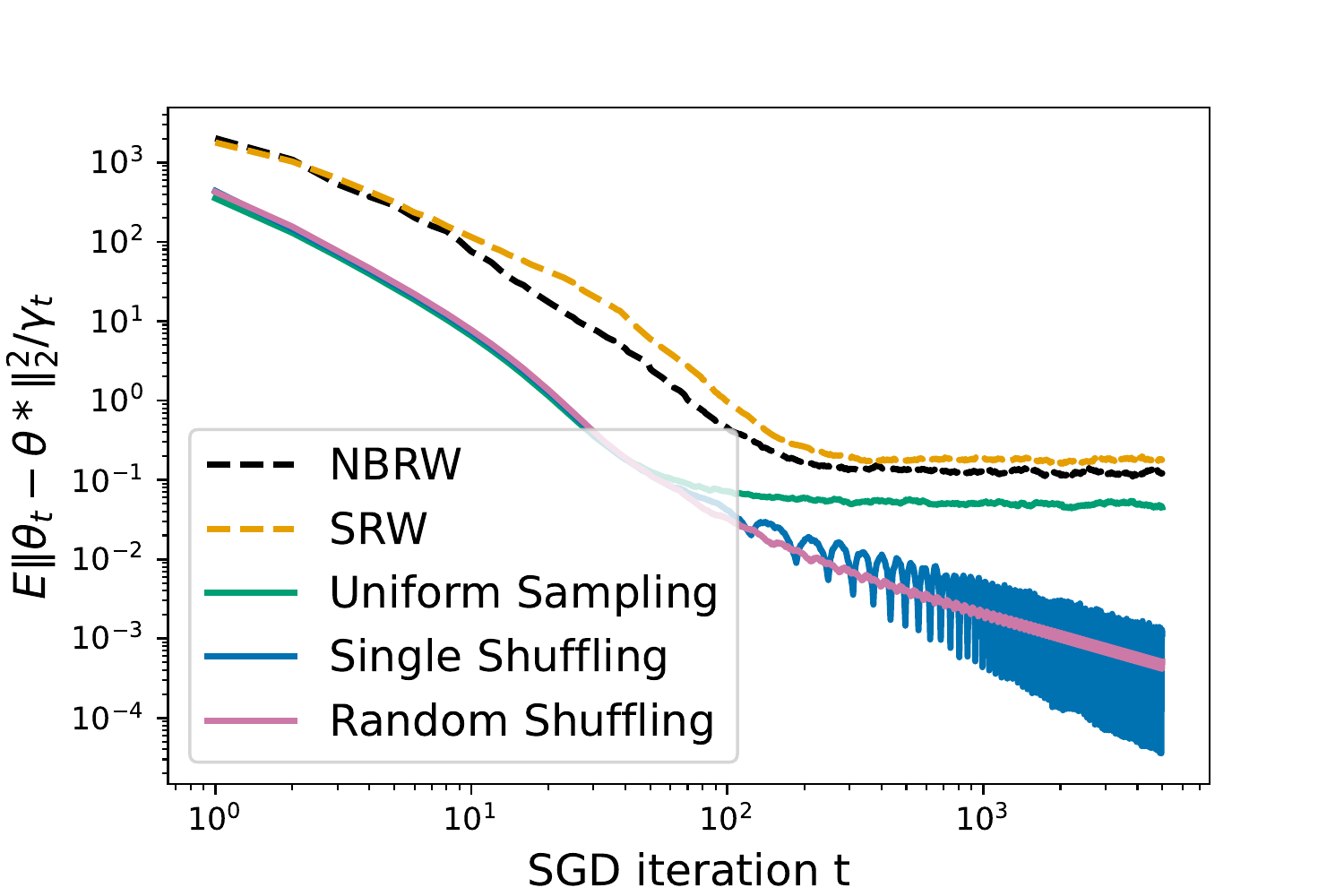}
         \caption{Logistic regression;
         (scaled MSE)}
         \label{fig:2}
     \end{subfigure} 

     \begin{subfigure}[b]{0.37\textwidth}
         \centering
         \includegraphics[width=\textwidth]{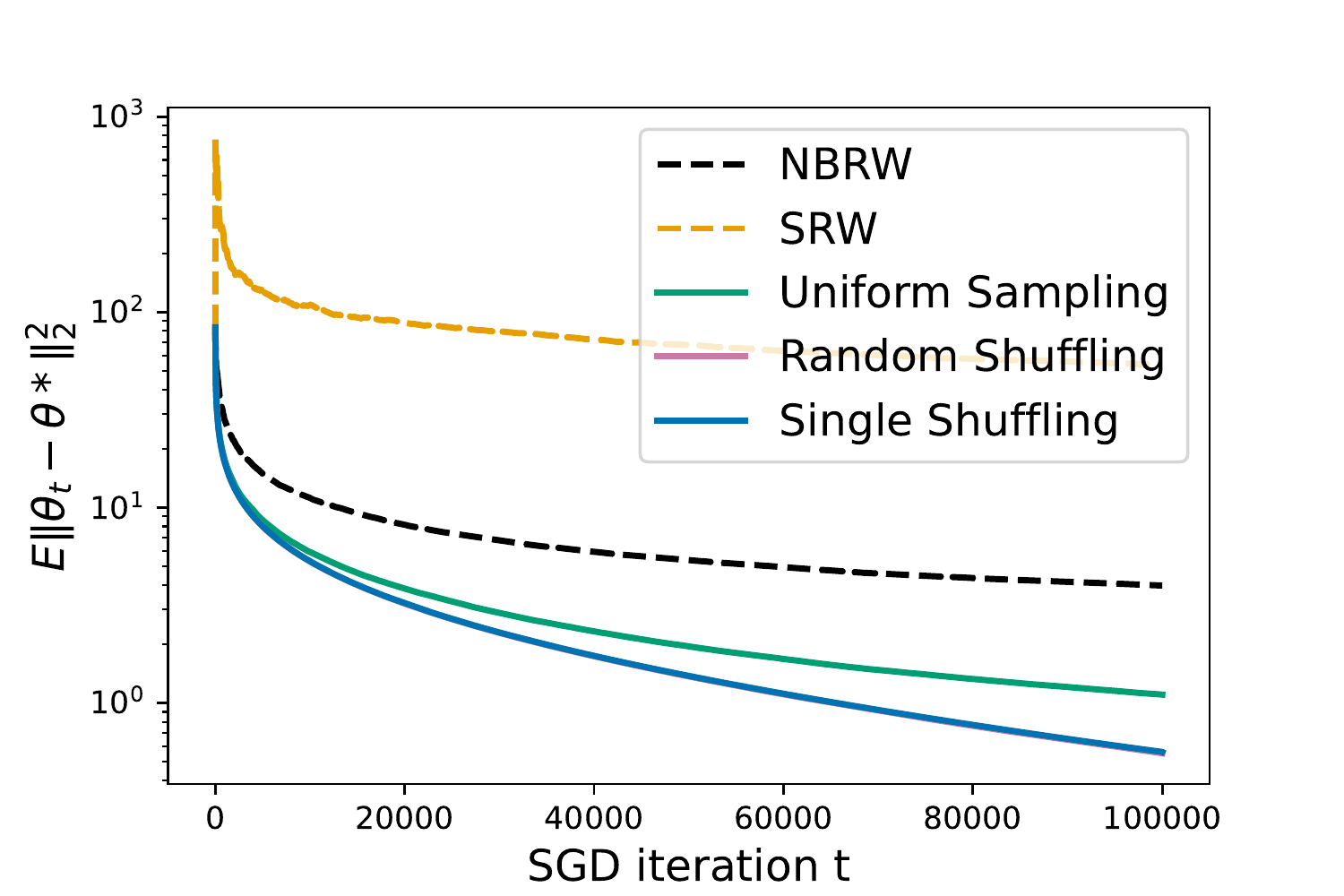}
         \caption{sum-non-convex fn.;
         (MSE)}
         \label{fig:3}
     \end{subfigure}\quad
     \begin{subfigure}[b]{0.37\textwidth}
         \centering
         \includegraphics[width=\textwidth]{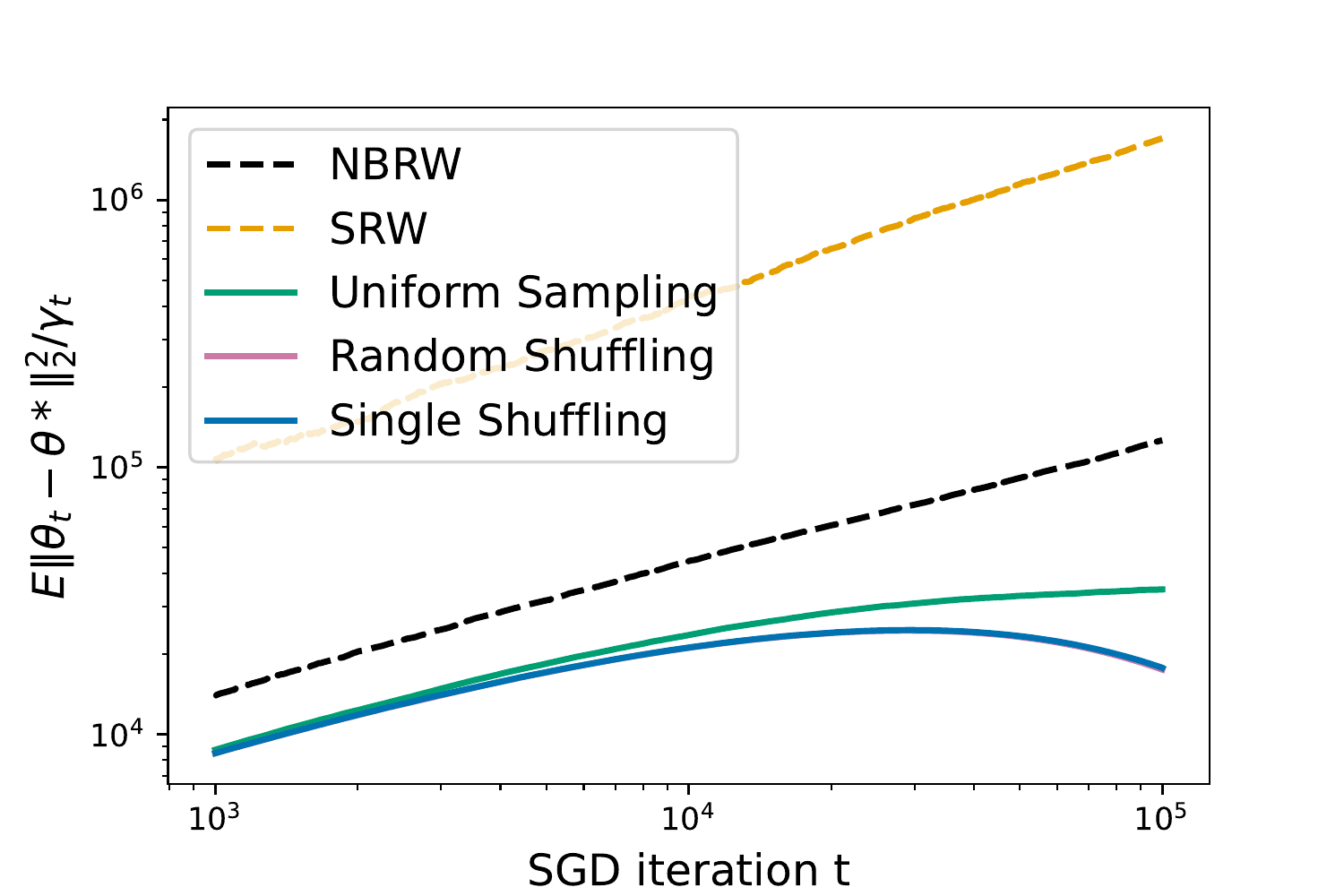}
         \caption{sum-non-convex fn.;
         (scaled MSE)}
         \label{fig:4}
     \end{subfigure}
     \vspace{-2mm}
     \caption{Performance comparison of different stochastic inputs on the graph `Dolphins'.}
     \label{fig:5}
     \vspace{-4mm}
\end{figure}
\vspace{-3mm}

\section{Conclusion}\label{Section:conclusion}
We have introduced the notion of efficiency ordering of SGD algorithms, and shown that processes with smaller asymptotic covariance are always more efficient as input sequences for SGD algorithms. Furthermore, we numerically observe that this sense of efficiency ordering is also seen SGD variants such as Nesterov accelerated SGD and ADAM. Since the asymptotic analysis of such SGD variants is not well-established under general stochastic inputs, establishing theoretical results on efficiency ordering remain an open problem.

\newpage
\section{Acknowledgments and Disclosure of Funding}
We thank the anonymous reviewers for their constructive comments. The research was conducted while Vishwaraj Doshi was with the Operations Research Graduate Program, North Carolina State University. This work was supported in part by National Science Foundation under Grant Nos. CNS-2007423, IIS-1910749, and CNS-1824518. 

\section*{Checklist}

The checklist follows the references.  Please
read the checklist guidelines carefully for information on how to answer these
questions.  For each question, change the default \answerTODO{} to \answerYes{},
\answerNo{}, or \answerNA{}.  You are strongly encouraged to include a {\bf
justification to your answer}, either by referencing the appropriate section of
your paper or providing a brief inline description.  For example:
\begin{itemize}
  \item Did you include the license to the code and datasets? \answerYes{See Section \_\_.}
  \item Did you include the license to the code and datasets? \answerNo{The code and the data are proprietary.}
  \item Did you include the license to the code and datasets? \answerNA{}
\end{itemize}
Please do not modify the questions and only use the provided macros for your
answers.  Note that the Checklist section does not count towards the page
limit.  In your paper, please delete this instructions block and only keep the
Checklist section heading above along with the questions/answers below.

\begin{enumerate}

\item For all authors...
\begin{enumerate}
  \item Do the main claims made in the abstract and introduction accurately reflect the paper's contributions and scope?
    \answerYes{}
  \item Did you describe the limitations of your work?
    \answerNo{}
  \item Did you discuss any potential negative societal impacts of your work?
    \answerNA{}
  \item Have you read the ethics review guidelines and ensured that your paper conforms to them?
    \answerYes{}
\end{enumerate}

\item If you are including theoretical results...
\begin{enumerate}
  \item Did you state the full set of assumptions of all theoretical results?
    \answerYes{}
        \item Did you include complete proofs of all theoretical results?
    \answerYes{All the proofs are included in the supplementary material.} 
\end{enumerate}

\item If you ran experiments...
\begin{enumerate}
  \item Did you include the code, data, and instructions needed to reproduce the main experimental results (either in the supplemental material or as a URL)?
    \answerNo{}
  \item Did you specify all the training details (e.g., data splits, hyperparameters, how they were chosen)?
    \answerYes{The details of the simulation setup are provided in the supplementary material.}
        \item Did you report error bars (e.g., with respect to the random seed after running experiments multiple times)?
    \answerNo{}
        \item Did you include the total amount of compute and the type of resources used (e.g., type of GPUs, internal cluster, or cloud provider)?
    \answerNo{}
\end{enumerate}

\item If you are using existing assets (e.g., code, data, models) or curating/releasing new assets...
\begin{enumerate}
  \item If your work uses existing assets, did you cite the creators?
    \answerYes{The citations can be found in Section 5.}
  \item Did you mention the license of the assets?
    \answerNA{}
  \item Did you include any new assets either in the supplemental material or as a URL?
    \answerNA{}
  \item Did you discuss whether and how consent was obtained from people whose data you're using/curating?
    \answerNA{}
  \item Did you discuss whether the data you are using/curating contains personally identifiable information or offensive content?
    \answerNA{}
\end{enumerate}

\item If you used crowdsourcing or conducted research with human subjects...
\begin{enumerate}
  \item Did you include the full text of instructions given to participants and screenshots, if applicable?
    \answerNA{}
  \item Did you describe any potential participant risks, with links to Institutional Review Board (IRB) approvals, if applicable?
    \answerNA{}
  \item Did you include the estimated hourly wage paid to participants and the total amount spent on participant compensation?
    \answerNA{}
\end{enumerate}

\end{enumerate}

\appendix
\section{CLT for Stochastic Approximation Algorithms}\label{Appendix:CLT_assumptions}
The existing central limit theorem (CLT) for stochastic approximation (SA) with Markovian dynamics \citep{benveniste2012adaptive,delyon2000stochastic,fort2015central} usually studied a general Markov process $\{X_t\}_{t\geq 0}$ on the finite state space $\cV$ and its transition kernel $P_{\theta}$ dependent on $\theta$ such that $P(X_{t+1}\in A|X_t=x,\theta_t=\theta) = P_{\theta}(x,A)$ for any subset $A\subseteq \cV$. Denote $\vpi_{\theta}$ the stationary distribution of $P_{\theta}$. Define $P_{\theta}v_{\theta}(x) \triangleq \sum_{l\in\cV} [P_{\theta}]_{x,l}\cdot v_{\theta}(l)$. The general SA algorithm is of the form
\begin{equation}\label{eqn:SA_iteration}
    \theta_{t+1} = \text{Proj}_{\Theta}\left(\theta_t + \gamma_{t+1}H(\theta_t,X_{t+1})\right),
\end{equation}
where $\Theta \subset \R^d$ is a closed and convex set. The main goal is to find the root $\theta^*$ of function $$h(\theta) \triangleq \E_{X\sim\vpi_{\theta}}[H(\theta,X)] ~\text{i.e.,} ~ h(\theta^*) = 0.$$
As mentioned in \citep{benveniste2012adaptive} p.332 Theorem 13, and \citep{delyon2000stochastic} p.31 Theorem 15, p.59 Theorem 25, the usual assumptions are given as
\begin{enumerate}
    \item[(B1)] Function $h:\Theta\to\R^d$ is continuous on $\Theta$, there exists a non-negative $C^1$ function $V$ such that $\langle \nabla V(\theta),h(\theta)\rangle \leq 0, \forall \theta\in\Theta$ and the set $\cS=\{\theta; \langle \nabla V(\theta),h(\theta)\rangle = 0\}$ is such that $V(\cS)$ has empty interior. Also, $V(\theta)$ tends to $+\infty$ if $\theta \to \partial \Theta$, where $\partial\Theta$ is the boundary of $\Theta$, or $|\!|\theta|\!|_2 \to \infty$. There exists a compact set $\mathcal{K}\subset \Theta$ such that $\langle \nabla V(\theta),h(\theta)\rangle < 0$ if $\theta\notin \mathcal{K}$; 
    \item[(B2)] For every $\theta$, there exist a function $v_{\theta}(x)$ such that the Poisson equation
    \begin{equation}\label{eqn:possion_b2}
       v_{\theta}(x) \!-\! \E[v_{\theta}(X_{t+1})|X_t=x,\theta_t=\theta] \!=\! H(\theta, X) \!-\! h(\theta).
    \end{equation}
    For any compact set $\mathcal{C} \subset \Theta$, 
    \begin{equation}\label{eqn:boundedness}
        \sup_{\theta\in\mathcal{C},x\in\cV} \|H(\theta,x)\|_2 + \|v_{\theta}(x)\|_2 < \infty.
    \end{equation}
    There exists a continuous function $\phi_{\mathcal{C}}$, $\phi_{\mathcal{C}}(0)=0$, such that for any $\theta,\theta'\in\mathcal{C}$,
    \begin{equation}\label{eqn:l-lipschitz}
        \sup_{X\in\cV} \|P_{\theta}v_{\theta}(x) - P_{\theta'}v_{\theta'}(x)\|_2 \leq \phi_{\mathcal{C}}\left(\| \theta-\theta'\|_2\right).
    \end{equation}
    \item[(B3)] The step size follows $\gamma_t\geq 0, \sum_{t\geq 1} \gamma_t = \infty, \sum_{t\geq 1} \gamma_t^2 < \infty$ and $\sum_{t\geq 1} |\gamma_{t+1} - \gamma_{t}| < \infty$.
    \item[(B4)] Assume $\theta_t$ converges to some limit $\theta^*\in \cS$. Function $h$ is $C^1$ in some neighborhood of $\theta^*$ with first derivatives Lipschitz, and matrix $\nabla h(\theta^*)$ has all its eigenvalues with negative real part.

\end{enumerate}
Then, we have the following convergence and CLT result.
\begin{theorem}\label{Thm:CLT_SA}\citep{benveniste2012adaptive,delyon2000stochastic,fort2015central}
Assume $\theta_t$ is given by the SA iteration \eqref{eqn:SA_iteration} that satisfies assumptions (B1) -- (B3) above, then iterate $\theta_t$ converges almost surely to the set $\cS$ defined in (B1). Moreover, with additional assumption (B4), we have
\begin{equation}\label{eqn:error2}
    \frac{1}{\sqrt{\gamma_t}}\cdot\left(\theta_t - \theta^*\right) \xrightarrow[t \to \infty]{Dist} \cN(0, \mV_{\!\!X}),
\end{equation}
where covariance matrix $\mV_{\!\!X}$ is the unique solution to the following Lyapunov equation:
\begin{equation}\label{eqn:Lyapunov_equation}
    \!\!\begin{cases}
    \mSigma_X \!+\! \mK\mV_{\!\!X} \!+\! \mV_{\!\!X}\mK^T\!=\! \vzero &\text{if}~ \alpha \!\in\! (\frac{1}{2},1), \\
    \mSigma_X \!+\! \left(\mK\!+\!\frac{\mI}{2}\right)\!\!\mV_{\!\!X} \!+\! \mV_{\!\!X}\!\!\left(\mK\!+\!\frac{\mI}{2}\right)^T \!=\! \vzero &\text{if}~ \alpha = 1. \\
    \end{cases}
\end{equation}
Here, $\mK \triangleq \nabla h(\theta^*)$ and $\mSigma_X \triangleq \mSigma_X(H(\theta^*,\cdot))$ is the asymptotic covariance matrix as in \eqref{eqn:asymptotic covariance matrix}, evaluated at function $H(\theta^*,\cdot)$. 

In addition, for averaged iterates $\Bar{\theta}_t \triangleq \frac{1}{t}\sum_{i=0}^{t-1}\theta_t$, we still have
$\bar{\theta}_t \xrightarrow[t \to \infty]{a.s.} \theta^*$, and
\begin{equation}\label{eqn:error_pr_averaging2}
    \sqrt{t}\cdot(\bar{\theta}_t-\theta^*) \xrightarrow[t \to \infty]{Dist} \cN(0, \mV'_{\!\!X}),
\end{equation}
where $\mV'_{\!\!X} = \mK^{-1}\mSigma_X(\mK^{-1})^T$ with the same matrices $\mK$ and $\mSigma_X$ as in \eqref{eqn:Lyapunov_equation}. \qed
\end{theorem}

\section{Proof of Lemma \ref{Lemma:CLT_SGD}}\label{app: proof of clt lemma}
To prove Lemma \ref{Lemma:CLT_SGD} with existing Theorem \ref{Thm:CLT_SA}, we need to show that (A1) -- (A5) is a special case of (B1) -- (B4). We list (A1) -- (A5) here for self-contained purpose.

\begin{enumerate}
    \item[(A1)] The step size is given by $\gamma_t = t^{-\alpha}$ for $\alpha \in (1/2,1]$;
    
    \item[(A2)] There exists a unique minimizer $\theta^*$ in the interior of the compact set $\Theta$ with $\nabla f(\theta^*) = 0$, and matrix $\nabla^2 f(\theta^*)$ (resp.\! $\nabla^2 f(\theta^*)-\mI/2$) is positive definite for $a \!\in\! (1/2,1)$ (resp.\! $a \!=\! 1$);
    
    \item[(A3)] Gradients are bounded in the compact set $\Theta$, that is, $~\sup_{\theta \in \Theta}\sup_{i \in [n]} |\!|\nabla F(\theta,i)|\!|_2 < \infty$;
    
    \item[(A4)] For every $z\in[n],\theta\in\R^d$, the solution $\Tilde{F}(\theta,z)\in\R^d$  of the Poisson equation 
    \begin{equation}\label{eqn:poisson_equation}
        \Tilde{F}(\theta,z) \!-\! \E[\Tilde{F}(\theta,X_{t+1})~|~X_t=z] = \nabla F(\theta, z) \!-\! \nabla f(\theta)
    \end{equation}
    exists, and $\sup_{\theta\in\Theta,z\in[n]} \|\Tilde{F}(\theta,z)\|_2<\infty$;
    
    \item[(A5)] The functions $F(\theta,i)$ are $L$-smooth for all $i \in [n]$, that is, $\forall \theta_1,\theta_2 \in \Theta, \forall i \in [n]$, we have $\|\nabla F(\theta_1, i) - \nabla F(\theta_2, i)|\!|_2 \leq L |\!|\theta_1 - \theta_2 \|_2$.
\end{enumerate}

Let $H(\theta,X) \triangleq -\nabla F(\theta,X)$ for function $F(\theta,X)$ defined in \eqref{eqn:objective_function}. Then, we have $h(\theta) \triangleq \E_{X\sim \vpi}[H(\theta,X)] = -\nabla f(\theta)$. By choosing $V(\theta) \triangleq f(\theta)$, we know $\langle\nabla V(\theta),h(\theta)\rangle = -\nabla f(\theta)^2\leq 0$. From (A2) we know $\theta^*$ is the unique minimizer of function $f$, by letting $\mathcal{K} = \{\theta^*\}$,  we have $\langle \nabla V(\theta), h(\theta) \rangle < 0$ when $\theta \notin \mathcal{K}$. Therefore, \textbf{(B1) is satisfied}.

Now we need to check assumption (B2). Assumption (A4) is a direct translation to \eqref{eqn:possion_b2} in (B2), and $\sup_{\theta\in\Theta,z\in[n]} \|\Tilde{F}(\theta,z)\|_2<\infty$, as well as assumption (A3), implies \eqref{eqn:boundedness}. We still need to show \eqref{eqn:l-lipschitz}.
By assuming an $n$-state ergodic Markov chain $\{X_t\}_{t\geq 0}$ ($\theta$-independent) with transition kernel $\mP\in\R^{n \times n}$ and stationary distribution $\vpi$, the solution $\Tilde{F}(\theta,z)$ to the Poisson equation \eqref{eqn:poisson_equation} in (A4) exists and is given as follows.\footnote{In this paper, we only consider $\theta$-independent Markovian inputs. The more general conditions of the $\theta$-dependent Markov chain under which the solution of \eqref{eqn:poisson_equation} in (A4) exists are referred to \citep{delyon2000stochastic} p.71 Theorem 35 or \citep{benveniste2012adaptive} p.217.}
\begin{equation}\label{eqn:solution_poisson_equation2}
    \Tilde{F}(\theta,z) \!=\! \nabla F(\theta, z) \!-\! \nabla f(\theta) \!+\! \sum_{l=1}^n \mP_{z,l}(\nabla F(\theta, l) \!-\! \nabla f(\theta)) \!+\! \sum_{l=1}^n \left[\mP^2\right]_{z,l}(\nabla F(\theta, l) \!-\! \nabla f(\theta)) \!+\! \cdots.
\end{equation}
Next, we can rewrite $\Tilde{F}(\theta,z)$ in the closed form and show that it is Lipschitz continuous and satisfies (20) in assumption (B2). Note that by definition of expectation and Chapman–Kolmogorov equation ($\sum_{k=1}^n\sum_{l=1}^n \mP_{z,k}\mP_{k,l} = \sum_{l=1}^n [\mP^2]_{z,l}$), we have
\begin{equation}\label{eqn:25}
\begin{split}
    \E[\Tilde{F}(\theta,X_{t+1})~|~X_t=z] =& \sum_{l=1}^n \mP_{z,l}(\nabla F(\theta, l) \!-\! \nabla f(\theta)) \!+\! \sum_{l=1}^n \left[\mP^2\right]_{z,l}(\nabla F(\theta, l) \!-\! \nabla f(\theta)) \\
    &+ \sum_{l=1}^n \left[\mP^3\right]_{z,l}(\nabla F(\theta, l) \!-\! \nabla f(\theta)) + \cdots.    
\end{split}
\end{equation}
Then, from \eqref{eqn:solution_poisson_equation2} and \eqref{eqn:25} we have $\Tilde{F}(\theta,z) - \E[\Tilde{F}(\theta,X_{t+1})~|~X_t=z] = \nabla F(\theta,z) - \nabla f(\theta)$, which is exactly \eqref{eqn:poisson_equation} in (A4). Moreover, since $\vone$ and $\vpi$ are the right and left eigenvectors of $\mP$ respectively with eigenvalue $1$, by induction we know 
\begin{equation}\label{eqn:26}
    \mP^k - \vone\vpi^T = (\mP-\vone\vpi^T)^k,\forall ~k\in\Z,~k\geq 1.
\end{equation}
Along with the fact that 
\begin{equation}\label{eqn:27}
    \nabla f(\theta) = \sum_{l=1}^n{\pi_l}\nabla F(\theta,l) = \sum_{l=1}^n [\vone\vpi^T]_{z,l} \nabla F(\theta,l),
\end{equation}
we can further simplify and get a closed form of \eqref{eqn:solution_poisson_equation2}, which is given below.
\begin{equation}\label{eqn:solution_poisson_equation}
\begin{split}
            \Tilde{F}(\theta,z) &= \nabla F(\theta, z) \!-\! \nabla f(\theta) \!+\! \sum_{l=1}^n \mP_{z,l}(\nabla F(\theta, l) \!-\! \nabla f(\theta)) \!+\! \sum_{l=1}^n \left[\mP^2\right]_{z,l}(\nabla F(\theta, l) \!-\! \nabla f(\theta)) \!+\! \cdots\\
        &= \sum_{l=1}^n \left[\mP^0-\vone\vpi^T\right]_{z,l} \nabla F(\theta, l) + \sum_{l=1}^n \left[\mP^1-\vone\vpi^T\right]_{z,l} \nabla F(\theta, l) + \cdots \\
        &= \left(\sum_{k=0}^{\infty}  \sum_{l=1}^n \left[\mP^k-\vone\vpi^T\right]_{z,l} \nabla F(\theta, l)\right) - \nabla f(\theta) \\
        &= \left(\sum_{l=1}^n \left[\sum_{k=0}^{\infty} \left[\mP^k-\vone\vpi^T\right]_{z,l}\right] \nabla F(\theta, l)\right) - \nabla f(\theta) \\
        &= \left(\sum_{l=1}^n \left[\sum_{k=0}^{\infty} \left[\left(\mP-\vone\vpi^T\right)^k\right]_{z,l}\right] \nabla F(\theta, l)\right) - \nabla f(\theta) \\
        &= \sum_{l=1}^n \left[\left(\mI - \mP + \vone\vpi^T\right)^{-1}\right]_{z,l} \nabla F(\theta,l) - \nabla f(\theta),
\end{split}
\end{equation}
where the second equality comes from \eqref{eqn:27} and the fifth equality is from \eqref{eqn:26}. Recall the definition $\mP F(\theta,z) \triangleq \sum_{l=1}^n \mP_{z,l} F(\theta,l)$, we can show that
\begin{equation}\label{eqn:29}
\begin{split}
    &\|\mP\Tilde{F}(\theta,z) - \mP\Tilde{F}(\theta',z)\|_2 \\ =& \left\|\sum_{l=1}^n \mP_{z,l} \left(\Tilde{F}(\theta,l) - \Tilde{F}(\theta',l)\right)\right\|_2 \\
    \leq& \sum_{l=1}^n \mP_{z,l}\|\Tilde{F}(\theta,l) - \Tilde{F}(\theta',l)\|_2 \\
    \leq& \sup_{z\in\cV}\|\Tilde{F}(\theta,z) - \Tilde{F}(\theta',z)\|_2 \\
    \leq& \sup_{z\in\cV}\left\|\sum_{l=1}^n \left[\left(\mI - \mP + \vone\vpi^T\right)^{-1}\right]_{z,l} \left(\nabla F(\theta,l)-\nabla F(\theta',l)\right) \right\|_2 + \|\nabla f(\theta) - \nabla f(\theta')\|_2 \\
    \leq& C\sup_{z\in\cV}\|\nabla F(\theta,z) - \nabla F(\theta',z)\|_2 + \|\nabla f(\theta) - \nabla f(\theta')\|_2 \\
    \leq& (C+1)L|\!|\theta - \theta'|\!|_2
\end{split}
\end{equation}
for some constant $C$ related to matrix $(\mI-\mP+\vone\vpi^T)^{-1}$, where the first and the third inequalities are from triangular inequality and the last inequality comes from assumption (A5). Note that we have 
$$\|\nabla f(\theta)\!-\!\nabla f(\theta')\|_2 \!=\! \left\|\!\sum_{i=1}^n\! \pi_i \left(\nabla F(\theta,i)\!-\!\nabla F(\theta',i)\right)\right\|_2\!\!\leq\!\sup_{z\in\cV} \|\nabla F(\theta,i)\!-\!\nabla F(\theta',i)\|_2 \!\leq\! L\|\theta-\theta'\|_2$$ 
in the last inequality of \eqref{eqn:29}. So $\eqref{eqn:l-lipschitz}$ is shown and \textbf{(B2) is satisfied}. 

For assumption (A1) with respect to the conditions on the step size, we know for $a\in(1/2,1]$, $\sum_{t \geq 1} 1/t^a = \infty$ and $\sum_{t\geq 0} 1/t^{2a} < \infty$. Besides, 
\begin{equation*}
    \gamma_t - \gamma_{t+1} = \frac{1}{t^a} - \frac{1}{(t+1)^a} = \frac{(t+1)^a-t^a}{t^{a}(t+1)^a} \leq \frac{(t+1)^a-t^a}{t^{2a}} \leq \frac{1}{t^{2a}}
\end{equation*}
where the second inequality comes from $(t+1)^a-t^a$ monotone decreasing in $t$ for $a\in(1/2,1]$. Then, we have $\sum_{t\geq 1}|\gamma_t - \gamma_{t+1}| \leq \sum_{t\geq 1} 1/t^{2a} < \infty$. Then, \textbf{(B3) is satisfied}. 

Since (B1) -- (B3) are satisfied and (A2) assumes unique minimizer such that $\mathcal{K} = \{\theta^*\}$, from Theorem \ref{Thm:CLT_SA} we know $ \theta_t \xrightarrow[t \to \infty]{a.s.} \theta^*$. Along with assumption (A2) on the positive definite matrix $\nabla^2 f(\theta^*)$, \textbf{(B4) is satisfied}.

Therefore, (A1)-(A5) implies (B1) -- (B4) and all the results from Theorem \ref{Thm:CLT_SA} can be carried over to Lemma \ref{Lemma:CLT_SGD}. 

\section{Discussion on Polyak-Lojasiewicz Inequality and Positive Definite Matrix $\nabla^2 f(\theta^*)$}\label{app:discussion_pl_pdm}
In this part, we discuss the strictness of the condition on the objective function $f$ between Polyak-Lojasiewicz (P-L) inequality and our assumption (A2) - positive definite matrix $\nabla^2 f(\theta^*)$. We say that if a scalar-valued function $f$ satisfies $\mu$-P-L inequality, then for any $\theta\in\R^d$, the following condition holds:
\begin{equation}\label{eqn:PLinequality}
    \frac{1}{2}\|\nabla f(\theta)\|_2^2\geq \mu(f(\theta) - f(\theta^*)),
\end{equation}
where $\nabla f(\theta)\in \R^d$, $f(\theta^*) = \min_{\theta\in\R^d} f(\theta)$ and the minimizer $\theta^*$ belongs to a  non-empty solution set. We define a new function $$g(\theta)\triangleq \frac{1}{2} \|\nabla f(\theta)\|_2^2 - \mu(f(\theta) - f^*).$$ 
Then, \eqref{eqn:PLinequality} is equivalent to saying $\min_{\theta} g(\theta) \geq 0$, and the necessary condition to ensure that $\theta^*$ is the local minimizer is $\nabla^2 g(\theta^*) \geq_L \vzero$ (e.g., Chapter 1.2 \citep{liberzon2011calculus}). We have
\begin{equation*}
    \nabla^2 g(\theta) = \left(\nabla^2 f(\theta) - \mu\mI\right)\nabla^2 f(\theta) + \mM \otimes \nabla f(\theta),
\end{equation*}
where matrix $\mM$ is a 3D matrix with dimension $d\times d \times d$ and $\otimes$ is the tensor product. Since $\nabla f(\theta^*)=0$, we have $\mM \otimes (\nabla f(\theta^*)) = 0$. Then, $\nabla^2 g(\theta^*) \geq_L \vzero$ implies $(\nabla^2 f(\theta^*))^2 \geq_L \mu \nabla^2 f(\theta^*)$. Denote $\lambda_i\geq 0, i=1,2,\cdots,d$ the eigenvalues of matrix $\nabla^2 f(\theta^*)$, by spectral decomposition we need $\lambda_i\geq \mu$ or $\lambda_i = 0$ for each $i$. If all the eigenvalues of $\nabla^2 f(\theta^*)$ are no smaller than $\mu$, then $\nabla^2 f(\theta^*)$ is a positive definite matrix by definition. For example, $\mu$-strongly convex objective function $f$ satisfies both P-L inequality and $\nabla^2 f(\theta^*)$ being positive definite. If there exists at least one eigenvalue with zero value, then $\nabla^2 f(\theta^*)$ is no longer positive definite. 

On the other hand, positive definite matrix $\nabla^2 f(\theta^*)$ does not necessarily imply P-L inequality. We give a toy example of objective function $f$ that satisfies positive definite matrix $\nabla^2 f(\theta^*)$ while fails to satisfy P-L inequality. For some smooth convex function 
$$f(\theta) = \sqrt{\|\theta\|_2^2 + 1}\geq 1,$$
we know $$f'(\theta)=\frac{\theta}{\sqrt{\|\theta\|_2^2+1}},~~~~ f''(\theta) = \frac{1}{\sqrt{\|\theta\|_2^2+1}} \mI - \frac{1}{(\|\theta\|_2^2+1)^{3/2}}\theta\theta^T$$ 
for any $\theta \in \R^d$. Since $\theta^* = \vzero$, $f(\theta^*) = 1$ and $f''(\theta^*) = \mI$ is a positive definite matrix such that this objective function satisfies our assumption (A2). However, for any $\theta\in\R^d$, there always exists a constant $\epsilon>0$ such that
$$\epsilon (f(\theta)-f(\theta^*)) \geq \|f'(\theta)\|_2^2,$$ 
which fails to satisfy \eqref{eqn:PLinequality}. Therefore, there is no inclusive relationship between P-L inequality and positive definite matrix $\nabla^2 f(\theta^*)$. Both of the conditions can cover different types of functions.


\section{Proof of Theorem \ref{Thm:main_result}}\label{app:main_result}
\subsection{Proof of Theorem \ref{Thm:main_result} (i)}
We first prove the direction that efficiency ordering implies Loewner ordering. For any vector $\vv \triangleq [v_1,v_2,\cdots,v_d]^T \in \R^d\backslash \{\vzero\}$ and vector-valued function $\vf(X) \triangleq [f_1(X), f_2(X), \cdots, f_d(X)]^T$, with $\mSigma(\vf)$ defined in \eqref{eqn:asymptotic covariance matrix} we can get
\begin{equation}\label{eqn:direction_co_matrix}
\begin{split}
    \vv^T \mSigma(\vf) \vv &= \lim_{t\to\infty} \vv^T\mSigma(\vf,t)\vv \\
    &= \lim_{t\to\infty} \frac{1}{t}\E\left\{\vv^T\left[\sum_{s=1}^t\left(\vf(X_s) - \E_{\vpi}\left[\vf(X) \right]\right)\right] \left[\sum_{s=1}^t\left(\vf(X_s) - \E_{\vpi}\left[\vf(X) \right]\right)\right]^T\vv  \right\} \\ 
    &= \lim_{t\to\infty} \frac{1}{t}\E\left\{\left[\sum_{s=1}^t\left(g_{\vv,\vf}(X_s) - \E_{\vpi}\left[g_{\vv,\vf}(X) \right]\right)\right]^2 \right\} \\ 
    &= \sigma^2(g_{\vv,\vf}),
\end{split}
\end{equation}
where function $$g_{\vv,\vf}(X) \triangleq v_1f_1(X) + v_2f_2(X) + \cdots + v_df_d(X)$$
is a linear combination of $f_i(X)$. For two random processes with efficiency ordering and an arbitrary vector-valued function $\vf$, $\sigma^2_X(g_{\vv,\vf}) \leq \sigma^2_Y(g_{\vv,\vf})$ for any vector $\vv$, which is exactly $\vv^T \mSigma_X(\vf) \vv \leq \vv^T \mSigma_Y(\vf) \vv$. Then, by definition of Loewner ordering, we have $\mSigma_X(\vf) \leq_L \mSigma_Y(\vf)$ for any vector-valued function $\vf$.

On the other direction, let $\vv = [1,0,\cdots,0]^T$ and vector-valued function $\vf(X) = [g(X), 0, \cdots, 0]^T$, where $g$ can be any scalar-valued function. Then, \eqref{eqn:direction_co_matrix} can be written as
$\vv^T \mSigma(\vf) \vv = \sigma^2(g)$ and it holds for any scalar-valued function $g$. For two Markov chains $\{X_t\}, \{Y_t\}$ with $\mSigma_X(\vf) \leq_L \mSigma_Y(\vf)$ for any vector-valued function $\vf$, we have $\vv^T \mSigma_X(\vf) \vv \leq \vv^T \mSigma_Y(\vf) \vv$ for any vector $\vv$. Then, with $\vv = [1,0,\cdots,0]^T$ we show that $\sigma^2_X(g) \leq \sigma^2_Y(g)$ for any scalar-valued function $g$, which proves the efficiency ordering.

\subsection{Proof of Theorem \ref{Thm:main_result} (ii)}
We first introduce the closed form of the solution $\mV$ to the Lyapunov equation in Lemma \ref{Lemma:CLT_SGD} and the useful lemma on Loewner ordering.
\begin{lemma}[\citep{chellaboina2008nonlinear} Theorem 3.16 and (3.160)]
If all the eigenvalues of matrix $\mK$ have negative real part, then for every positive-definite matrix $\mU$ there exists a unique positive-definite matrix $\mV$ satisfying $\mU + \mK\mV + \mV\mK^T = \vzero$. The explicit solution $\mV$ is given as
\begin{equation}\label{eqn:solution_lyapunov}
    \mV = \int_0^{\infty} e^{\mK t}\mU e^{(\mK^T)t} dt.
\end{equation}
\end{lemma}

\begin{lemma}[\citep{mitra2010matrix} Theorem 8.2.7]\label{lemma:loewener_order_property}
If two real matrix $\mA,\mB \in \R^{m \times m}$ are Loewner ordered $\mA \leq_L \mB$, then $\mC\mA\mC^T \leq_L \mC\mB\mC^T$ for any real matrix $\mC\in \R^{m\times m}$.
\end{lemma}

From Theorem \ref{Thm:main_result} (i), we know efficiency ordering $\sigma^2_X(g) \leq \sigma^2_Y(g)$ for any scalar-valued function $g$ leads to Loewner ordering $\mSigma_X(\vf) \leq_L \mSigma_Y(\vf)$ for any vector-valued function $\vf$. Consider two random process $\{X_t\}_{t\geq 0}, \{Y_t\}_{t\geq 0}$ with efficiency ordering $X \geq_E Y$, we have $\mSigma_X \leq_L \mSigma_Y$. By Lemma \ref{lemma:loewener_order_property} and \eqref{eqn:Lyapunov_equation}, for any $t$ in \eqref{eqn:solution_lyapunov}, we have $e^{\mK t}\mSigma_X e^{(\mK^T)t} \leq_L e^{\mK t}\mSigma_Y e^{(\mK^T)t}$. Then, for any vector $\vv\in\R^d\backslash\{\vzero\}$, we have 
\begin{equation*}
    \vv^T\mV_X\vv = \int_0^{\infty} \vv^Te^{\mK t}\mSigma_X e^{(\mK^T)t}\vv dt \leq \int_0^{\infty} \vv^Te^{\mK t}\mSigma_Y e^{(\mK^T)t}\vv dt = \vv^T\mV_Y\vv,
\end{equation*}
such that $\mV_X \leq_L \mV_Y$ by definition of Loewner ordering. Similarly, for averaged iterates, we have $\mV'_X \leq_L \mV'_Y$ immediately from Lemma \ref{lemma:loewener_order_property} because $\mSigma_X \leq_L \mSigma_Y$ and $\mV'_X = \mK^{-1}\mSigma_X(\mK^{-1})^T$, $\mV'_Y = \mK^{-1}\mSigma_Y(\mK^{-1})^T$.

\section{Additional Convergence and CLT results for SGD with Constant Step Size and Quadratic Objective Function}\label{appendix:constant_step_size}
Lemma \ref{Lemma:CLT_SGD} has shown the CLT result for general SGD iteration \eqref{eqn:update_rule_general} with diminishing step size. A natural question would be if any CLT result exists for the same SGD iteration with constant step size $\gamma$. For the \textit{i.i.d} inputs and a special case of the iteration
\begin{equation}\label{eqn:35}
    \theta_{t+1} = \theta_t - \gamma(\mA(X_{t+1}) \theta_t - b(X_{t+1})),
\end{equation}
which is usually called linear stochastic approximation in the stochastic approximation literature, it has been studied in \citep{dieuleveut2020bridging,mou2020linear} that $\theta_t$ forms a Markov chain and its time-averaged iterate $\bar{\theta}_t = \frac{1}{t}\sum_{i=0}^{t-1}\theta_i$ converges to the minimizer $\theta^*$ almost surely and a CLT result is given in Theorem 1 \citep{mou2020linear}. However, for Markovian inputs $\{X_t\}_{t\geq 0}$, V. Borkar and S. Meyn mentioned in \citep{borkar2021ode} that the behavior of $(\theta_t,X_t)$ itself is still an open problem under the SGD iteration \eqref{eqn:update_rule_general}. In this part, we propose Lemma \ref{lemma:CLT_cons_stepsize} that studies the special case of the SGD iteration \citep{chen2020explicit} (which studied the diminishing step size) and complements the CLT result for constant step size with time-averaged iterates.

We consider a quadratic objective function 
$$f(\theta) = \frac{1}{n}\sum_{i=1}^n \left(\frac{1}{2}\theta^T\mA\theta - \theta^T\vb(i)\right),$$
where matrix $\mA\in\R^{d\times d}$ is positive definite and vector $\vb(X)\in\R^{d}$ only depends on the state $X\in\cV$ of the Markovian input. Then, the SGD iteration studied in \citep{chen2020explicit} is given as 
\begin{equation}\label{eqn:linear_SA2}
    \theta_{t+1} = \theta_t - \gamma_{t+1}\left(\mA\theta_t - \vb(X_{t+1})\right).
\end{equation}
Here, we study the constant step size $\gamma_t = \gamma, ~\forall t\geq 0$. Define $\bar{\vb} \triangleq \sum_{i\in[n]}\vb(X_i)\pi_i$. The minimizer is given by $\theta^* = \mA^{-1}\bar{\vb}$ such that $\nabla f(\theta^*) = 0$. Then, we have the following CLT result for the SGD update rule \eqref{eqn:linear_SA2} with constant step size and Markovian input $\{X_t\}_{t\geq 0}$.
\begin{lemma}\label{lemma:CLT_cons_stepsize}
Consider the update rule \eqref{eqn:linear_SA2} with positive definite matrix $\mA$ and constant step size $\gamma$ such that $0<\gamma<2/\|\mA\|_2$. Then, for averaged iterates $\bar{\theta}_t = \frac{1}{t}\sum_{i=0}^{t-1}\theta_i$, we have 
\begin{equation}\label{eqn:clt_constant}
    \bar{\theta}_t \xrightarrow[t \to \infty]{a.s.} \theta^*,~~~\text{and}~~~\sqrt{t}(\bar{\theta}_t-\theta^*) \xrightarrow[t \to \infty]{Dist} \cN(0, \mV_{\!\!X}),
\end{equation}
where $\mV_{\!\!X} = \mA^{-1}\mSigma_X(\mA^{-1})^T$ and $\mSigma_X = \lim_{t\to\infty}\frac{1}{t}\E[B_t B_t^T]$, $B_t \triangleq \sum_{s=1}^t (\vb(X_s)-\bar{\vb})$.
\end{lemma}

\begin{proof}
Let $\Tilde{\theta}_t = \theta_t - \theta^*$ and recall $\theta^* = \mA^{-1}\bar{\vb}$, we can rewrite \eqref{eqn:linear_SA2} as
\begin{equation}\label{eqn:39}
    \Tilde{\theta}_{t+1} = \Tilde{\theta}_t - \gamma (\mA \Tilde{\theta}_t - \vb(X_{t+1}) + \bar{\vb}).
\end{equation}
Recursively solving \eqref{eqn:39} gives 
\begin{equation}\label{eqn:recursive_equation}
        \Tilde{\theta}_{t} = (\mI - \gamma \mA)^t \Tilde{\theta}_{0} - \gamma \sum_{i=1}^t (\mI - \gamma \mA)^{t-i} (\vb(X_i)-\bar{\vb}).
\end{equation}

For averaged iterates $\bar{\theta}_t = \frac{1}{t}\sum_{i=0}^{t-1}\theta_i$, \eqref{eqn:recursive_equation} gives
\begin{equation}\label{eqn:pr-averaging}
\begin{split}
    \bar{\theta}_t - \theta^* &= \frac{1}{t}\sum_{i=0}^{t-1}\Tilde{\theta}_t \\
    &= \frac{1}{t} \sum_{i=0}^{t-1} (\mI - \gamma \mA)^i\Tilde{\theta}_0 -  \frac{\gamma}{t}\sum_{i=1}^{t-1} \sum_{j=1}^i (\mI - \gamma \mA)^{i-j} (\vb(X_j)-\bar{\vb}) \\
    &= \frac{1}{t} \sum_{i=0}^{t-1} (\mI - \gamma \mA)^i\Tilde{\theta}_0 -  \frac{\gamma}{t} \sum_{i=1}^{t-1} \left[\sum_{j=0}^{t-i-1} (\mI - \gamma \mA)^j\right](\vb(X_i)-\bar{\vb}) \\
    &= \frac{1}{t}(\gamma\mA)^{-1}(\mI - (\mI -\gamma\mA)^t)\Tilde{\theta}_0 - \frac{\gamma}{t} \sum_{i=1}^{t-1} (\gamma \mA)^{-1}(\mI - (\mI - \gamma\mA)^{t-i}) (\vb(X_i)-\bar{\vb}), 
\end{split}
\end{equation}
where the third equality comes from rearranging the summation order in the second term on the RHS. The fourth equality comes from the fact that $\sum_{i=0}^{t-1} (\mI - \gamma \mA)^i = (\gamma \mA)^{-1}(\mI - (\mI-\gamma\mA)^{t})$. 

Next we want to show $\lim_{t\to\infty}(\mI - \gamma \mA)^t = \vzero$. Since we assume $0<\gamma<2/\|\mA\|_2$, we have $\|\mI - \gamma \mA\|_2 = \max_{i=1,2,\cdots,n} |1-\gamma \lambda_i(\mA)| < 1$, where $\lambda_i(\mA)>0$ is the $i$-the eigenvalue of the positive definite matrix $\mA$. Then, by submultiplicative property, $\|(\mI - \gamma \mA)^t\|_2 \leq \|\mI - \gamma \mA\|_2^t$ such that $\lim_{t\to\infty} \|(\mI - \gamma \mA)^t\|_2 \leq \lim_{t\to\infty} \|\mI - \gamma \mA\|_2^t = 0$, which implies that $\lim_{t\to\infty}(\mI - \gamma \mA)^t = \vzero$.

Now we want to show $\lim_{t\to\infty}\|\sum_{i=1}^{t-1} (\mI-\gamma \mA)^{t-i}(\vb(X_i)-\bar{\vb})\|_2 < \infty$. Since vector-valued function $\vb(\cdot)$ is defined on the finite state space $\cV$, it is safe to assume $\|\vb(X)-\bar{\vb}\|_2 \leq C$ for some constant $C$. Then,
\begin{equation}\label{eqn:42}
\begin{split}
        \lim_{t\to\infty}\left\|\sum_{i=1}^{t-1} (\mI-\gamma \mA)^{t-i}(\vb(X_i)-\bar{\vb})\right\|_2 &\leq \lim_{t\to\infty}\sum_{i=1}^{t-1} \|(\mI-\gamma \mA)^{t-i}\|_2\|(\vb(X_i)-\bar{\vb})\|_2 \\
        &\leq C \lim_{t\to\infty}\sum_{i=1}^{t-1} \|\mI-\gamma\mA\|_2^{t-i} \\
        &= C \lim_{t\to\infty}\sum_{i=1}^{t-1} \|\mI-\gamma\mA\|_2^{i} < \infty,
\end{split}
\end{equation}
where the first inequality comes from submultiplicative property and triangular inequality, the first equality is by rewriting the index inside the summation, and the third inequality comes from the fact that $\|\mI-\gamma\mA\|_2 < 1$. Then, we have
\begin{equation}\label{eqn:43}
    \lim_{t\to\infty} \frac{1}{t}\sum_{i=1}^{t-1} (\mI-\gamma \mA)^{t-i}(\vb(X_i)-\bar{\vb}) = \vzero,
\end{equation}
and
\begin{equation}\label{eqn:44}
    \lim_{t\to\infty} \frac{1}{\sqrt{t}}\sum_{i=1}^{t-1} (\mI-\gamma \mA)^{t-i}(\vb(X_i)-\bar{\vb}) = \vzero.
\end{equation}

With $\lim_{t\to\infty}(\mI - \gamma \mA)^t = \vzero$ and \eqref{eqn:43}, we have from \eqref{eqn:pr-averaging} that
\begin{equation}
    \lim_{t\to\infty} \bar{\theta}_t - \theta^* = \lim_{t\to\infty} -\frac{\gamma}{t}\sum_{i=1}^{t-1} (\gamma \mA)^{-1}(\vb(X_i)-\bar{\vb})= -\mA^{-1} \lim_{t\to\infty} \frac{1}{t}\sum_{i=1}^{t-1}(\vb(X_i)-\bar{\vb}).    
\end{equation}
From the ergodic theorem for Markov chains (\citep{bremaud2013markov} Theorem 3.3.2), we have $\lim_{t\to\infty} \frac{1}{t}\sum_{i=1}^{t-1}(\vb(X_i)-\bar{\vb}) = \vzero$ and therefore $\lim_{t\to\infty} \bar{\theta}_t = \theta^*$.

To get the CLT result in \eqref{eqn:clt_constant}, we first scale $\bar{\theta}_t - \theta^*$ from \eqref{eqn:pr-averaging}, along with \eqref{eqn:44}, such that
\begin{equation}
\begin{split}
        \lim_{t\to\infty} \sqrt{t}(\bar{\theta}_t - \theta^*) &= \lim_{t\to\infty} -\frac{\gamma}{\sqrt{t}}\sum_{i=1}^{t-1} (\gamma \mA)^{-1}(\vb(X_i)-\bar{\vb}) \\ 
        &= -\mA^{-1} \lim_{t\to\infty} \frac{\sqrt{t-1}}{\sqrt{t}} \left(\frac{1}{\sqrt{t-1}}\sum_{i=1}^{t-1} (\vb(X_i)-\bar{\vb})\right).
\end{split}
\end{equation}
From the CLT of Markov chain in Theorem \ref{Thm:CLT multivariate mcmc}, we know $\frac{1}{\sqrt{t}}\sum_{i=1}^t (\vb(X_i)-\bar{\vb}) \xrightarrow[t \to \infty]{dist}  \cN(0,\mSigma_X)$, where $\mSigma_X = \lim_{t\to\infty} \frac{1}{t}\E[(\sum_{s=1}^t (b(X_i)-\bar{b}))(\sum_{s=1}^t (b(X_i)-\bar{b}))^T]$. This result shows that time-averaged iterate $\bar{\theta}_t$ will guarantee the convergence to the exact solution and we have CLT result for $\sqrt{t}(\bar{\theta}_t - \theta^*)$ too.

Finally, we need to quantify the covariance matrix in the CLT result to $\sqrt{t}(\bar{\theta}_t - \theta^*)$. We will look at $\lim_{t\to\infty} t\E[(\bar{\theta}_t-\theta^*)(\bar{\theta}_t-\theta^*)^T]$. Note that the second term in \eqref{eqn:pr-averaging} is bounded (see \eqref{eqn:42} for the proof) such that the cross term in the outer-product of $\bar{\theta}_t-\theta^*$ will vanish when $t\to\infty$. Then, we have 
\begin{equation}\label{eqn:45}
\begin{split}
       &\lim_{t\to\infty} t(\bar{\theta}_t-\theta^*)(\bar{\theta}_t-\theta^*)^T \\ =& \lim_{t\to\infty}\frac{t-1}{t}\left(-\frac{1}{\sqrt{t-1}}\mA^{-1} \sum_{i=1}^{t-1} (\vb(X_i)-\bar{\vb})\right)\left(-\frac{1}{\sqrt{t-1}}\mA^{-1} \sum_{i=1}^{t-1} (\vb(X_i)-\bar{\vb})\right)^T\\ 
       =& \mA^{-1} \lim_{t\to\infty} \frac{1}{t} \left(\sum_{i=1}^{t-1} (\vb(X_i)-\bar{\vb})\right)\left(\sum_{i=1}^{t-1} (\vb(X_i)-\bar{\vb})\right)^T (\mA^{-1})^T.
\end{split}
\end{equation}
Taking the expectation of \eqref{eqn:45} gives
\begin{equation}
\begin{split}
        &\lim_{t\to\infty} t\E[(\bar{\theta}_t-\theta^*)(\bar{\theta}_t-\theta^*)^T] \\ =& \mA^{-1} \lim_{t\to\infty} \frac{t-1}{t}\E\left[\frac{1}{t-1}\left(\sum_{i=1}^{t-1} (\vb(X_i)-\bar{\vb})\right)\left(\sum_{i=1}^{t-1} (\vb(X_i)-\bar{\vb})\right)^T\right] (\mA^{-1})^T\\ 
        =& \mA^{-1} \mSigma_X (\mA^{-1})^T.
\end{split}
\end{equation}
Therefore, we have
$$\sqrt{t}(\bar{\theta}_t - \theta^*) \xrightarrow[t \to \infty]{dist} \cN(0,\mA^{-1} \mSigma_X (\mA^{-1})^T).$$
\end{proof}

In Lemma \ref{lemma:CLT_cons_stepsize}, $\mA = \nabla^2 f(\theta)$ and $\mSigma_X$, by definition \eqref{eqn:asymptotic covariance matrix}, is an asymptotic covariance matrix of the Markov chain $\{X_t\}_{t\geq 0}$ for vector-valued function $\vb(\cdot)$. Therefore, \eqref{eqn:clt_constant} shares a similar form to \eqref{eqn:error_pr_averaging} in Lemma \ref{Lemma:CLT_SGD}. Our Theorem \ref{Thm:main_result} can be carried over to Lemma \ref{lemma:CLT_cons_stepsize}, which enables us to compare the efficiency ordering of SGD algorithms driven by different stochastic inputs under the update rule \eqref{eqn:linear_SA2} and constant step size.
\section{Proof of Proposition \ref{cor:reweight_nonMarkov}}\label{app:apply_nbrw}

\citep{neal2004improving,lee2012beyond} proposed the guidance by modifying a reversible random walk into a non-Markovian random walk to achieve higher sampling efficiency and it was applied to other applications to improve sampling efficiency (e.g., \citep{li2015random,li2019walking}). Specifically speaking, consider a reversible random walk $\{X_t \}_{t\geq 0}$ (e.g., SRW) with transition matrix $\mP$ and stationary distribution $\vpi$. Let its counterpart (e.g., NBRW) on the augmented state space be given by $\{ Z_t \}_{t\geq 0} \triangleq \{ (Y_{t-1}, Y_{t})\}_{t \geq 0}$, where $Y_{t-1},Y_t \in \cV$ and $Z_0 = (Y_0,Y_0)$. Additionally, $\{Z_t\}_{t\geq 0}$ is \textit{a Markov chain on the augmented state space} 
$$\cE \triangleq \{(i,j): i,j \in \cV ~~s.t.~~ P(i,j) > 0\} \subseteq \cV \times \cV$$ 
with stationary distribution $\vpi'$. For notation simplicity, we use $e_{ij}$ to represent edge $(i,j)$. Note that by definition $e_{ij} \neq e_{ji}$ and we allow $i=j$ if $P(i,j) > 0$, which is a bit different from the edge set that does not include edge $(i,i)$. As proved in \citep{lee2012beyond}, the properties of NBRW $\{Z_t\}_{t\geq 0}$ are detailed in the following theorem.

\begin{theorem}[\citep{neal2004improving} Theorem $2$]\label{theorem:augmented_mc}
Suppose that $\{X_t\}$ is an irreducible, reversible Markov chain on the state space $\cV = \{1,2,\cdots,n\}$ with transition matrix $\mP = \{P(i,j)\}$ and stationary distribution $\vpi$. Construct a Markov chain $\{Z_t\}$ on the augmented state space $\cE$ with transition matrix $\mP' = \{P'(e_{ij},e_{lk})\}$ in which the transition probabilities $P'(e_{ij},e_{lk})$ satisfy the following two conditions: for all $e_{ij},e_{ji}, e_{jk}, e_{kj} \in \cE$ with $i \neq k$, 
\begin{subequations}
\begin{equation}\label{eqn:quasi_balance_equation}
    P(j,i)P'(e_{ij},e_{jk}) = P(j,k)P'(e_{kj},e_{ji}),
\end{equation}
\begin{equation}
    P'(e_{ij},e_{jk}) \geq P(j,k).
\end{equation}
\end{subequations}
Then, the Markov chain $\{Z_t\}_{t\geq 0}$ is irreducible and non-reversible with a unique stationary distribution $\vpi'$ in which
\begin{equation}\label{eqn:stationary_augmented}
    \pi'(e_{ij}) = \pi_i P(i,j) = \pi_j P(j,i), ~~e_{ij} \in \cE.
\end{equation}
Also, for any scalar-valued function $g$, the asymptotic variance $\sigma^2_Z(g)\leq \sigma^2_X(g)$.
\end{theorem}

Now, we show how the non-Markov random walk with properties in Theorem \ref{theorem:augmented_mc} can be included in Lemma \ref{Lemma:CLT_SGD}. For the original function $G: \R^d \times \cV \to \R^d$, we define another function $\Phi: \R^d \times \cE \to \R^d$ such that $\Phi(\theta, e_{ij}) = G(\theta, j)$. 
Then, the SGD update rule \eqref{eqn:update_rule_general} becomes $\theta_{t+1} = Proj_{\Theta}\left(\theta_t -\gamma_{t+1} \nabla \Phi\left(\theta_t, Z_{t+1}\right)\right)$.
From \eqref{eqn:stationary_augmented}, we have for any $\theta \in \R^d$,
\begin{equation}\label{eqn:mean_field_function_equal}
    \phi(\theta) \!\triangleq\! \E_{Z \sim \vpi'} \Phi(\theta,Z) \!= \!\!\!\sum_{e_{ij} \in \cE}\!\!\Phi(\theta,e_{ij}) \pi'(e_{ij}) \!=\!\!\!\sum_{i,j \in \cV}\!\! G(\theta, j) \pi_j P(j,i) \!=\!\! \sum_{j \in \cV}\frac{1}{n\pi_j} F(\theta, j) \pi_j = f(\theta),
\end{equation}
showing the mean-field function $\phi(\theta)$ for $\Phi(\theta,Z)$ is the same as the objective function $f(\theta)$. 

Next, we show assumptions (A1)-(A5) of Lemma \ref{Lemma:CLT_SGD} still hold for $\{Z_t\}_{t\geq 0}$ on the augmented state space $\cE$ and function $\Phi$. Assumption (A1), (A5) and (A3) hold for function $\Phi(\theta,Z)$ because of our definition $\Phi(\theta, e_{ij}) = G(\theta, j)$. Assumption (A4) is satisfied because from Theorem \ref{theorem:augmented_mc}, $\{Z_t\}_{t\geq 0}$ is an irreducible and non-reversible Markov chain on the augmented state space $\cE$ and there always exists a solution \eqref{eqn:solution_poisson_equation} to \eqref{eqn:poisson_equation}. Assumption (A2) holds because matrix $\mK = \nabla^2 \phi(\theta^*) = \nabla^2 f(\theta^*)$ by \eqref{eqn:mean_field_function_equal}. Therefore, we can say the Markov chain $\{Z_t\}_{t\geq 0}$ on the augmented state space $\cE$, along with the newly defined function $\Phi$, can apply Lemma \ref{Lemma:CLT_SGD}. The asymptotic covariance matrix $\mSigma_{Z} \triangleq \mSigma_Z(\nabla \Phi(\theta^*,\cdot))$ is given as
\begin{equation}\label{eqn:ztoy}
\begin{split}
        \mSigma_Z \!&=\! \text{Var}_{\vpi'}(\nabla \Phi(\theta^*,Z_0)) \!\!+\!\! \sum_{k\geq 1}\!\text{Cov}_{\vpi'}(\nabla \Phi(\theta^*,Z_0),\!\nabla \Phi(\theta^*,Z_k))\text{Cov}_{\vpi'}(\nabla \Phi(\theta^*,Z_0),\!\nabla \Phi(\theta^*,Z_k))^T\\ 
        &=\! \lim_{t \rightarrow \infty} \frac{1}{t}\E\left\{\!\left[\sum_{s=1}^t\left(\nabla \Phi(\theta^*, Z_s) \!-\! \E_{\vpi'}(\nabla \Phi(\theta^*,\cdot))\right)\right]\! \left[\sum_{s=1}^t\left(\nabla \Phi(\theta^*, Z_s) \!-\! \E_{\vpi'}(\nabla \Phi(\theta^*,\cdot))\right)\right]^T  \right\} \\
        &= \!\lim_{t \rightarrow \infty} \frac{1}{t}\E\left\{\!\left[\sum_{s=1}^t\left(\nabla G(\theta^*, Y_s) \!-\! \E_{\vpi}(\nabla G(\theta^*,\cdot))\right)\right]\! \left[\sum_{s=1}^t\left(\nabla G(\theta^*, Y_s) \!-\! \E_{\vpi}(\nabla G(\theta^*,\cdot))\right)\right]^T  \right\} \\
        &= \!\mSigma_{Y}(\nabla G(\theta^*,\cdot)),
\end{split}
\end{equation}
where the third equality comes from \eqref{eqn:stationary_augmented} because
\begin{equation*}
    \E_{\vpi'}[(\nabla \Phi(\theta^*,Z))] = \nabla f(\theta^*) = \sum_{j \in \cV}\frac{1}{n\pi_j} \nabla F(\theta, j) \pi_j = \sum_{j \in \cV} \pi_j \nabla G(\theta,j) = \E_{\vpi}(\nabla G(\theta^*,\cdot)).
\end{equation*}
$\{Y_t\}_{t\geq 0}$ on the node space $\cV$ is the trajectory generated by $\{Z_t\}_{t\geq 0}$ on the augmented state space $\cE$. Let $\mV_Z$ be the covariance matrix generated by the SGD algorithm driven by $\{Z_t\}_{t\geq 0}$. Denote $\mSigma_X \triangleq \mSigma_X(\nabla G(\theta^*,\cdot))$ the asymptotic covariance matrix and $\mV_X$ the covariance matrix in \eqref{eqn:Lyapunov_equation} from the original Markov chain $\{X_t\}$. Then, from Theorem \ref{theorem:augmented_mc} we know the asymptotic variances of NBRW and SRW are ordered for any scalar-valued function. Then, with Theorem \ref{Thm:main_result} (i) we know that the asymptotic covariance matrices of NBRW and SRW are Loewner ordered for any vector-valued function such that $\mSigma_Y(\nabla G(\theta^*,\cdot))\leq_L \mSigma_X(\nabla G(\theta^*,\cdot))$. From \eqref{eqn:ztoy} we have $\mSigma_{Z}\leq_L \mSigma_X$. By applying Theorem \ref{Thm:main_result} (ii), we have $\mV_Z\leq_L \mV_X$.

\section{Proof of Lemma \ref{Lem:Shuffling finite time var}}
Assume we have a vector-valued function $\vg: [n] \to \R^d$. For shuffling without replacement, which traverses every node in each epoch with length $n$, we group all the terms in each $k$-th epoch (shown in Figure \ref{fig:Diagram_G}) and analyze the term $\sum_{i = (k-1)n}^{kn-1} (\vg(X_i) - \E_{\vpi}(\vg)) < \infty$ for $k\in \Z_+$. Note that we have 
\begin{equation}\label{eqn:shuffling_equiv}
    \sum_{i = (k-1)n}^{kn-1} \vg(X_i) = \sum_{j=1}^n \vg(j)
\end{equation}
by definition of shuffling without replacement. By \eqref{eqn:shuffling_equiv} and $\E_{\vpi}(\vg) = \frac{1}{n}\sum_{i=1}^n \vg(i)$, we have
\begin{equation}\label{eqn:shuffling_zero}
    \sum_{i = (k-1)n}^{kn-1} (\vg(X_i) - \E_{\vpi}(\vg)) = \sum_{j=1}^n \vg(j) - n\sum_{i=1}^n \frac{1}{n}\vg(i) = 0.
\end{equation}
\begin{figure}
    \centering
    \includegraphics[width = 0.8\textwidth]{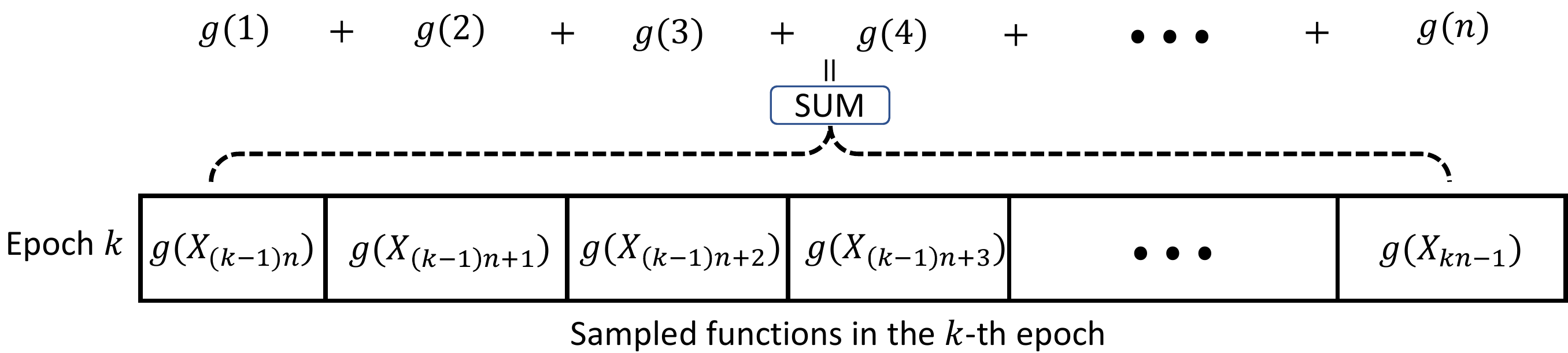}
    \caption{Diagram of the sampled functions in each epoch.}
    \label{fig:Diagram_G}
\end{figure}

With the definition of the asymptotic covariance matrix, we have
\begin{equation}\label{eqn:cov_shuffling}
\begin{split}
        \mSigma_X(\vg) &=\lim_{t\to\infty} \frac{1}{t}\E\left\{\left(\sum_{i=1}^t(\vg(X_s) \!-\! \E_{\vpi}(\vg))\right)\left(\sum_{i=1}^t(\vg(X_s) \!-\! \E_{\vpi}(\vg))\right)^T\right\} \\
        &= \lim_{t\to\infty} \frac{1}{t} \E \Bigg\{\!\!\left[\sum_{i=1}^{s} (\vg(X_i) \!-\! \E_{\vpi}(\vg)) \!+\! \sum_{i = s+1}^{t} \left(\vg(X_i) \!-\! \E_{\vpi}(\vg)\right) \right] \\
        &~~~~~~~~~~~~~~~~~~~~~\cdot\left[\sum_{i=1}^{s} (\vg(X_i) \!-\! \E_{\vpi}(\vg)) \!+\! \sum_{i = s+1}^{t} \left(\vg(X_i) \!-\! \E_{\vpi}(\vg)\right) \right]^{\!\!T}\!\Bigg\} \\
        &= \lim_{t\to\infty} \frac{1}{t} \E \left\{\left[\sum_{i = s+1}^{t} \left(\vg(X_i) \!-\! \E_{\vpi}(\vg)\right) \right]\left[\sum_{i = s+1}^{t} \left(\vg(X_i) \!-\! \E_{\vpi}(\vg)\right) \right]^{T}\right\},
\end{split}
\end{equation}
where $\vpi$ is the uniform stationary distribution, $s\triangleq t - (t ~\text{mod}~ n)$ is the time at which the previous epoch ended before time $t$ and we let $\sum_{i=t+1}^t (\vg(X_i) \!-\! \E_{\vpi}(\vg)) = 0$ by default. The third equality in \eqref{eqn:cov_shuffling} comes from \eqref{eqn:shuffling_zero}. Note that there always exists a constant $D$ such that $|\!|\vg(i)-\E_{\vpi}{\vg}|\!|_2<D$ for any $i\in[n]$ because of the boundedness of function $\vg$. Then, we have for any $t$,
$$\left\|\sum_{i = s+1}^{t} (\vg(X_i) \!-\! \E_{\vpi}(\vg))\right\|_2 \leq \sum_{i=s+1}^{t} \|\vg(X_i) \!-\! \E_{\vpi}(\vg)\|_2 < (t-s-1)D < nD < \infty,$$ 
where the second last inequality holds since $t-s-1 = (t \mod n) - 1 < n$. Back to \eqref{eqn:cov_shuffling}, we have 
$$|\!|\mSigma_X(\vg)|\!|_2 \leq  \lim_{t\to\infty} \frac{1}{t}\E\left[\left\|\sum_{i = s+1}^{t} (\vg(X_i) \!-\! \E_{\vpi}(\vg))\right\|_2^2\right] \leq \lim_{t\to\infty} \frac{nD}{t} = 0,$$
where the first inequality comes from Jensen's inequality. Finally, we have $\mSigma_X(\vg) = \vzero$ such that the asymptotic covariance matrices for both random and single shuffling are zero for any vector-valued function $\vg$.

\section{Proof of Proposition \ref{Cor:Shuffling vs iid} and Extension to Mini-batch Gradient Descent}\label{app: generalized_minibatch_shuffling_proof}

\subsection{Single Shuffling in SGD CLT Analysis}
Single shuffling is seen as a time-homogeneous, irreducible, periodic Markov chain and we know (A4) is the only requirement for Markov chain in the CLT result. As mentioned in \citep{glynn1996liapounov} and \citep{meyn2012markov} Chapter 17, the necessary condition to ensure the existence of function $\Tilde{F}$ as in \eqref{eqn:solution_poisson_equation} is that the inverse $(\mI - \mP + \vone\vpi^T)^{-1}$ exists. This is true for periodic Markov chain, and is shown in the following lemma.
\begin{lemma}\label{lemma:periodic_chain_in_CLT}
The solution \eqref{eqn:solution_poisson_equation} to the Poisson equation \eqref{eqn:poisson_equation} exists for an underlying finite, irreducible periodic Markov chain with transition matrix $\mP\in\R^{m\times m}$, stationary distribution $\vpi$ and period $n\leq m$.
\end{lemma}
\begin{proof}
From Perron–Frobenius theorem for irreducible, non-negative stochastic matrices \citep{seneta2006non}, we know there are $n$ complex eigenvalues uniformly distributed on the unit circle, including the \textit{unique} eigenvalue with value $1$.  Other $m-n$ eigenvalues fall inside the unit circle, but still uniformly distributed on some circles with absolute value strictly smaller than $1$ because transition matrix $\mP$ is similar to $e^{i\omega}\mP$ where $i=\sqrt{-1}$ and $\omega = 2\pi/n$. 

Denote $\lambda_1,\lambda_2,\cdots,\lambda_m=1$ the eigenvalues of the transition matrix $\mP$ and let $\mJ$ be the Jordan norm form. There exists an invertible matrix $\mQ = [\vu_1,\vu_2,\cdots,\vu_m]^T$ and $\mQ^{-1} = [\vv_1,\vv_2,\cdots,\vv_m]$ such that $\mP = \mQ \mJ \mQ^{-1}$ and $\vu_i^T\vv_i = 1$ for all $i\in [m]$. In particular, $\vu_m = \vpi$ and $\vv_m = \vone$. Now, $(\lambda_m=1, \vu_m, \vv_m)$ is also the PF eigenpair for the matrix $\mathbf{\Pi}$, which enables us to write down $\vone\vpi^T = \mQ \mathbf{\Lambda}\mQ^{-1}$, where $\mathbf{\Lambda} = \text{diag}(0,0,\cdots,1)$. Therefore, $\mI - \mP + \vone\vpi^T = \mQ(\mI - \mJ + \mathbf{\Lambda})\mQ^{-1}$. Note that $\mI - \mJ + \mathbf{\Lambda}$ is a new Jordan norm form with \textbf{non-zero} entries on the main diagonal. Assume $\mJ_i = \lambda_i\mI + \mN$ is one of its Jordan block with nilpotent matrix $\mN$ such that $\mN^{p_i} = 0$ for some $p_i\geq 2$, then 
$$\mJ_i^{-1} = \lambda_i^{-1}(\mI + \lambda_i^{-1}\mN)^{-1} = \lambda_i^{-1}(\mI - \lambda_i^{-1}\mN + \cdots + (\lambda_i)^{-p_i+1}\mN^{p_i-1}),$$ 
showing that $\mJ_i^{-1}$ exists for all Jordan blocks in the new Jordan norm form $\mI - \mJ + \mathbf{\Lambda}$ because all $\lambda_i$ are non-zero. Therefore, $(\mI - \mP + \vone\vpi^T)^{-1}$ exists and \eqref{eqn:solution_poisson_equation} holds.
\end{proof}
With Lemma \ref{lemma:periodic_chain_in_CLT}, single shuffling can indeed be included in the SGD CLT result, and its covariance matrix is the zero matrix $\vzero$ (from Lemma \ref{Lem:Shuffling finite time var}).
Random shuffling is a \textit{time-inhomogeneous} Markov chain due to its nature of reshuffling at the beginning of each epoch. Before providing our main proof, we first present the augmentation of the random shuffling sequence which transforms it into a \textit{time-homogeneous} periodic Markov chain on the augmented state space.
\subsection{Augmentation of Random Shuffling for CLT Analysis}\label{app:augmentation_random_shuffling}
By the definition of random shuffling, in each epoch of length $n$, the sampler traverses one permutation sequence drawn uniformly at random from the permutation sequence set with size $n!$. Due to its random nature across each epoch, random shuffling is not a Markov chain on state space $[n]$. In order to include random shuffling in the SGD CLT result, Lemma \ref{Lemma:CLT_SGD}, we need to transform it into a Markov chain on a augmented state space. 

Let $\{X_t\}_{t\geq 0}$ be the sequence generated by random shuffling. We first define an augmented state space $\cS$, where for each state $s_t \triangleq \{\{A_j^{(t)}\}_{j\in[n]},c_t\}\in\cS$, the sequence $\{A_j^{(t)}\}_{j\in[n]}\triangleq \{X_{t-n+1},X_{t-n+2},\cdots,X_t\}$ is of length $n$, and records the history of past $n$ indices until time $t$. The integer $c_t\in\{1,2,\cdots,n\}$ is the time spent in current epoch at time $t$. For examples, consider the state space to be $\cS = \{1,2,\cdots,6\}$ in total and assume the sequence of visited states until $t=8$ is $\{3,6,2,1,5,4,2,5\}$. Here, $\{3,6,2,1,5,4\}$ is one complete permutation sequence in the first epoch and $\{2,5\}$ are in the second epoch. At time $t=8$, the sampler is at index $5$ and the sequence of past $6$ indices is $\{2,1,5,4,2,5\}$ and $c_t=2$, such that $s_8 = \{\{2,1,5,4,2,5\},2\}$. In the next iteration $t=9$, the sequence will be $\{1,5,4,2,5,X\}$ and $c_t=3$, where $X$ is the index that can be chosen from $\{1,3,4,6\}$ uniformly at random because $\{2,5\}$ have been chosen in the current epoch. Then, we have
\begin{equation}
    s_9 = \begin{cases} \{\{1,5,4,2,5,1\},3\} & \text{w.p}~~ 1/4, \\ \{\{1,5,4,2,5,3\},3\} & \text{w.p}~~ 1/4, \\ \{\{1,5,4,2,5,4\},3\} & \text{w.p}~~ 1/4, \\ \{\{1,5,4,2,5,6\},3\} & \text{w.p}~~ 1/4. \end{cases}
\end{equation}
Assume $s_{12} = \{{2,5,6,1,3,4},6\}$ at $t=12$, the next state $s_{13} = \{\{5,6,1,3,4,X\},1\}$ and $X$ is chosen from $\{1,2,\cdots,6\}$ uniformly at random. 

Note that 
\begin{itemize}
    \item We only include \textit{proper combination} of sequence $\{A_j\}_{j\in[n]}$ in the augmented state space, where `proper' means the sequence is possible to appear with the current value of $c_t$. For instance, $\{\{2,1,5,4,2,2\},2\}$ or $\{\{2,1,5,1,2,5\},2\}$ is inproper because $\{2,2\}$ or $\{2,1,5,1\}$ doesn't exist in the permutation sequence in one epoch.
    \item Transition probability $P(s_{t},s_{t+1})$ is possibly non-zero \textbf{only} when $c_{t+1} = c_t + 1$ for $c_t\leq n-1$, or $c_{t+1} = 1$ when $c_t = n$.
\end{itemize}
Next, we show the proposition that will be used later to show that random shuffling can also be fitted into the CLT result.
\begin{proposition}\label{prop:random_shuffling_markov_chain}
$\{s_t\}_{t\geq 0}$ forms a finite, irreducible and periodic Markov chain with period $n$.
\end{proposition}
\begin{proof}
By our construction, the size of choice of $\{A_j\}_{j\in[n]}$ with $c=i$ is $(C_n^i)^2 i!(n-i)!$, because the first $i$ indices has $C_n^i i!$ choices and remaining sequence has $C_n^{n-i}(n-i)!$ choices. The size of the augmented state space is $\sum_{i=1}^n (C_n^i)^2 i!(n-i)!$ and is still \textit{finite}. 

The \textit{irreducibility} can be shown by $P(s_{t+2n}=s'|s_t=s)> 0$ because we can always construct two permutation sequences in two epochs; one including first $i$ indices of $\{A'_j\}_{j\in[n]}$ in state $s'$ and the other including the remaining sequences. 

For \textit{periodicity}, if the sequence $\{A_j^{(t)}\}_{j\in[n]}$ in the current state $s_t=s\in\mathcal{S}$ includes repeated index at time $t$, e.g., index $i\in\{X_{t-n+1},\cdots,X_m\}$ (in the $\frac{m}{n}$-th epoch) and $i\in \{X_{m+1},\cdots,X_t\}$ (in the $(\frac{m}{n}+1)$-th epoch), then $i\notin \{X_{t+1},\cdots,X_{m+n}\}$ due to the nature of shuffling without replacement in an epoch, which leads to $P(s_{t+n}=s|s_t=s) = 0$. In addition, for $k\geq 2$ we can always construct intermediate sequences $\{X_{t+1},\cdots,X_{t+(k-1)n}\}$ such that $\{X_{t-n+1},\cdots,X_m\} = \{X_{t+(k-1)n+1},\cdots, X_{m+kn}\}$ and $\{X_{m+1},\cdots,X_t\} = \{X_{m+kn+1},\cdots,X_{t+kn}\}$, implying that $P(s_{t+kn}=s|s_t=s) > 0$ for $k\geq 2$. On the other hand, if $\{A_j^{(t)}\}_{j\in[n]}$ does not include repeated index, $P(s_{t+kn}=s|s_t=s)> 0$ for $k\in\mathbb{N}$. We also note that $P(s_{t+j}=s|s_t=s) = 0$ for $s\in\mathcal{S}$, $j\neq kn$ and $k\in\mathbb{N}$. Since $\{c_t\}$ by its definition is a periodic sequence $\{1,2,\cdots,n,1,2,\cdots,n,1,2,\cdots\}$ with period of length $n$, we know that $c_t=c_{t+j}$ holds only when $j=kn$ for $k\in\mathbb{N}$. Then, for $j\neq kn$, we have $c_{t+j}\neq c_{t}$ such that $s_{t+j} \neq s_t$, which leads to $P(s_{t+j}=s|s_t=s)=0$ for $j\neq kn$ and $k\in\mathbb{N}$. Therefore, by definition of periodicity, the Markov chain is of period $n$. 
\end{proof}

Together with Lemma \ref{lemma:periodic_chain_in_CLT} and Proposition \ref{prop:random_shuffling_markov_chain}, we can see random shuffling can also be include in the SGD CLT.

\subsection{Extension to Mini-batch Gradient Descent}\label{app:appendixH3}
Mini-batch gradient descent is another popular gradient descent variant and is widely used in the machine learning tools \citep{chollet2015keras,abadi2016tensorflow,paszke2019pytorch} to accelerate the learning process when compared to SGD. Instead of sampling a single element, mini-batch gradient descent samples multiple elements from $[n]$ in each iteration that form a batch.  

To incorporate the notion of mini-batches in our SGD framework, we provide a reformulation of the general SGD iteration based on a similar formulation in \citep{gower2019sgd} for the general analysis of SGD with \textit{i.i.d} inputs. Consider a stochastic process $\{B_t\}_{t \geq 0}$ as the driving sequence, which randomly samples batches of size $S$ (without replacement) from the state space $[n]$, that is $B_t\subset[n]$ and $|B_t| = S$ for all $t\geq0$. Here we assume $[n]\bmod S=0$ for simplicity. $B_t$ will therefore refer to the batch chosen at any time $t>0$. We assume that $B_t$ for all $t>0$ are \textit{i.i.d }random variables drawn from a distribution $\cP$, such that $\cP(B)>0$ is the probability with which a batch $B \subset [n]$ is picked.
We associate with any batch $B$, $\vv(B) \triangleq \left[ \sum_{i \in B}\ve_i \right]/ {\binom{N}{S} \cP(B)}$, where $\ve_i$ is the $i$'th vector of the canonical basis of $\R^d$. We then denote $\vF(\theta) \triangleq [F(\theta,1), \cdots, F(\theta,n)]^T$, and $\nabla \vF(\theta) \triangleq [\nabla F(\theta,1), \cdots , \nabla F(\theta,n)]^T$ for all $\theta \in \Theta$. With this notation, we can rewrite the general update rule for mini-batch SGD as
\begin{equation}\label{eqn:update rule mini-batch}
    \theta_{t+1} = \text{Proj}_{\Theta}\left(\theta_t - \gamma_{t+1} \nabla \vF(\theta_t)^T\vv(B_{t+1})\right). 
\end{equation}
Note that this way of defining the mini-batch based random input ensures that $\E_{\cP}[\vF(\theta)^T\vv(\cdot)] = f(\theta)$ for all $\theta \in \Theta$, maintaining the same objective function irrespective of the distribution from which batches are sampled.

With $X_t = B_t$ for all $t \geq 0$, and $\nabla G(\theta_t,X_{t+1}) = \nabla \vF(\theta_t)^T\vv(B_{t+1})$, the iteration \eqref{eqn:update rule mini-batch} can still be written in the form of \eqref{eqn:update_rule_general} with \textit{i.i.d} input sequence $\{X_t\}_{t\geq 0}$. We can thus apply the CLT for SGD algorithms to the mini-batch SGD with \textit{i.i.d} input, and in a similar fashion as \eqref{eqn:cov_matrix_iid} derive the explicit form of the asymptotic covariance matrix of \eqref{eqn:update rule mini-batch}, that is,
\begin{equation}\label{eqn:asymptotic covariance mini-batch}
    \mSigma_B(\nabla \vF(\theta^*)^T\vv(\cdot)) \triangleq \text{Var}_{B_0 \sim \cP}(\nabla \vF(\theta^*)^T\vv(B_0)).
\end{equation}

In practice, mini-batch gradient descent with shuffling is more widely used than $i.i.d$ sampling \citep{abadi2016tensorflow}, in which $B_t$ is generated by shuffling-based method instead of independent drawn from a distribution.\footnote{The reformulation \eqref{eqn:update rule mini-batch} enables us to analyze mini-batch gradient descent with various stochastic processes that samples $B_t$, not just i.i.d input and shuffling. However, discussing general processes $\{B_t\}_{t\geq0}$ is beyond the scope of this paper. } At the beginning of each epoch, \textit{Mini-batch gradient descent with random shuffling} shuffles the whole dataset $[n]$ and split it into small batches. On the other hand, \textit{mini-batch gradient descent with single shuffling} only shuffles the dataset $[n]$ once before dividing it into batches, sticking to a predetermined sequence of batches for all epochs of the training process. As pointed out by \citep{yun2021minibatch}, there is still a gap between practical implementation and theoretical analysis for mini-batch gradient descent with shuffling. Nevertheless, by extrapolating the analysis from Proposition \ref{Cor:Shuffling vs iid}, we are able to analyze the efficiency ordering of shuffling and \textit{i.i.d} sampling in the mini-batch version, as stated next.

\begin{proposition}\label{Cor:mini-batch with shuffling}
Consider the mini-batch gradient descent \eqref{eqn:update rule mini-batch} with stochastic inputs single/random shuffling $\{X_t\}_{t\geq 0}$ and \textit{i.i.d} sampling $\{Y_t\}_{t\geq 0}$, we have $\theta^X_t, \theta^Y_t \xrightarrow[t \to \infty]{a.s.} \theta^*$ and $\mV_X = \vzero \leq_L \mV_Y$.
\end{proposition}\vspace{-2mm}
\begin{proof}
Let $l \triangleq n/B \in \N$. We first give the following corollary.
\begin{corollary}\label{cor:mini_batch_single_shufffling}
Mini-batch gradient descent with single shuffling is an irreducible, periodic Markov chain with period $l$.
\end{corollary}
\begin{proof}
For single shuffling version, we divide the whole dataset $[n]$ into $B(1),B(2),\cdots,B(l)$ and the corresponding sampling vector will be $\vv(1),\vv(2),\cdots,\vv(l)$. We shuffle the indices once, denoted by $a(1),a(2),\cdots,a(l)$, and stick to this sequence all the time. Then, in each epoch, sampler will update $\theta_t$ according to the sequence $\vv(a(1)),\vv(a(2)),\cdots,\vv(a(l))$, where $\{\vv_t\}_{t\geq 0}$ forms the finite, irreducible periodic Markov chain with period $l$.
\end{proof}
Together with Corollary \ref{cor:mini_batch_single_shufffling} and Lemma \ref{lemma:periodic_chain_in_CLT}, mini-batch gradient descent with single shuffling can be included in the SGD CLT analysis and Theorem \ref{Thm:main_result}. Then, by Lemma \ref{lemma:periodic_chain_in_CLT} and 4.2, mini-batch SGD with single shuffling can be applied to CLT result, which gives the asymptotic covariance matrix $\mSigma_{\vw}(\nabla \vF(\theta^*)^T\vv(\cdot)) = \vzero$ and thus the covariance matrix in the CLT result is also zero.

For random shuffling version, we can use similar method as in Appendix \ref{app:augmentation_random_shuffling} to augment the state space, which forms the corollary as follows.
\begin{corollary}\label{cor:mini_batch_random_shuffling}
$\{x_t\}_{t\geq 0}$ forms a finite, irreducible and periodic Markov chain with period $l$.
\end{corollary}
\begin{proof}
Let $\cX$ be the augmented space, where state $x_t \triangleq \{\{W_j^{(t)}\}_{j\in [l]},c_t\}$. Sequence $\{W_j^{(t)}\}_{j\in [l]} = \{\vv_{t-l+1},\vv_{t-l+2},\cdots, \vv_{t}\}$ records the last $l$ selected batches and $c_t\in\{1,2,\cdots,l\}$ is the relative position of the batch in the current epoch at time $t$. The only difference for mini-batch version to the single element version is that we sample one batch of size $B$ without replacement according to the indices yet to be chosen in the current epoch. Similar to the proof in Proposition \ref{prop:random_shuffling_markov_chain}, $\{X_t\}_{t\geq 0}$ is also a finite, irreducible and periodic Markov chain with period $l$.
\end{proof}
Corollary \ref{cor:mini_batch_random_shuffling} and Lemma \ref{lemma:periodic_chain_in_CLT} show that mini-batch gradient descent with random shuffling can also be included in the SGD CLT analysis. However, we still need to check the form of asymptotic covariance matrix due to the augmentation. We follow the same idea from Appendix \ref{app:apply_nbrw} and define a function $\Phi(\theta,x_t) \triangleq \vF(\theta)^T\vv(B_{t})$. Then, from Lemma \ref{Lem:Shuffling finite time var}, asymptotic covariance matrix $\mSigma_{x}$ is given as
\begin{equation}
\begin{split}
        \mSigma_x \!&=\! \text{Var}_{\vpi}(\nabla \Phi(\theta^*,x_0)) \!+\!\! \sum_{k\geq 1}\text{Cov}_{\vpi}(\nabla \Phi(\theta^*,x_0),\nabla \Phi(\theta^*,x_k))\text{Cov}_{\vpi}(\nabla \Phi(\theta^*,x_0),\nabla \Phi(\theta^*,x_k))^T\\ 
        &= \lim_{t \rightarrow \infty} \frac{1}{t}\E\left\{\left[\sum_{s=1}^t\left(\nabla \Phi(\theta^*, x_s) - \E_{\vpi}(\nabla \Phi(\theta^*,\cdot))\right)\right]\!\! \left[\sum_{s=1}^t\left(\nabla \Phi(\theta^*, x_s) - \E_{\vpi}(\nabla \Phi(\theta^*,\cdot))\right)\right]^T \! \right\} \\
        &= \lim_{t \rightarrow \infty} \frac{1}{t}\E\left\{\left[\sum_{s=1}^t\left(\nabla \vF(\theta^*)^T\vv(B_{s}) - \nabla f(\theta^*)\right)\right]\!\! \left[\sum_{s=1}^t\left(\nabla \vF(\theta^*)^T\vv(B_{s}) - \nabla f(\theta^*))\right)\right]^T \! \right\} \\
        &= \mSigma_{x}(\nabla \vF(\theta^*)^T\vv(\cdot)) = \vzero,
\end{split}
\end{equation}
where the third equality comes from the limiting distribution of random shuffling that is uniform. Thus, the covariance matrix of random shuffling in the CLT result is also zero.

Above results show that both single shuffling and random shuffling in mini-batch SGD have higher efficiency than mini-batch SGD with i.i.d sampling.
\end{proof}
Proposition \ref{Cor:mini-batch with shuffling} generalizes Proposition \ref{Cor:Shuffling vs iid} (special case with mini-batch of size $S=1$) in that the same efficiency ordering between shuffling and \textit{i.i.d} input holds true even with mini-batches.

\section{Simulation}\label{app:simulation}
In Appendix \ref{subsection:detail_figure1}, we give the details of our simulation setup for Figure 1, involving three reversible Markov chains - the Metropolis-Hasting random walk (MHRW), a modification of MHRW (Modified-MHRW) and fastest mixing Markov chain (FMMC), each having the uniform distribution as their stationary measure. In Appendix \ref{subsection:additional_result}, we expand upon the numerical results in Section 5 by including additional results for large graphs.

\subsection{Details behind Figure 1}\label{subsection:detail_figure1}

For the random walk SGD (RWSGD) simulation in Figure 1, we consider the problem of minimizing a (scalar-valued) quadratic objective function 
\begin{equation}\label{eqn:obj_simu}
    f(\theta) \triangleq \frac{1}{n}\sum_{i=1}^n F(\theta,i) = \frac{1}{2n}\sum_{i=1}^n (\theta - b(i))^2,
\end{equation}
where $\theta, ~b(i) \in \R$ for $i=1,2,\cdots,n$ and $n$ is the number of nodes on the graph. 
The minimizer is given by $\theta^* \triangleq \arg\min_{\theta} f(\theta) = \frac{1}{n}\sum_{i=1}^n b(i)$. 
The RWSGD iteration for the objective function \eqref{eqn:obj_simu} is then given by
\begin{equation}\label{eqn:iteration_simulation}
    \theta_{t+1} = \theta_t - \gamma_{t+1}(\theta_t - b(X_{t+1})),
\end{equation}
where we choose $\gamma_t = 1/t^{0.9}$ and $\{X_t\}_{t\geq 0}$ is the stochastic input, e.g., MHRW, Modified-MHRW, and FMMC. 


In Figure 1, we simulate the SGD algorithm on two graphs; one is an $8$-node graph $\cG_1$ and the other is a $5$-node graph $\cG_2$. The two graphs are arbitrarily constructed while ensuring connectivity. See Figure \ref{fig:topology_graph} for resulting topologies. 
\begin{figure}[!ht]
    \centering
     \begin{subfigure}[b]{0.47\textwidth}
         \centering
         \includegraphics[width=\textwidth]{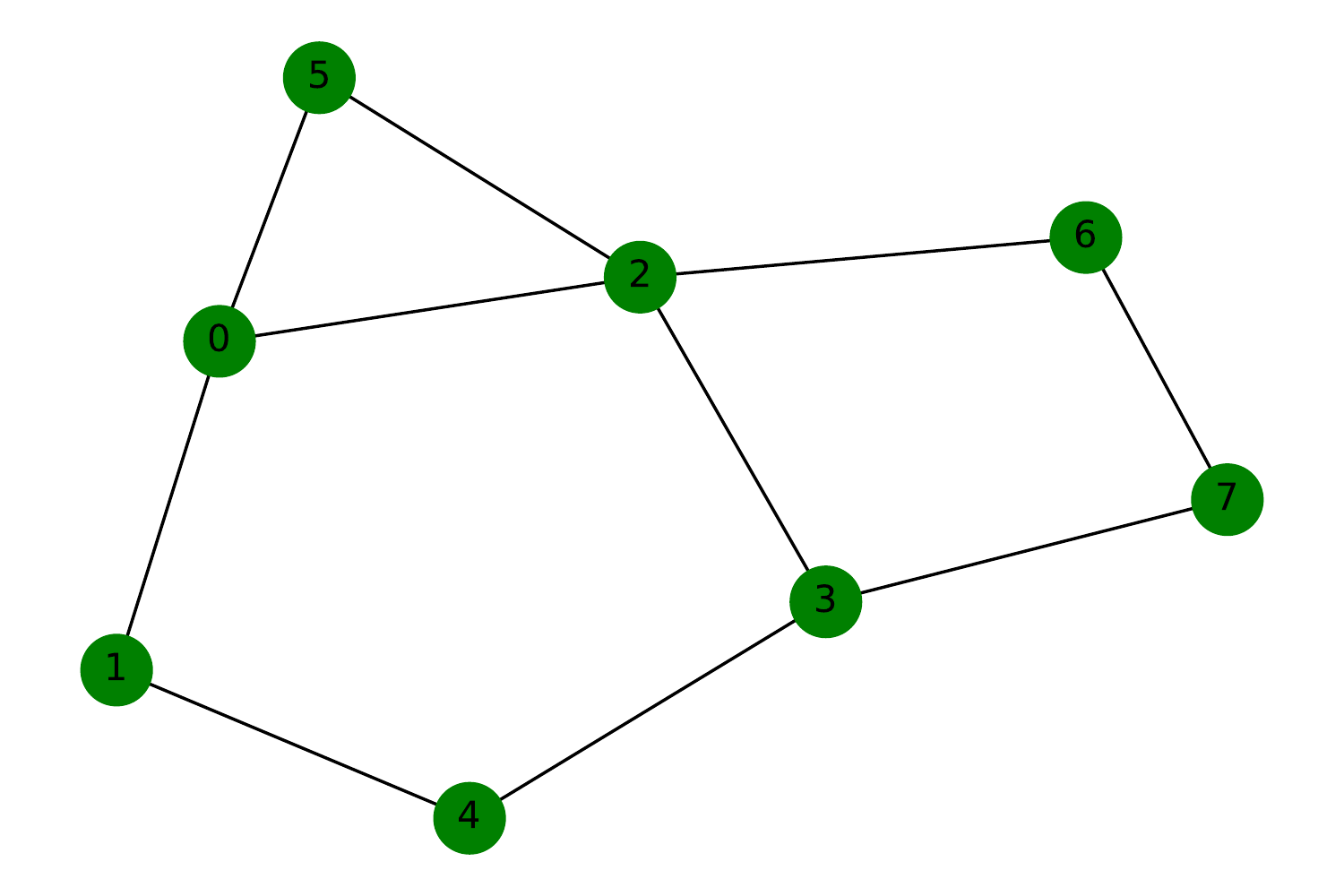}
         \caption{$8$-node graph $\cG_1$.}
         \label{fig:8node}
     \end{subfigure}
     \begin{subfigure}[b]{0.47\textwidth}
         \centering
         \includegraphics[width=\textwidth]{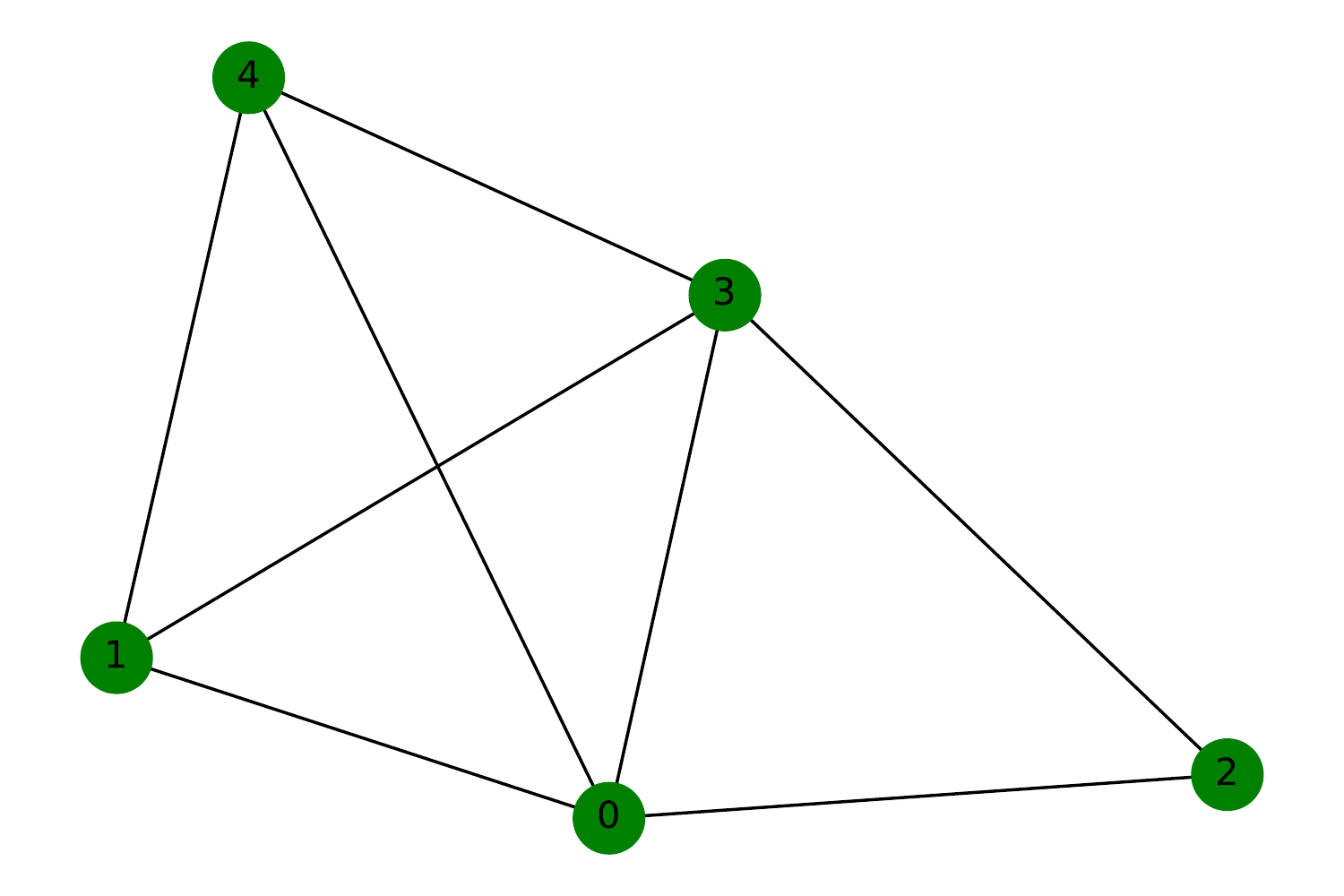}
         \caption{$5$-node graph $\cG_2$.}
         \label{fig:5node}
     \end{subfigure}
    \caption{Topology of two graphs.}
    \label{fig:topology_graph}
\end{figure}


Now, we are ready to introduce the construction of three Markov chains on two graphs in Figure \ref{fig:topology_graph}. 

\textbf{MHRW:} Metropolis-Hasting algorithm \citep{metropolis1953equation} shows that the transition matrix of MHRW is constructed in the following manner:
\begin{equation}
    P(i,j) = \begin{cases} \min\left\{\frac{1}{d_i}, \frac{1}{d_i}\right\}, & j\in N(i), \\ 1 - \sum_{j\in N(i)} P(i,j), & j = i, \end{cases}
\end{equation}
where $d_i$ is the degree of node $i$ and $N(i)$ is the set of node $i$'s neighbors.


\textbf{Modified-MHRW:} To construct a `modified-MHRW', which is more \textit{efficient} than the standard MHRW,\footnote{\textit{Efficiency ordering} of Markov chains is introduced in Definition 3.5. In short, a Markov chain $\{X_t\}_{t\geq 0}$ is more efficient than $\{Y_t\}_{t\geq 0}$ if the asymptotic variances (AV) satisfy $\sigma^2_X(g) \leq \sigma^2_Y(g)$ for any scalar-valued function $g$ with $\E_{\vpi}(g^2) < \infty$, where $\sigma^2_X(g)$ is defined in \eqref{eqn:AV scalar}.} we employ the notion of `Peskun ordering', originated from \citep{peskun1973optimum}.
\begin{definition}[Peskun ordering \citep{peskun1973optimum}]
For two finite, ergodic, reversible Markov chains $\{X_t\}_{t\geq 0}, \{Y_t\}_{t\geq 0}$ on the state space $\cV$ with transition matrices $\mP_X,\mP_Y$ having the same stationary distribution $\vpi$, it is said that $\mP_Y$ dominates $\mP_X$ off the diagonal, written as $\mP_X \preceq \mP_Y$ if $P_{X}(i,j) \leq P_Y(i,j)$ for all $i,j\in\cV$ and $i\neq j$.
\end{definition}
We have the following lemma that connects the Peskun ordering to the efficiency ordering.
\begin{lemma}[\citep{peskun1973optimum} Theorem 2.1.1]\label{lemma:peskun_ordering}
If $\mP_X \preceq \mP_Y$, then $\sigma^2_X(g) \geq \sigma^2_Y(g)$ for any scalar-valued function $g$ with $\E_{\vpi}(g^2) < \infty$, that is, $\{Y_t\}_{t\geq 0}$ is more efficient than $ \{X_t\}_{t\geq 0}$.
\end{lemma}
We can manually construct a more efficient Markov chain by reducing the self-transition probability $P(i,i)$ of the MHRW and redistributing to off-diagonal entries, whenever possible, in a way that each row still sums to one and the resulting matrix is doubly-stochastic (i.e., the resulting Markov chain is reversible w.r.t the uniform distribution). In view of Lemma \ref{lemma:peskun_ordering}, this modification improves the efficiency (smaller AV $\sigma^2$ compared to the standard MHRW).\footnote{Note that there can be many ways to modify the standard MHRW that make the Markov chain more efficient. The pursuit of the `optimal' modification w.r.t the efficiency is out of the scope of this paper.} 


\textbf{FMMC:} FMMC is obtained by solving a semidefinite programming (proposed in problem (6) of \citep{boyd2004fastest}), which gives a Markov chain that minimizes the SLEM of the transition matrix over the entire class of reversible Markov chains w.r.t the uniform stationary distribution for a given graph topology. This is done numerically by using the CVXOPT package \citep{diamond2016cvxpy}. Later we will show in the simulation that FMMC indeed has the smallest SLEM compared to MHRW and Modified-MHRW.


In what follows, we index these three Markov chains with numbers in the subscript: MHRW (indexed by $1$), Modified-MHRW (indexed by $2$), and FMMC (indexed by $3$). For graph $\cG_1$, the transition matrices of MHRW $\mP_{1}^{\cG_1}$, Modified-MHRW $\mP_{2}^{\cG_1}$, and FMMC $\mP_{3}^{\cG_1}$ are given by
\begin{equation}\label{eqn:graph1_transition}
    \begin{split}
        \mP_{1}^{\cG_1} &= \begin{bmatrix} \nicefrac{1}{12} & \nicefrac{1}{3} & \nicefrac{1}{4} & 0 & 0 & \nicefrac{1}{3} & 0 & 0 \\ \nicefrac{1}{3}&\nicefrac{1}{6}&0&0&\nicefrac{1}{2}&0&0&0 \\ \nicefrac{1}{4}&0&0&\nicefrac{1}{4}&0&\nicefrac{1}{4}&\nicefrac{1}{4}&0 \\ 0&0&\nicefrac{1}{4}&\nicefrac{1}{12}&\nicefrac{1}{3}&0&0&\nicefrac{1}{3} \\ 0&\nicefrac{1}{2}&0&\nicefrac{1}{3}&\nicefrac{1}{6}&0&0&0 \\ \nicefrac{1}{3}&0&\nicefrac{1}{4}&0&0&\nicefrac{5}{12}&0&0 \\ 0&0&\nicefrac{1}{4}&0&0&0&\nicefrac{1}{4}&\nicefrac{1}{2} \\ 0&0&0&\nicefrac{1}{3}&0&0&\nicefrac{1}{2}&\nicefrac{1}{6} \end{bmatrix}, \\
        \mP_{2}^{\cG_1} &= \begin{bmatrix} 0 & 0.35 & 0.25 & 0 & 0 & 0.4 & 0 & 0 \\ 0.35&0.02&0&0&0.63&0&0&0 \\ 0.25&0&0&0.25&0&0.25&0.25&0 \\ 0&0&0.25&0&0.37&0&0&0.38 \\ 0&0.63&0&0.37&0&0&0&0 \\ 0.4&0&0.25&0&0&0.35&0&0 \\ 0&0&0.25&0&0&0&0.13&0.62 \\ 0&0&0&0.38&0&0&0.62&0 \end{bmatrix}, \\
        \mP_{3}^{\cG_1} &= \begin{bmatrix} 0.13 & 0.42 & 0.17 & 0 & 0 & 0.28 & 0 & 0 \\ 0.42&0.1&0&0&0.48&0&0&0 \\ 0.17&0&0&0.06&0&0.32&0.45&0 \\ 0&0&0.06&0.14&0.46&0&0&0.34 \\ 0&0.48&0&0.46&0.06&0&0&0 \\ 0.28&0&0.32&0&0&0.4&0&0 \\ 0&0&0.45&0&0&0&0.09&0.46 \\ 0&0&0&0.34&0&0&0.46&0.2 \end{bmatrix}.
    \end{split}
\end{equation}
For graph $\cG_2$, the transition matrices of MHRW $\mP_{1}^{\cG_2}$, Modified-MHRW $\mP_{2}^{\cG_2}$, and FMMC $\mP_{3}^{\cG_2}$ are given by
\begin{equation}\label{eqn:graph2_transition}
    \begin{split}
        \mP_{1}^{\cG_2} &= \begin{bmatrix} 0 & \nicefrac{1}{4} & \nicefrac{1}{4} & \nicefrac{1}{4} & \nicefrac{1}{4} \\ \nicefrac{1}{4} & \nicefrac{1}{6} & 0 & \nicefrac{1}{4} & \nicefrac{1}{3} \\ \nicefrac{1}{4} & 0 & \nicefrac{1}{2} & \nicefrac{1}{4} & 0 \\ \nicefrac{1}{4} & \nicefrac{1}{4} & \nicefrac{1}{4} & 0 & \nicefrac{1}{4} \\ \nicefrac{1}{4} & \nicefrac{1}{3} & 0 & \nicefrac{1}{4} & \nicefrac{1}{6}\end{bmatrix}, \\ 
        \mP_{2}^{\cG_2} &= \begin{bmatrix} 0 & 0.25 & 0.25 & 0.25 & 0.25 \\ 0.25 & 0 & 0 & 0.25 & 0.5 \\ 0.25 & 0 & 0.5 & 0.25 & 0 \\ 0.25 & 0.25 & 0.25 & 0 & 0.25 \\ 0.25 & 0.5 & 0 & 0.25 & 0\end{bmatrix}, \\ 
        \mP_{3}^{\cG_2} &= \begin{bmatrix} 0.09 & 0.25 & 0.33 & 0.08 & 0.25 \\ 0.25 & 0.25 & 0 & 0.25 & 0.25 \\ 0.33 & 0 & 0.34 & 0.33 & 0 \\ 0.08 & 0.25 & 0.33 & 0.09 & 0.25 \\ 0.25 & 0.25 & 0 & 0.25 & 0.25 \end{bmatrix}. \\
    \end{split}
\end{equation}
In both \eqref{eqn:graph1_transition} and \eqref{eqn:graph2_transition}, observe that Modified-MHRW and MHRW follow the Peskun ordering, i.e., $\mP_1^{\cG_1}\preceq \mP_2^{\cG_1}$ and $\mP_1^{\cG_2}\preceq \mP_2^{\cG_2}$, such that Modified-MHRW is more efficient than MHRW according to Lemma \ref{lemma:peskun_ordering}. In addition, the SLEMs of these matrices are given in Table \ref{tab:SLEM}, where FMMC has the smallest SLEM in both graphs compared to MHRW and Modified-MHRW. Interestingly, Modified-MHRW has larger SLEM than MHRW in graph $\cG_1$, which means Modified-MHRW can mix slower than MHRW to the stationary distribution.
\begin{table}[!ht]
    \centering
    \begin{tabular}{|c|c|c|}
         \hline
         &  $\cG_1$ & $\cG_2$ \\
         \hline
         MHRW ($\beta_1$) &  0.761 & 0.500\\
         \hline
         Modified-MHRW ($\beta_2$) & 0.868 & 0.500\\
         \hline
         FMMC ($\beta_3$) & 0.712 & 0.408\\
         \hline
    \end{tabular}
    \vspace{2mm}
    \caption{SLEMs of the transition matrices in \eqref{eqn:graph1_transition} and \eqref{eqn:graph2_transition}.}
    \label{tab:SLEM}
\end{table}

In Figure \ref{fig:my_label}, we show the simulation result of each Markov chain in the RWSGD algorithm with iteration \eqref{eqn:iteration_simulation} w.r.t MSE $\E\|\theta_t-\theta^*\|_2^2$ in graph $\cG_1$ and $\cG_2$.\footnote{The reason we plot the same curves in each graph will be explained in the next paragraph.} In both graphs, Modified-MHRW (green curve) performs better than MHRW (red curve) and FMMC (blue curve) with smallest MSE while it has the largest SLEM shown in Table \ref{tab:SLEM}. This implies that the order of SLEM does not reflect the order of MSE in the RWSGD algorithm.

\begin{figure}[!ht]
    \centering
     \begin{subfigure}[b]{0.49\textwidth}
         \centering
         \includegraphics[width=\textwidth]{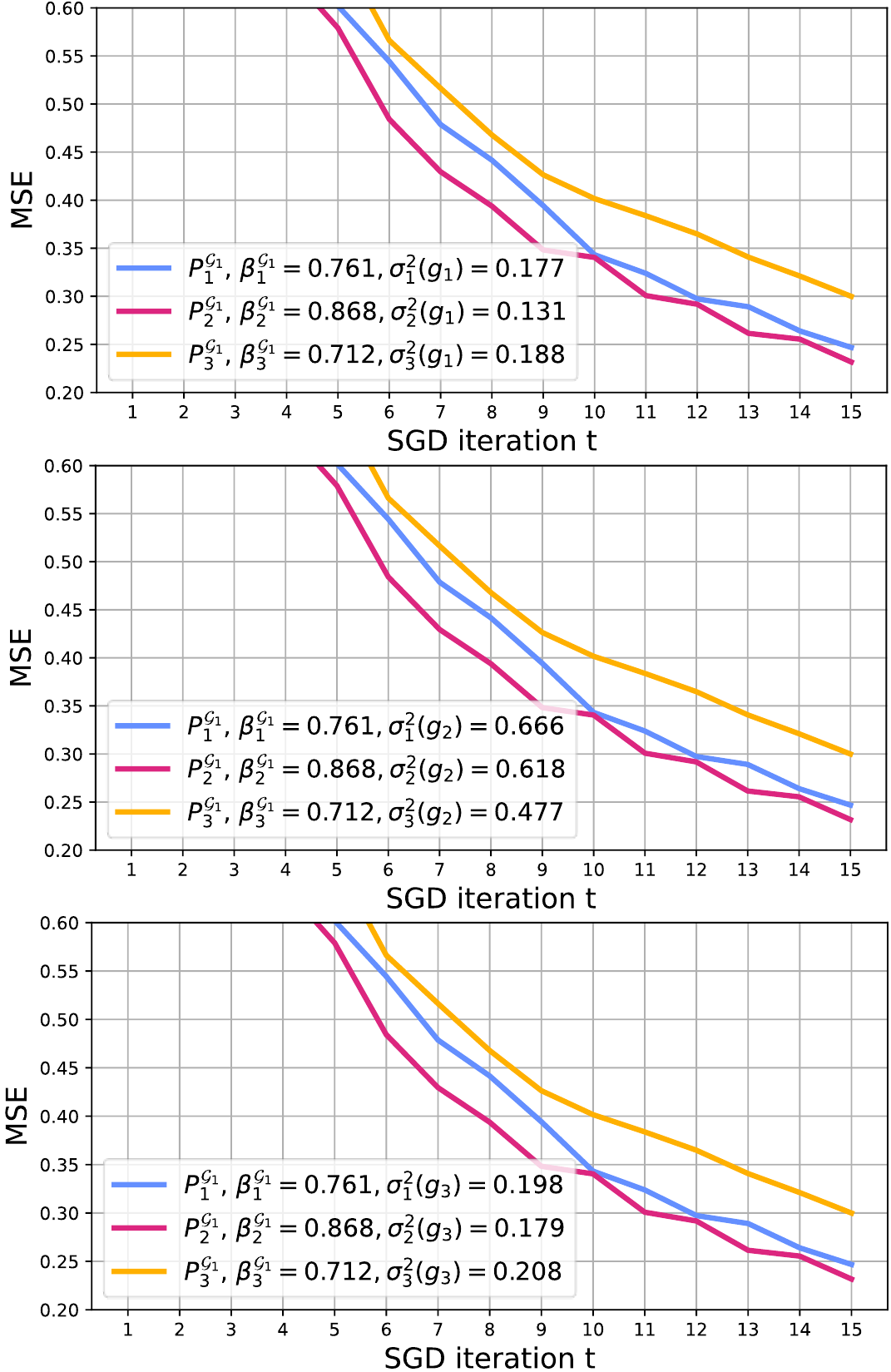}
         \caption{Graph $\cG_1$}
         \label{fig:fmmcvspeskun1}
     \end{subfigure}
     \begin{subfigure}[b]{0.49\textwidth}
         \centering
         \includegraphics[width=\textwidth]{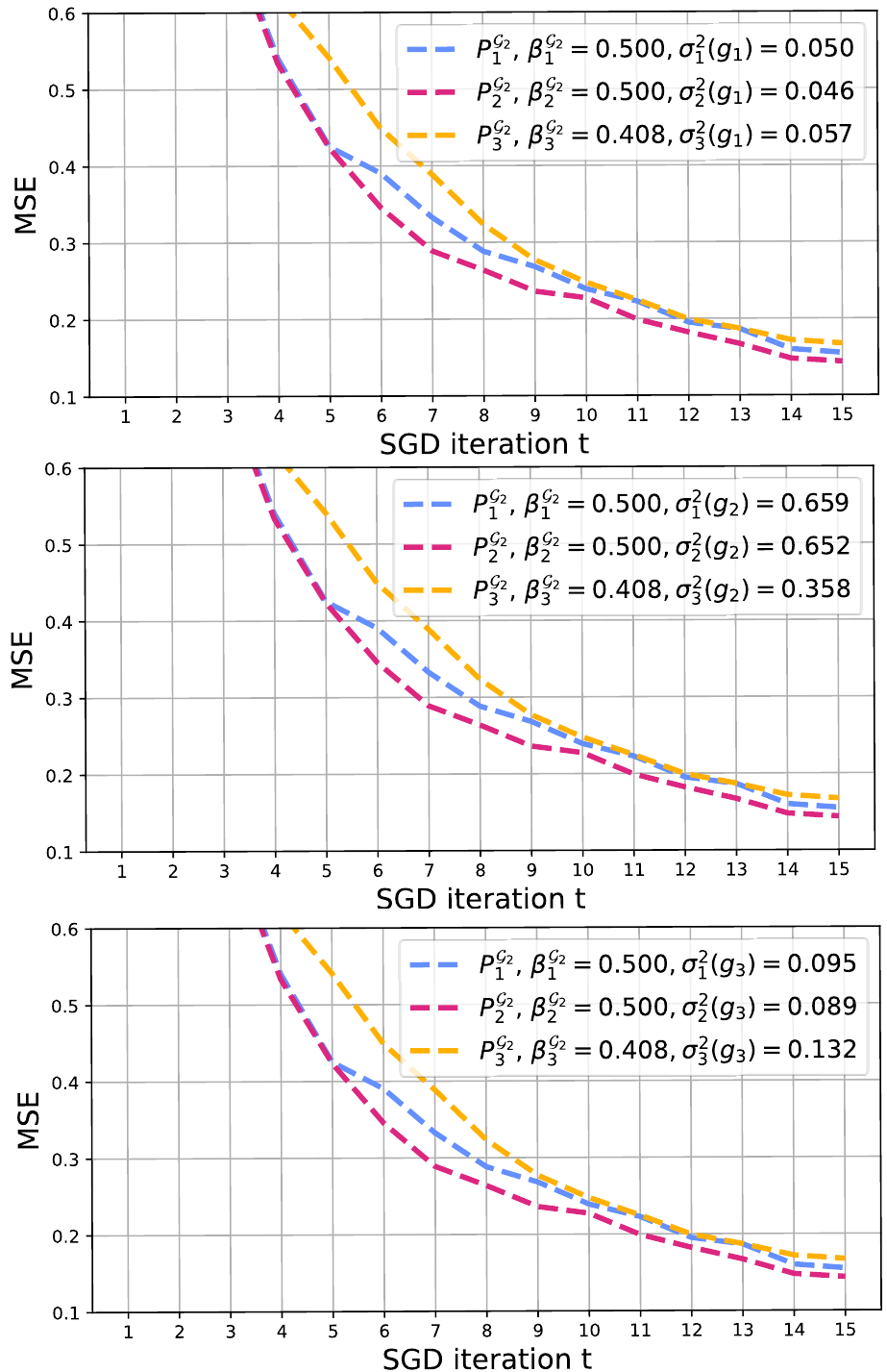}
         \caption{Graph $\cG_2$}
         \label{fig:fmmcvspeskun2}
     \end{subfigure}
    \caption{MSE $\E\|\theta_t-\theta^*\|_2^2$ of three Markov chains in the SGD algorithm with iteration \eqref{eqn:iteration_simulation}.}
    \label{fig:my_label}
\end{figure}


In Figure \ref{fig:my_label}, we repeat the plot in each graph three times with three different values of AVs inside the legend, the reason being that we want to see if the performance of each Markov chain is related to the AV $\sigma^2(g)$ and its test function $g$, other than SLEM solely. In the top row of Figure \ref{fig:my_label}, as well as in Figure 1, we choose the test function $g_1(i) = \nabla F(\theta^*,i)$ for $i=1,2,\cdots,n$, where $\nabla F(\theta,i)$ is the gradient of the local function $F(\theta,i)$ \eqref{eqn:obj_simu} w.r.t $\theta$. In the middle row of Figure \ref{fig:my_label}, the test function is $g_2(i) = d_i$, which estimates the average degree of the graph. In the bottom row of Figure \ref{fig:my_label}, the test function is $g_3(i) = \mathbbm{1}_{\{i = 1\}}$, which estimates the probability of visiting node $1$. We include the AVs of all three test functions $g_1,g_2,g_3$ in the legend of Figure \ref{fig:my_label}, e.g., $\sigma^2_3(g_1)$ is the AV of the test function $g_1$ for FMMC.\footnote{The AV of the test function for each Markov chain is calculated by running a stochastic simulation for a long time and directly computing according to the definition in \eqref{eqn:AV scalar}.} We observe that $\sigma^2_1(g_1) < \sigma_3^2(g_1)$ and $\sigma^2_1(g_3) < \sigma_3^2(g_3)$ while $\sigma^2_1(g_2) > \sigma_3^2(g_2)$ in both graphs. This means MHRW and FMMC are \textit{not} efficiency ordered, which is possible because efficiency ordering is a partial order such that not every two Markov chains can be ordered. On the other hand, in both graphs, $\sigma^2_1(g_k) > \sigma^2_2(g_k)$ for $k=1,2,3$, which is consistent with the fact that the constructed Modified-MHRW is more efficient than MHRW. Regarding the MSE, we find that Modified-MHRW performs better than MHRW in both graphs, which is in line with the efficiency ordering. This leads us to conjecture that two efficiency ordered Markov chains might also have their performance in the RWSGD algorithm ordered in the same way. 


\begin{remark}\label{remark:discussion_SLEM}
Section 1.1 in \citep{sun2018markov} numerically compares the performance of a reversible Markov chain and its non-reversible counterpart in the RWSGD algorithm w.r.t SLEM, and shows that the non-reversible counterpart with smaller SLEM performs better. The main theorem therein is also applicable to the comparison of two reversible Markov chains. However, as shown in Figure \ref{fig:my_label}, we provide examples to show that a reversible Markov chain with smaller SLEM does not necessarily lead to smaller MSE in the RWSGD algorithm. Note that our results do not contradict the simulation results in Section 1.1 of \citep{sun2018markov}, since it is possible for a Markov chain to have both smaller AV and smaller SLEM; which could be the case for the non-reversible counterpart in \citep{sun2018markov}, although they didn't specify the AV in their simulation. Moreover, their main theorem is an upper bound to the error terms considered, which means that an SLEM-based ordering does not guarantee a performance ordering of the error terms themselves, as also exemplified in Figure \ref{fig:my_label}. All of these together imply that SLEM alone cannot be the sole indicator of performance of the Markov chains as input sequences for RWSGD algorithms. 
\qed
\end{remark}

\subsection{Numerical Results on Large Graphs}\label{subsection:additional_result}
We first specify the process of dataset generation for the sum-of-nonconvex functions $\hat{f}(\theta)$ in \eqref{eqn:obj}. We generate random vectors $\va_1,\cdots\va_n,\vb \in \R^{10}$ uniformly from $[0,1]$ and ensure the invertibility of $\sum_{i=1}^n \va_i\va_i^T$. Then, we randomly select half of the matrices in $\{\mD_i\}_{i\in[n]}$ and assign $+1.1$ to their $j$-th diagonal; other matrices are assigned $-1.1$ to $j$-th diagonal. We repeat the above process for all diagonal values $j=1,2,\cdots,10$. This process guarantees $\sum_{i=1}^n \mD_i = \mathbf{0}$. 

We perform additional simulations on graph `AS-733' \citep{leskovec2005graphs} with $6474$ nodes, and graph `wikiVote' \citep{leskovec2010signed} with $889$ nodes with the same objective functions $\Tilde{f}(\theta)$ and $\hat{f}(\theta)$ in \eqref{eqn:obj}. The simulation results are given in Figure \ref{fig:as733} and \ref{fig:wiki}. We plot the curves of NBRW and SRW in the insets of Figures \ref{fig:as733_1}, \ref{fig:as733_2}, \ref{fig:wiki_1}, and \ref{fig:wiki_2}, with the same x,y axes but at linear scale, to better observe the difference in their performance. For both objective functions, NBRW has smaller MSE than SRW and both random and single shuffling perform better than uniform sampling, e.g., Figure \ref{fig:as733_1} and \ref{fig:wiki_1}.\footnote{The curves of NBRW and SRW in Figure \ref{fig:as733_1} and \ref{fig:wiki_1} appear flat because they are plotted in the same figure with uniform sampling and single/random shuffling, which have much smaller MSE. We plot the comparison between NBRW and SRW separately in the inset.} This demonstrate that NBRW and SRW are efficiency-ordered, which also holds for random/single shuffling and uniform sampling. Note that since we simulate on large graphs, for the logistic regression problem, the SGD algorithm with NBRW and SRW is yet to enter the asymptotic regime even in the $100,000$-th iteration, which can be explained by the blue and green increasing curves in the inset of Figure \ref{fig:as733_2} and \ref{fig:wiki_2}. On the other hand, the curve of uniform sampling becomes flat and the curves of single/random shuffling are starting to go down in Figure \ref{fig:as733_2}, \ref{fig:as733_4} and \ref{fig:wiki_2}, \ref{fig:wiki_4}, implying that they have entered the asymptotic regime. These results are consistent to the observations in Figure 2, which support our theory.

\begin{figure}[!ht]
    \centering
     \begin{subfigure}[b]{0.61\textwidth}
         \centering
         \includegraphics[width=\textwidth]{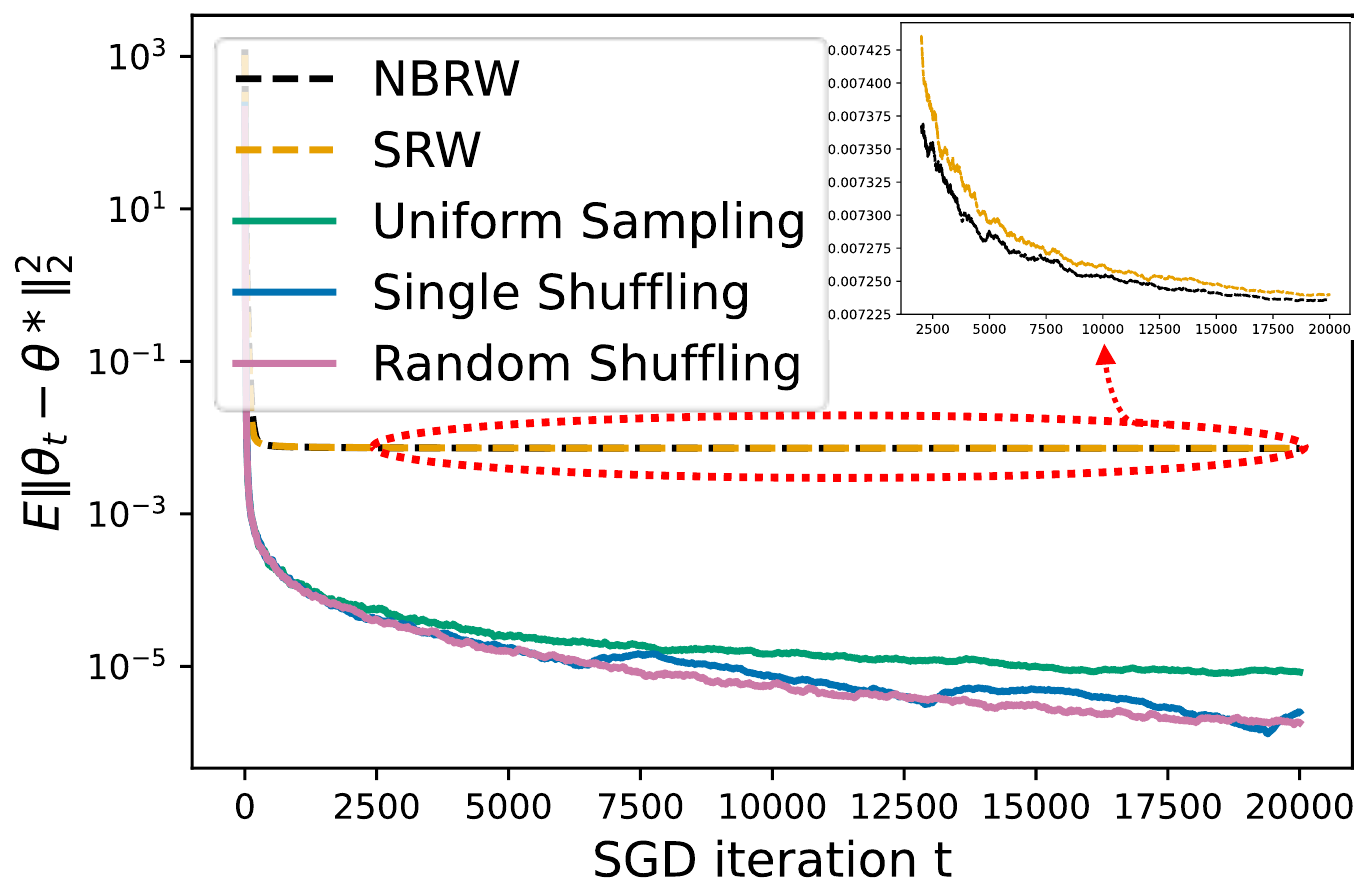}
         \caption{Logistic regression (MSE)}
         \label{fig:as733_1}
     \end{subfigure}
     \begin{subfigure}[b]{0.61\textwidth}
         \centering
         \includegraphics[width=\textwidth]{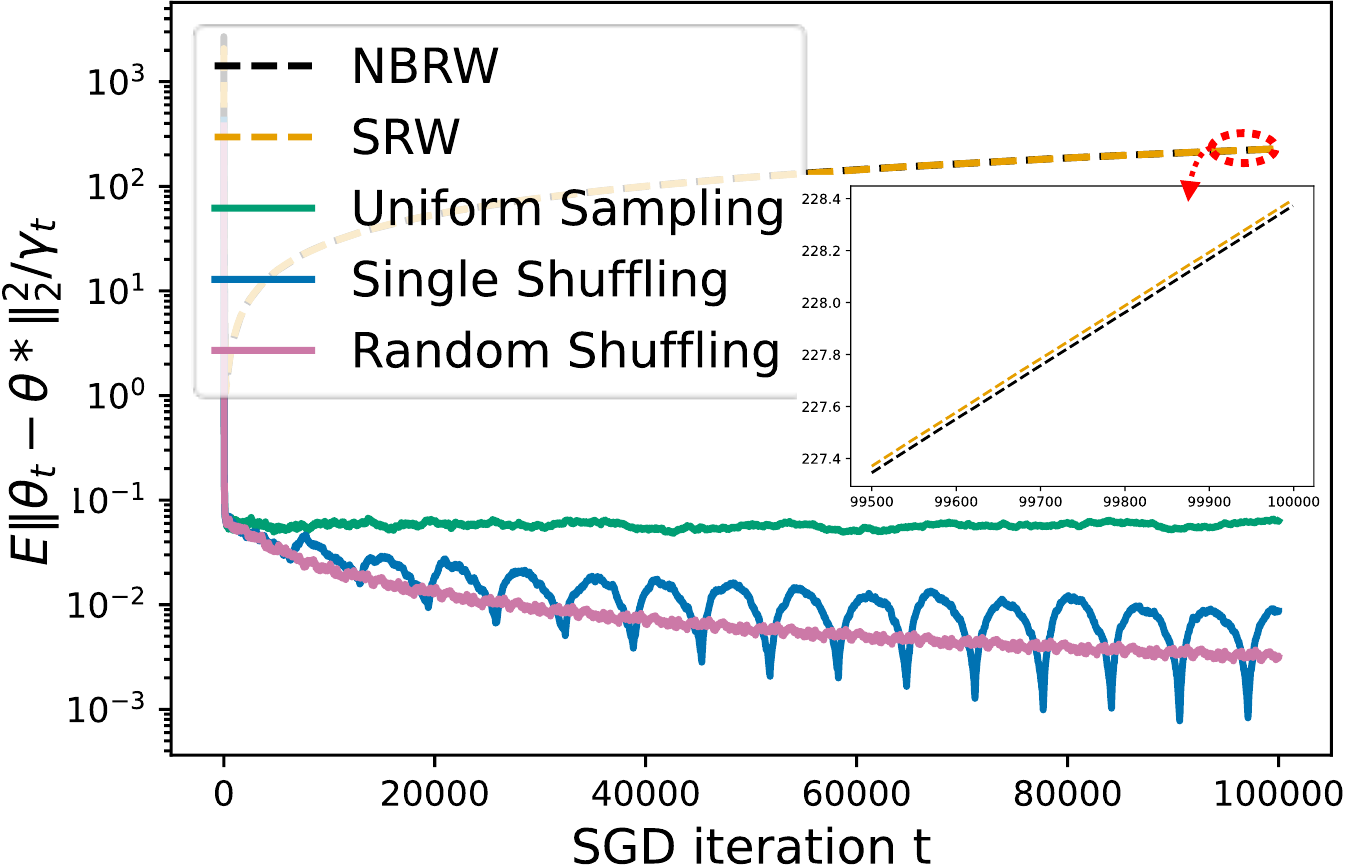}
         \caption{Logistic regression (scaled MSE)}
         \label{fig:as733_2}
     \end{subfigure}
     \begin{subfigure}[b]{0.61\textwidth}
         \centering
         \includegraphics[width=\textwidth]{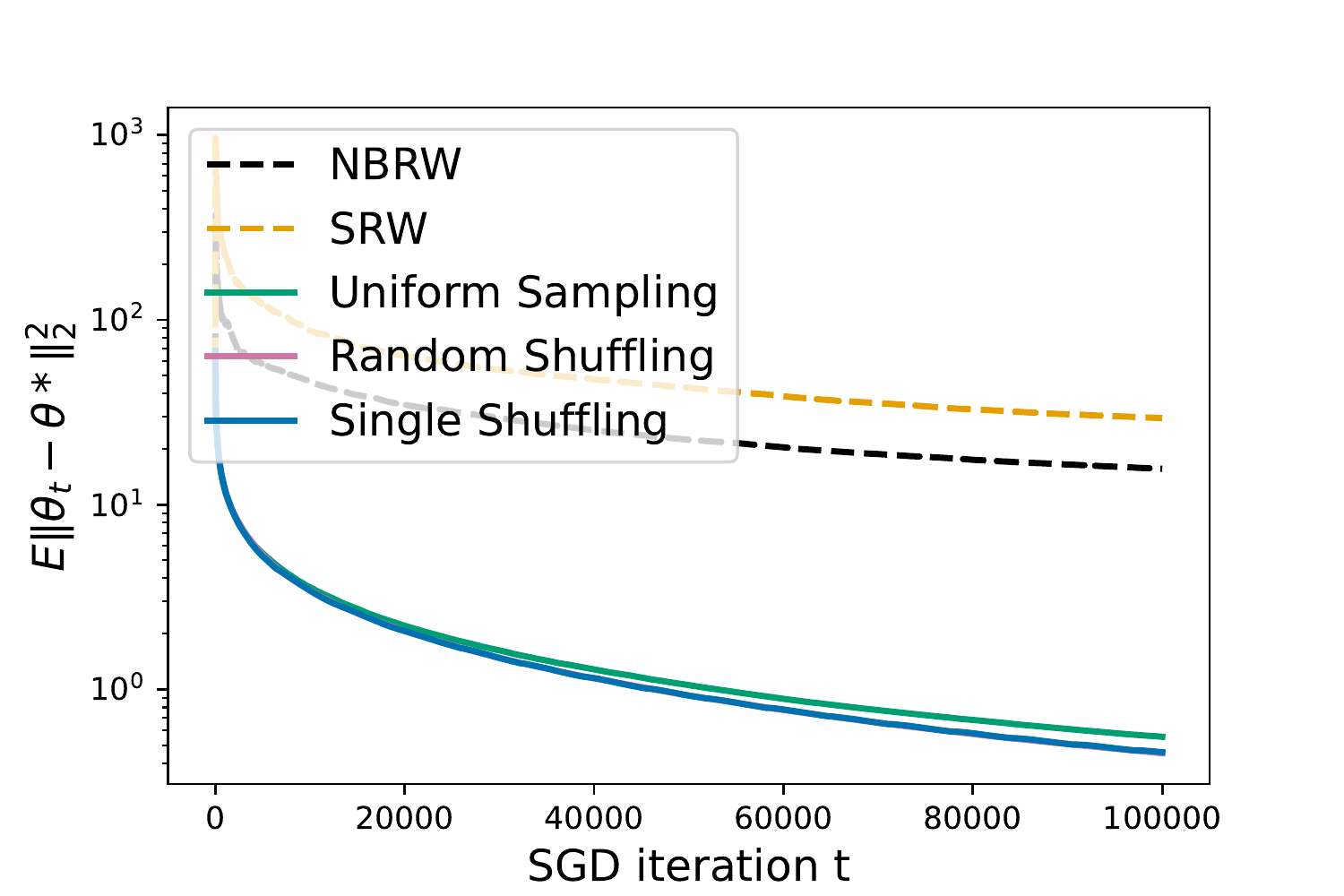}
         \caption{sum of non-convex fn. (MSE)}
         \label{fig:as733_3}
     \end{subfigure}
     \begin{subfigure}[b]{0.61\textwidth}
         \centering
         \includegraphics[width=\textwidth]{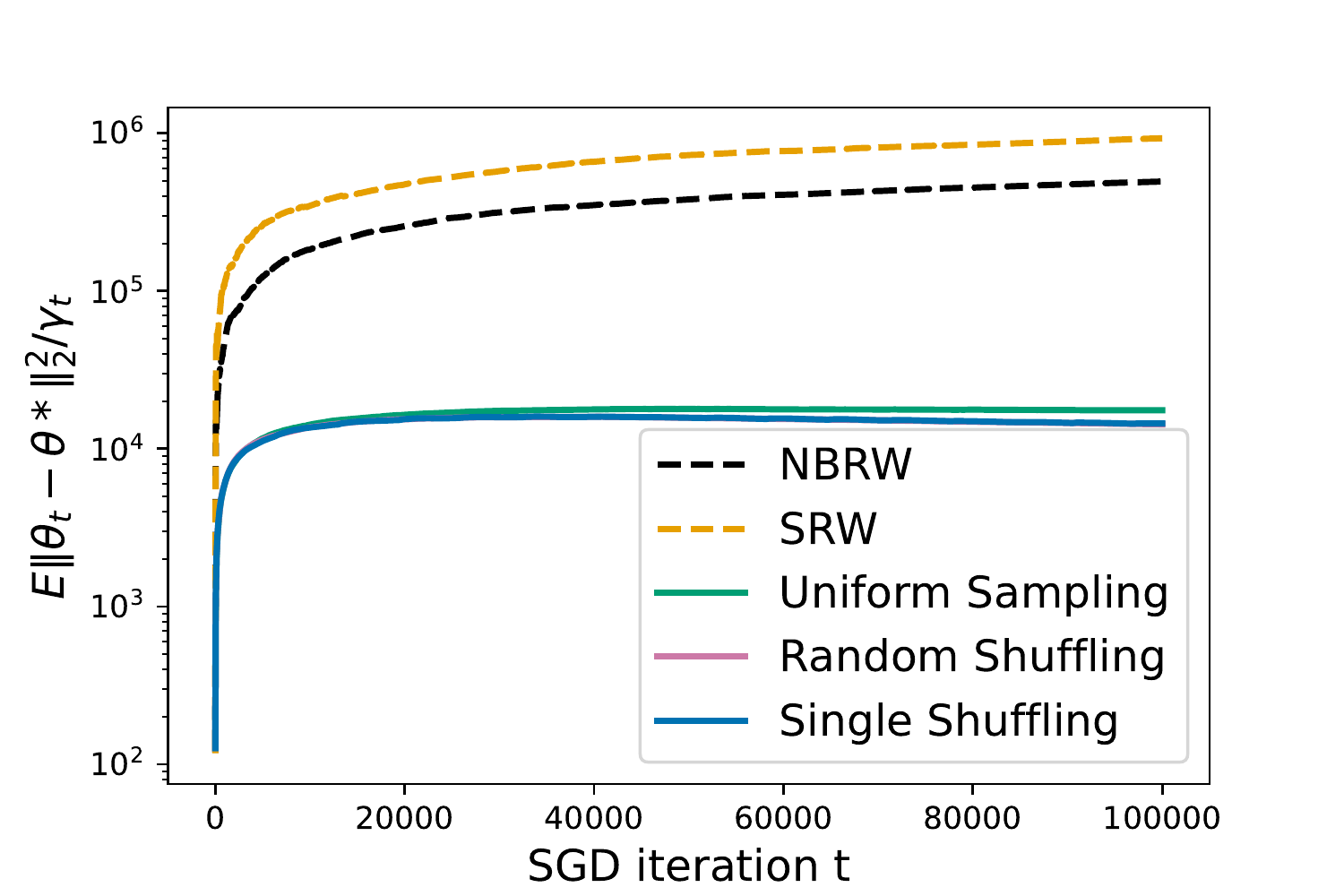}
         \caption{sum of non-convex fn. (scaled MSE)}
         \label{fig:as733_4}
     \end{subfigure}
     \vspace{-2mm}
     \caption{Performance comparison of different stochastic inputs on the graph `AS-733'.} 
     \label{fig:as733}
     \vspace{-2mm}
\end{figure}

\begin{figure}[!ht]
    \centering
     \begin{subfigure}[b]{0.61\textwidth}
         \centering
         \includegraphics[width=\textwidth]{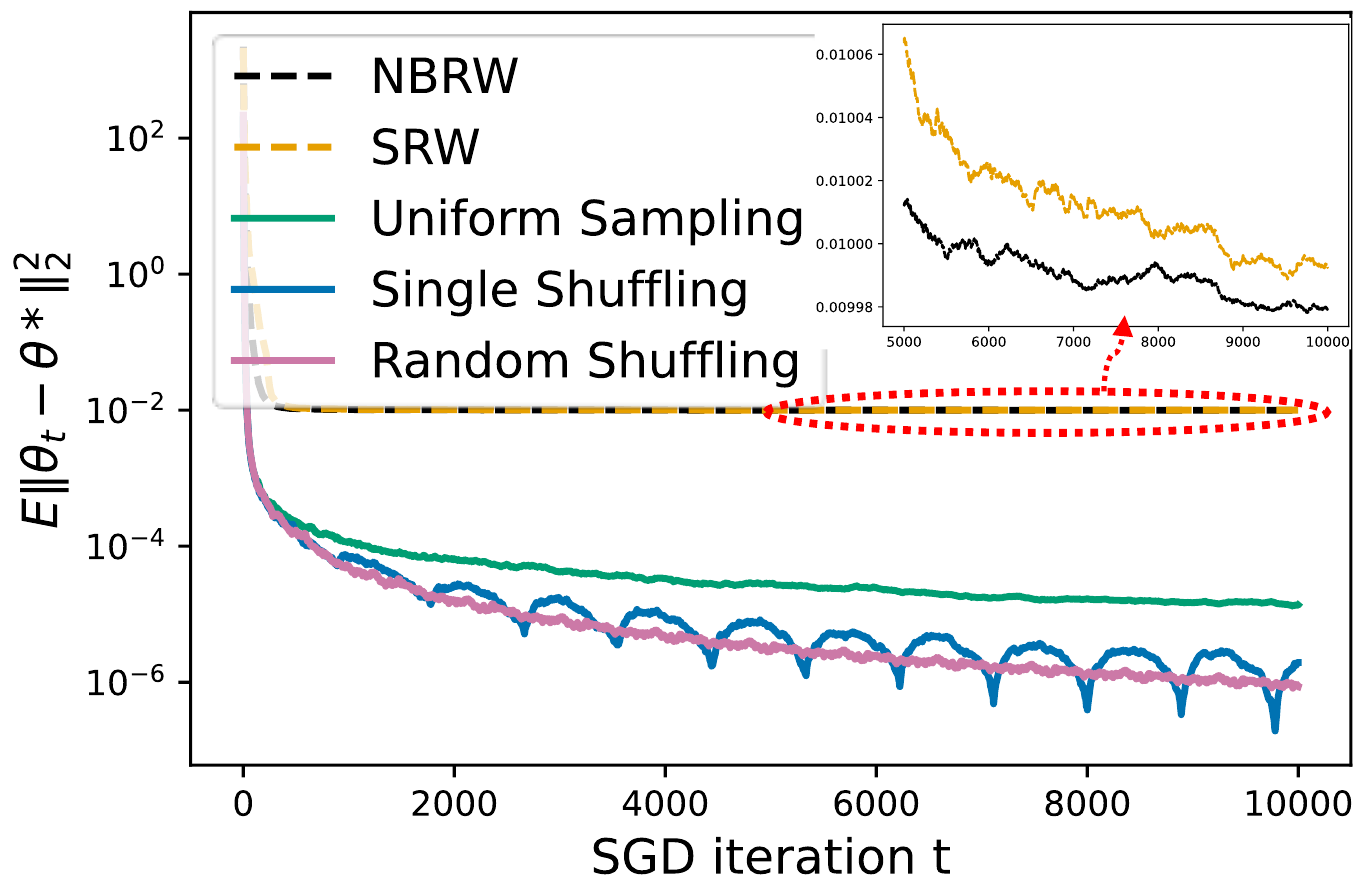}
         \caption{Logistic regression (MSE)}
         \label{fig:wiki_1}
     \end{subfigure}
     \begin{subfigure}[b]{0.61\textwidth}
         \centering
         \includegraphics[width=\textwidth]{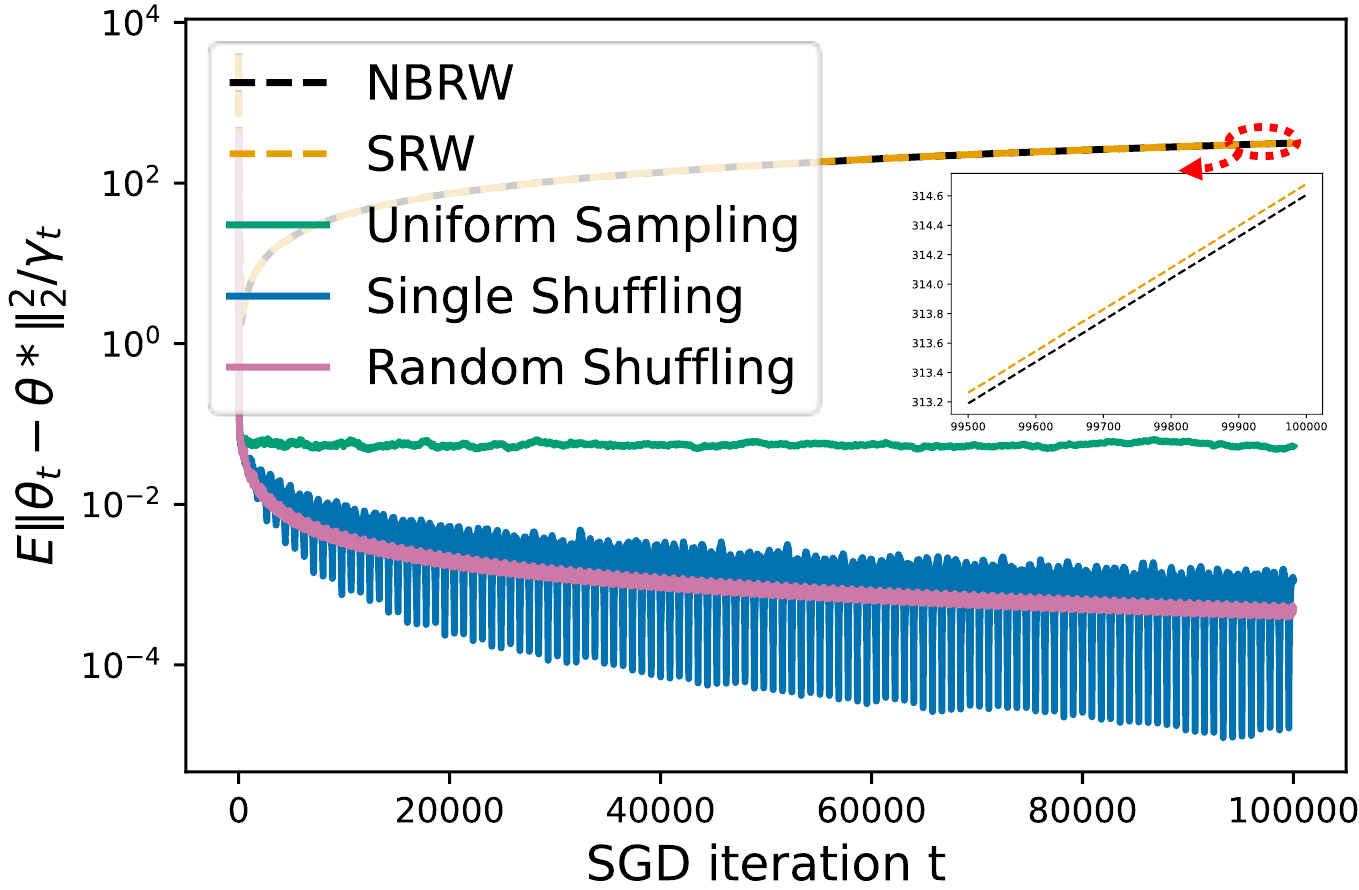}
         \caption{Logistic regression (scaled MSE)}
         \label{fig:wiki_2}
     \end{subfigure}
     \begin{subfigure}[b]{0.61\textwidth}
         \centering
         \includegraphics[width=\textwidth]{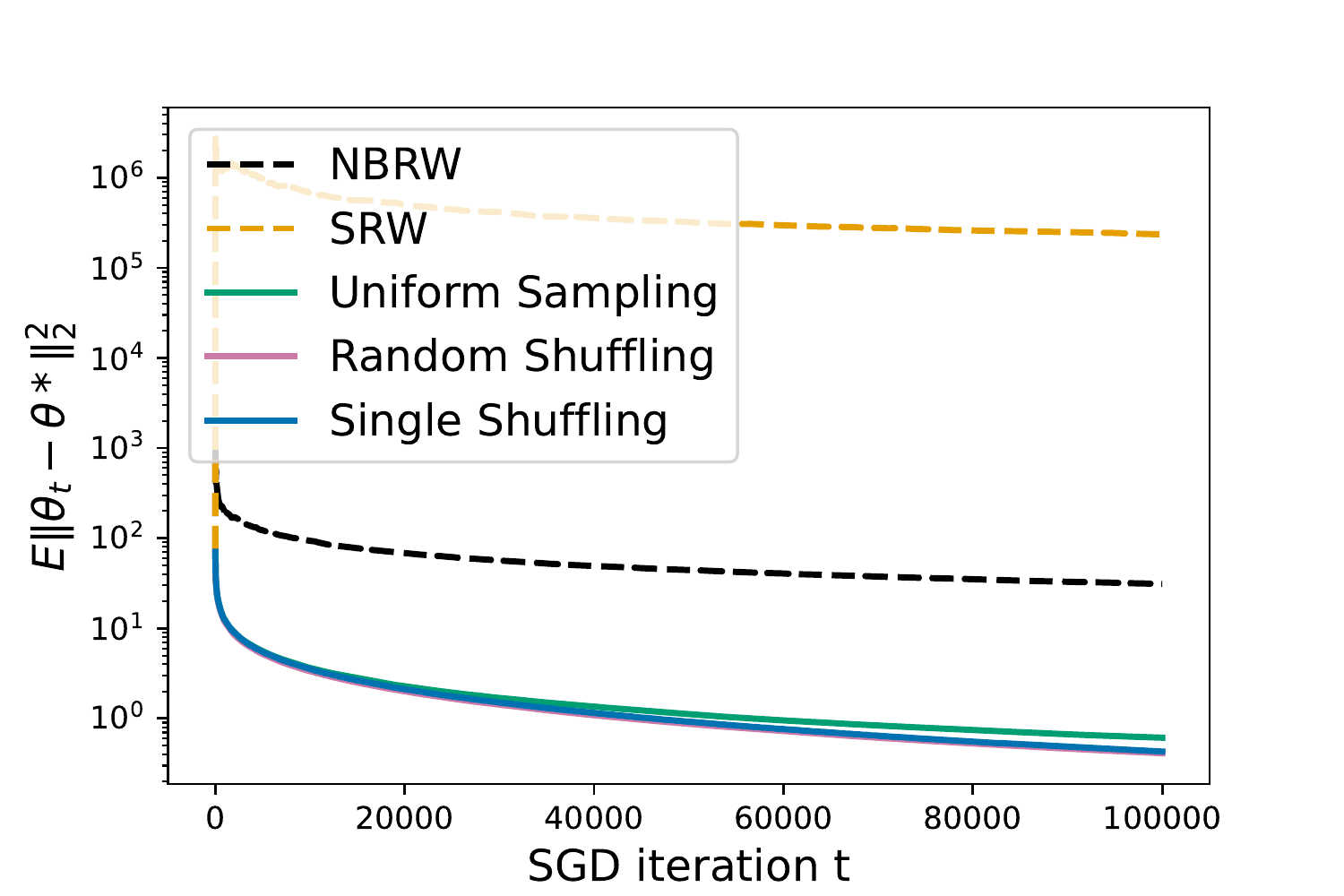}
         \caption{sum of non-convex fn. (MSE)}
         \label{fig:wiki_3}
     \end{subfigure}
     \begin{subfigure}[b]{0.61\textwidth}
         \centering
         \includegraphics[width=\textwidth]{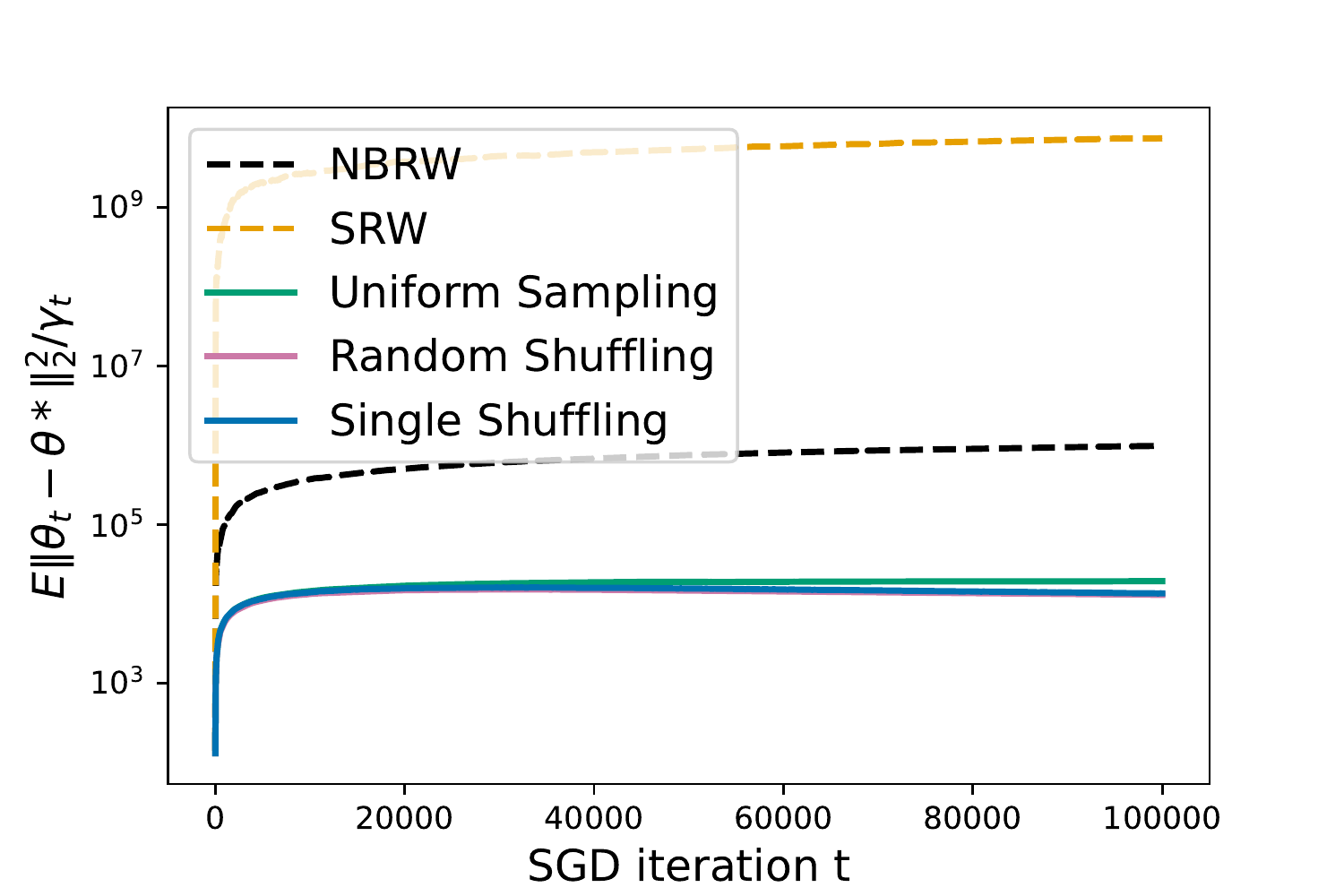}
         \caption{sum of non-convex fn. (scaled MSE)}
         \label{fig:wiki_4}
     \end{subfigure}
     \vspace{-2mm}
     \caption{Performance comparison of different stochastic inputs on the graph `wikiVote'.} 
     \label{fig:wiki}
     \vspace{-2mm}
\end{figure}

\subsection{Additional Simulations on Non-convex Objective Function and SGD Variants}\label{app:simulation_ext}

Regarding the SGD variants other than the vanilla SGD, central limit theorem (CLT) is less well studied in the literature. To list a few, \citep{lei2020variance} studied variance reduced SGD (SVRG) and obtained the CLT for constant step size. \citep{barakat2021convergence} analyzed Adam and their follow-up \citep{barakat2021stochastic} extended the CLT for a general SGD algorithm, which includes Stochastic Heavy Ball (SHB), Nesterov accelerated SGD (NaSGD) and Adam. \citep{li2022revisiting} established the CLT for momentum SGD (mSGD) and NaSGD under more general conditions on the step size. However, all of these recent works focus only on the Martingale difference noise 
$$\mathbb{E}[\delta_{t+1}|\mathcal{F}_{t} \triangleq \sigma(\theta_0,X_0,X_1,\cdots,X_t)] = \mathbb{E}[\nabla f(\theta_t) - \nabla F(\theta_t,X_{t+1})|\mathcal{F}_{t}] = 0,$$
which is equivalent to saying that the input $\{X_t\}_{t\geq 0}$ is independently sampled from some identical distribution for each time $t$ (\textit{i.i.d} input sequence). Meanwhile, for Markovian inputs, $$\mathbb{E}[\delta_{t+1}|\mathcal{F}_{t}]= \sum_{i\in[n]}\pi_i \nabla F(\theta_t,i) - \sum_{i\in[n]}P(X_{t},i)\nabla F(\theta_t,i) \neq 0$$
because $\pi_i \neq P(X_{t},i)$ in general (unless $\{X_t\}_{t\geq 0}$ is an \textit{i.i.d} sequence). It remains an open problem to obtain the CLT for these SGD variants with general Markovian inputs, which would be a prerequisite for our efficiency ordering. Indeed, one of our future works is to theoretically prove the CLT results for SGD variants with Markovian inputs and to carry over our efficiency ordering of different stochastic inputs.

Next, we simulate two SGD variants, i.e., Nesterov accelerated SGD (NaSGD) and ADAM, on graph ``AS-733'' (as used in Appendix \ref{subsection:additional_result}) with two pair of stochastic inputs, i.e., NBRW versus SRW and shuffling methods versus \textit{i.i.d} input sequence, with respect to both convex objective function and non-convex objective function. We choose the convex objective function $\hat{f}(\theta)$ from \eqref{eqn:obj} such that
\begin{equation}
    \hat{f}(\theta) \!=\! \frac{1}{n}\sum_{i=1}^n \theta^T(\va_i\va_i^T \!+\! \mD_i)\theta \!+\! \vb^T\theta,
\end{equation}
where $\sum_{i=1}^n \va_i\va_i^T$ is invertible and $\sum_{i=1}^n \mD_i = \mathbf{0}$. We can see $\nabla^2 \hat{f}(\theta) = \frac{2}{n}\sum_{i=1}^n \va_i\va_i^T$ is a positive semi-definite matrix and $\hat{f}(\theta)$ is convex. Then, we modify matrices $\{\mD_i\}_{i\in[n]}$ such that the first element on the main diagonal of each matrix $\mD_i$ is subtracted by $0.1$, and we denote the new matrices as $\{\mM_i\}_{i\in[n]}$. We define a new function $\hat{g}(\theta)$ such that
\begin{equation}
    \hat{g}(\theta) =\frac{1}{n}\sum_{i=1}^n \theta^T(\mathbf{a}_i\mathbf{a}_i^T \!+\! \mathbf{M}_i)\theta \!+\! \mathbf{b}^T\theta.
\end{equation}
We numerically compute $\nabla^2 \hat{g}(\theta) = \frac{2}{n}\sum_{i=1}^n (\va_i\va_i^T+\mM_i)$ and ensure it has at least one negative eigenvalue such that the objective function $\hat{g}(\theta)$ is non-convex. For Nesterov accelerated SGD, we employ the following iteration \citep{li2022revisiting}
\begin{equation}\label{eqn:nasgd}
\begin{split}
    \theta_{t+1} &= u_t - \gamma_{t+1} \nabla G(u_{t},X_{t+1}),\\
    u_{t+1} &= \theta_{t+1} + \beta_{t+1} (\theta_{t+1} -\theta_t),
\end{split}
\end{equation}
where $\gamma_t = 1/0.9^t$ and $\beta_{t+1} \equiv \beta = 0.5$ in our settings. For ADAM, we use the following iteration \citep{Kingma2015}
\begin{equation}\label{eqn:adam}
\begin{split}
    g_{t+1} &= \nabla G(\theta_t,X_{t+1}), \\
    m_{t+1} &= \alpha_1 m_t + (1-\alpha_1) g_{t+1}, \\
    v_{t+1} &= \alpha_2 v_t + (1-\alpha_2) g_{t+1}^2, \\
    m' &= m_{t+1}/(1-\alpha_1^t), \\
    v' &= v_{t+1}/(1-\alpha_2^t), \\
    \theta_{t+1} &= \theta_t - \gamma_t m'/(\sqrt{v'} + \epsilon),
\end{split}
\end{equation}
where $\gamma_t=1/0.9^t$, $\alpha_1 = 0.9, \alpha_2 = 0.999$, $\epsilon = 10^{-8}$, $g_{t+1}^2$ is the element-wise square for the vector $g_{t+1}$ and $\sqrt{v'}$ is the element-wise square root for the vector $v'$.  
In both \eqref{eqn:nasgd} and \eqref{eqn:adam}, function $G(\theta,i) = \theta^T(\va_i\va_i^T+\mD_i)\theta + \vb^T\theta$ for convex objective function $\hat{f}(\theta)$ and $G(\theta,i) = \theta^T(\va_i\va_i^T+\mM_i)\theta + \vb^T\theta$ for non-convex objective function $\hat{g}(\theta)$.

The insets of Figure \ref{fig:8} are to enlarge the curves of NBRW and SRW in ADAM algorithm for the iteration $t\in[40000,50000]$ to make them more distinguishable. In Figure \ref{fig:8}, we show that for both convex and non-convex objective functions, the curves of NBRW are always below those of SRW in vanilla SGD, NaSGD and ADAM, respectively. This not only supports our Theorem \ref{Thm:main_result} on vanilla SGD and both convex and non-convex objective functions, but also suggest that the efficiency ordering is still valid for other SGD variants. In Figure \ref{fig:9}, we also empirically test the performance of shuffling methods and uniform sampling on vanilla SGD, NaSGD and ADAM with non-convex objective function $\hat{g}(\theta)$. In all three SGD iterations, we show that shuffling methods are better than uniform sampling, although the gap between shuffling methods and uniform sampling is small in NaSGD and ADAM in Figure \ref{fig:9b} and \ref{fig:9c} and the reason could be that these SGD variants implicitly include the ``momentum'' may decrease the effect of the correlation from the stochastic inputs. We also notice from Figure \ref{fig:8} that for a given stochastic input, NaSGD and ADAM are better than vanilla SGD, while in Figure \ref{fig:9} the result is reversed. Currently, we only know that those SGD variants NaSGD and ADAM work better than vanilla SGD in practice for \textit{i.i.d} input sequence. It remains an open problem for SGD variants with general Markovian inputs, and thus, it's possible that Markovian inputs can influence the performance of NaSGD and ADAM, compared to vanilla SGD. In any case, Figure \ref{fig:8} and Figure \ref{fig:9} still validate our efficiency ordering.

\begin{figure}[!ht]
    \centering
     \begin{subfigure}[b]{0.9\textwidth}
         \centering
         \includegraphics[width=\textwidth]{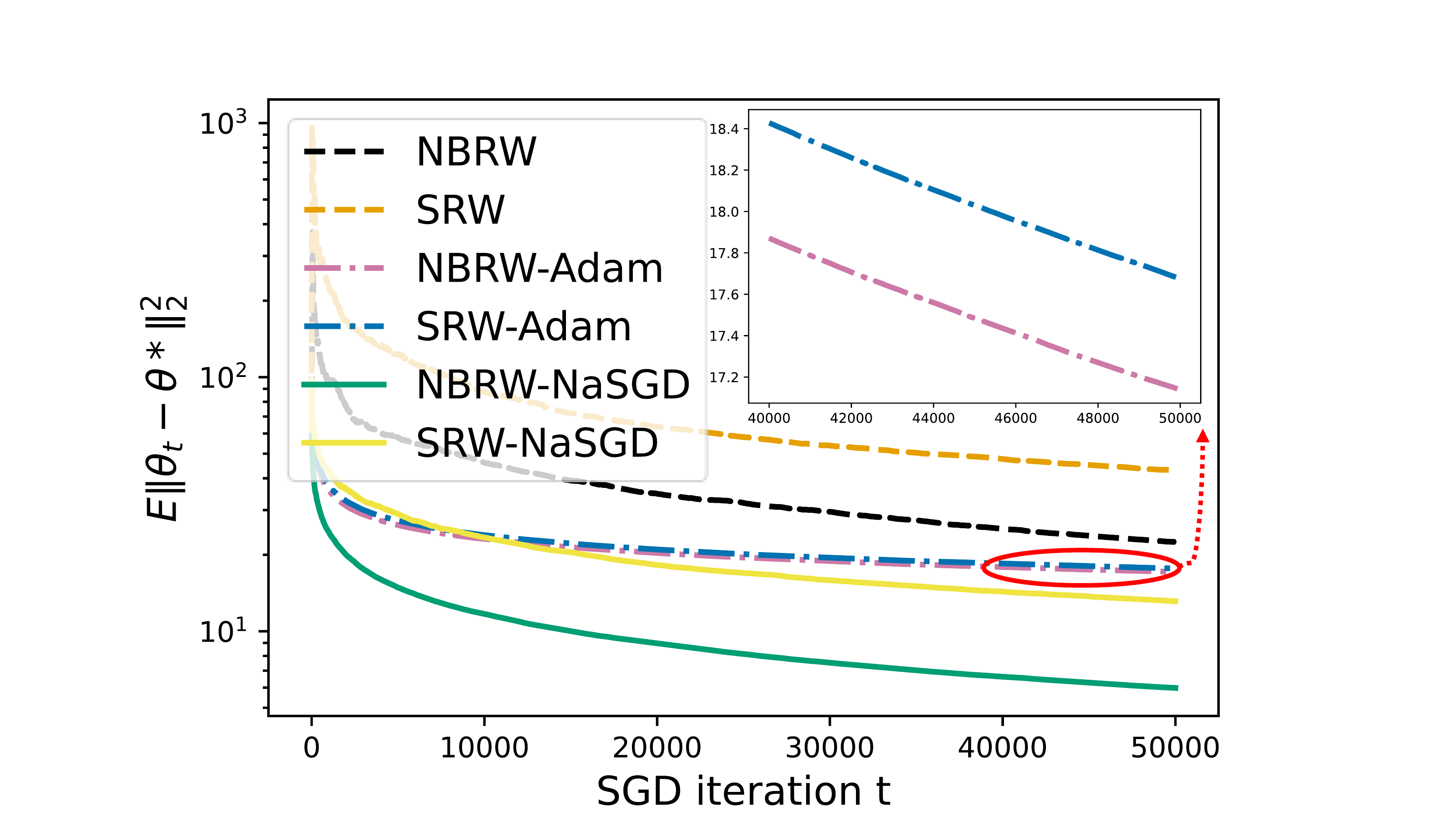}
         \caption{sum-non-convex functions $\hat{f}(\theta)$}
         \label{fig:8a}
     \end{subfigure}\quad
     \begin{subfigure}[b]{0.9\textwidth}
         \centering
         \includegraphics[width=\textwidth]{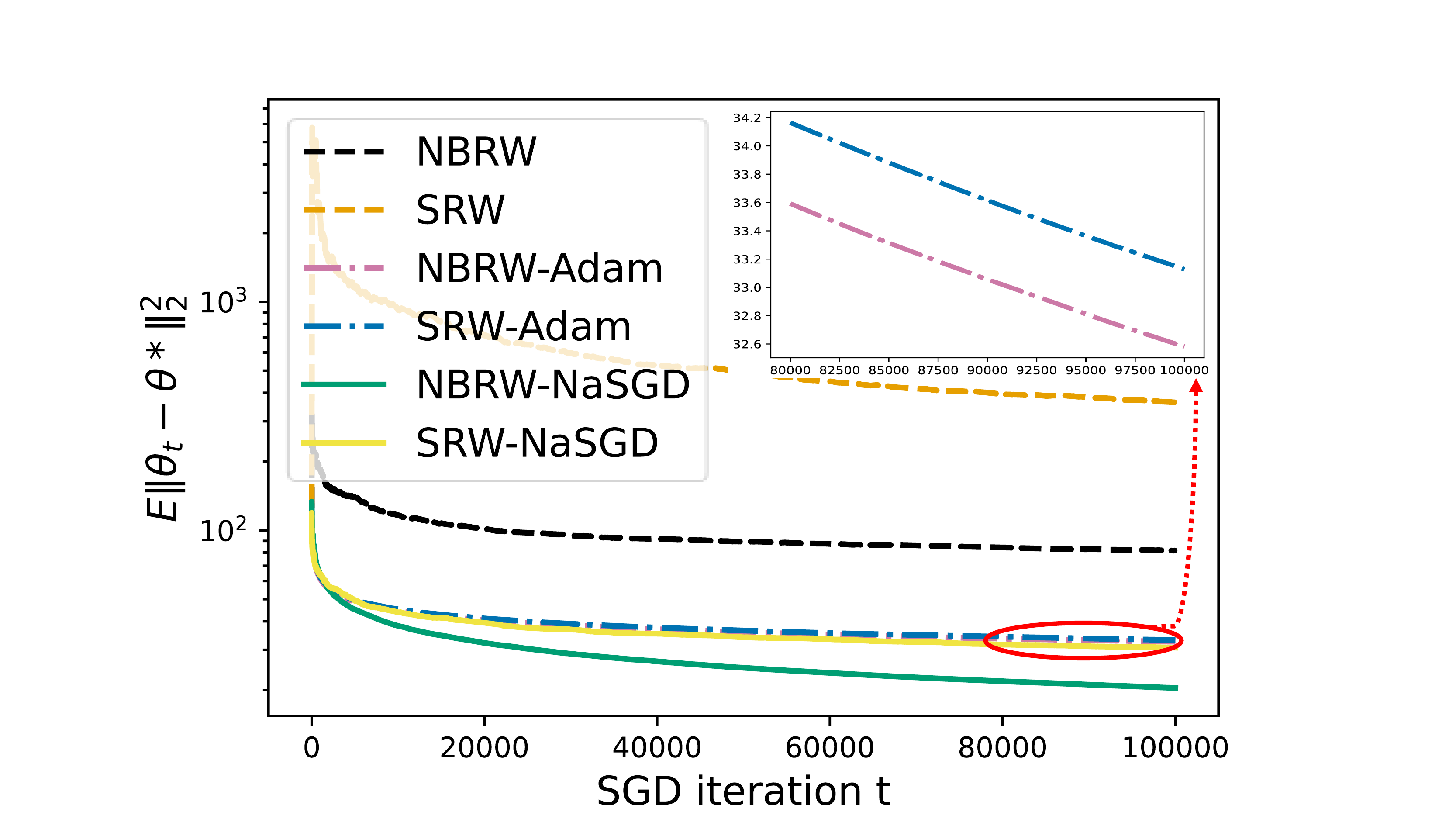}
         \caption{non-convex function $\hat{g}(\theta)$}
         \label{fig:8b}
     \end{subfigure} 
     \vspace{-2mm}
     \caption{Performance comparison of NBRW and SRW in vanilla SGD, NaSGD and ADAM algorithms on the graph ``AS-733''.}
     \label{fig:8}
     \vspace{-2mm}
\end{figure}

\begin{figure}[!ht]
    \centering
     \begin{subfigure}[b]{0.75\textwidth}
         \centering
         \includegraphics[width=\textwidth]{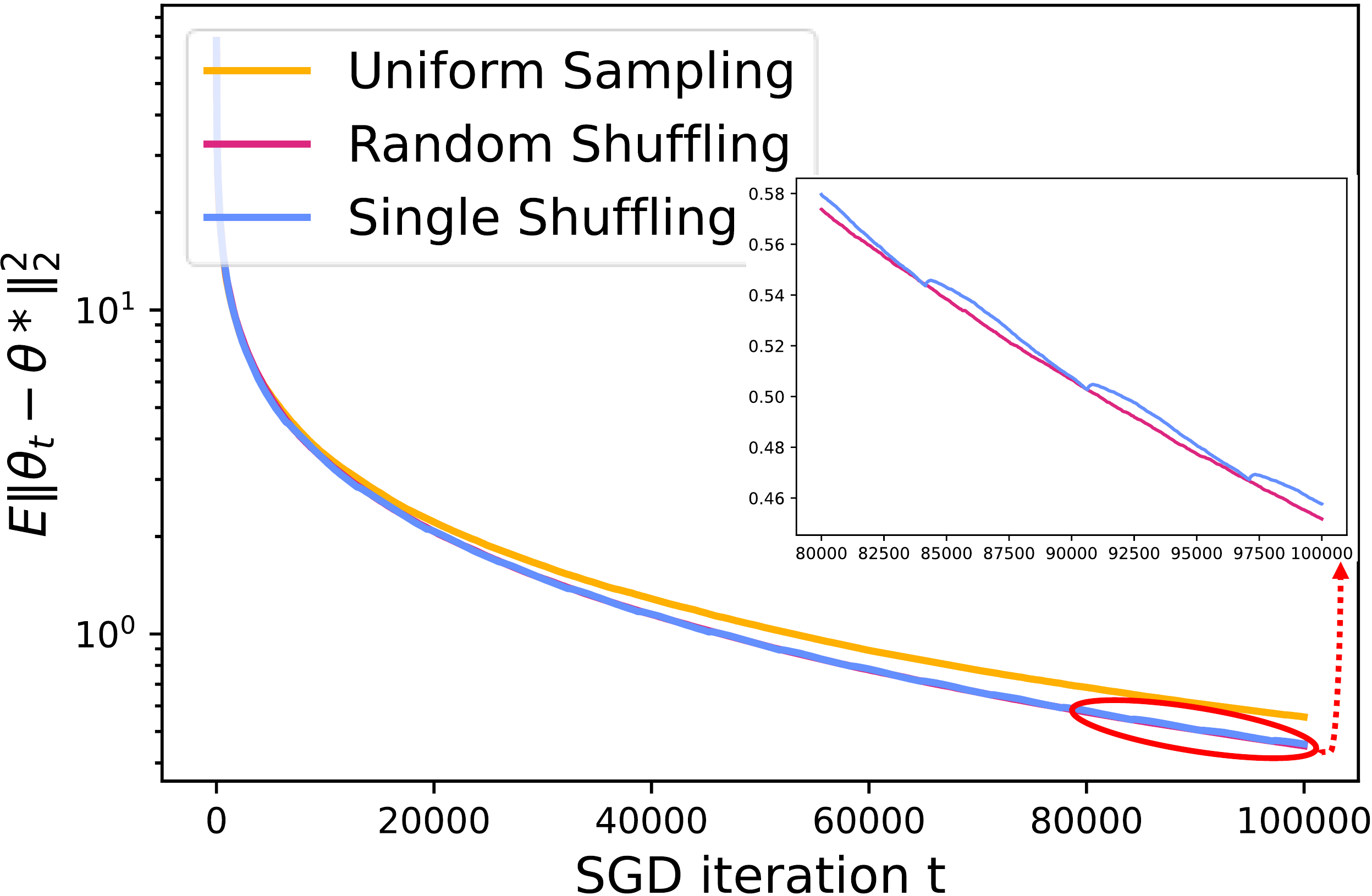}
         \caption{Vanilla SGD}
         \label{fig:9a}
     \end{subfigure}\quad
     \begin{subfigure}[b]{0.75\textwidth}
         \centering
         \includegraphics[width=\textwidth]{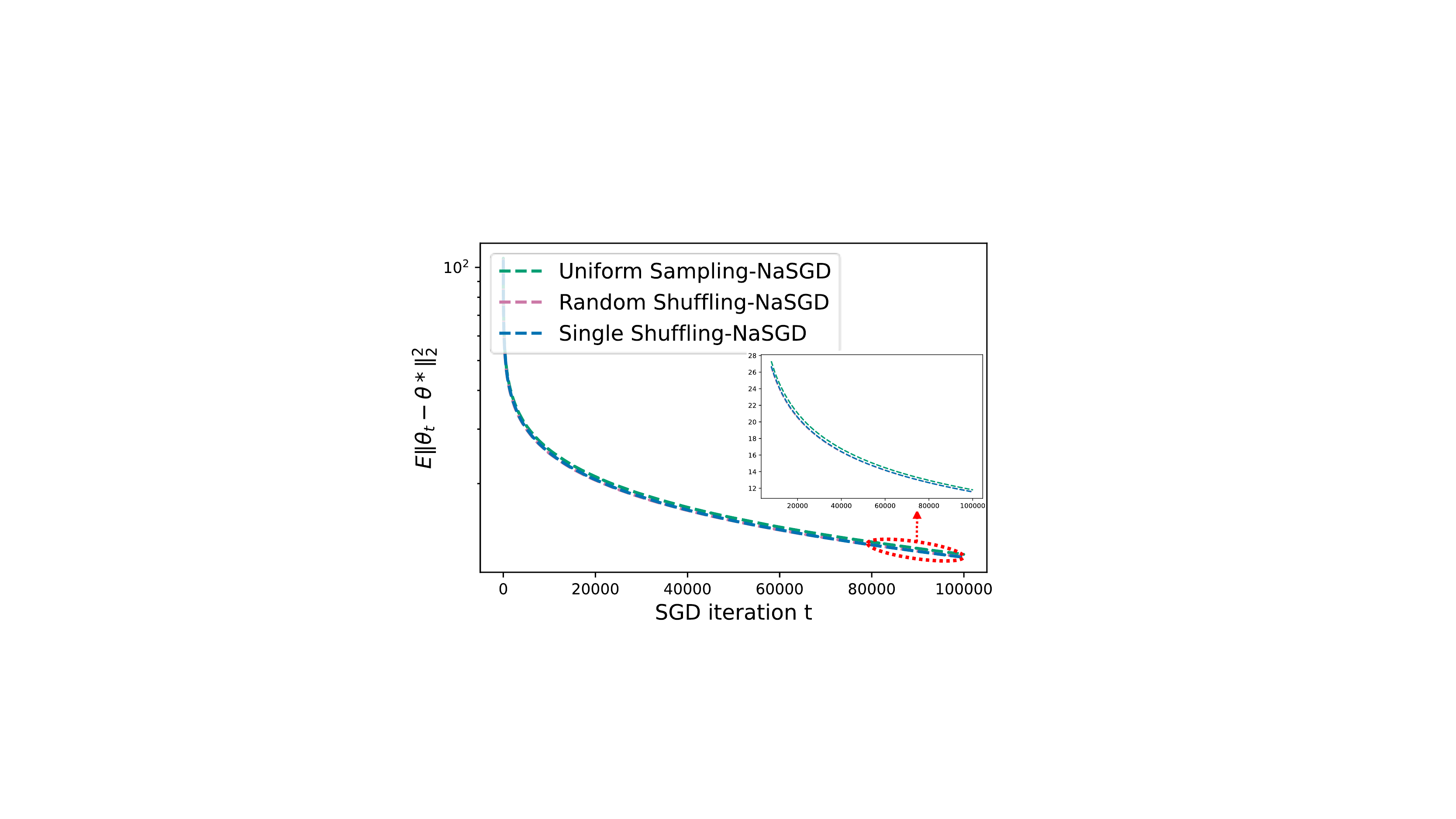}
         \caption{NaSGD}
         \label{fig:9b}
     \end{subfigure} 
     \vspace{-2mm}
     \begin{subfigure}[b]{0.8\textwidth}
         \centering
         \includegraphics[width=\textwidth]{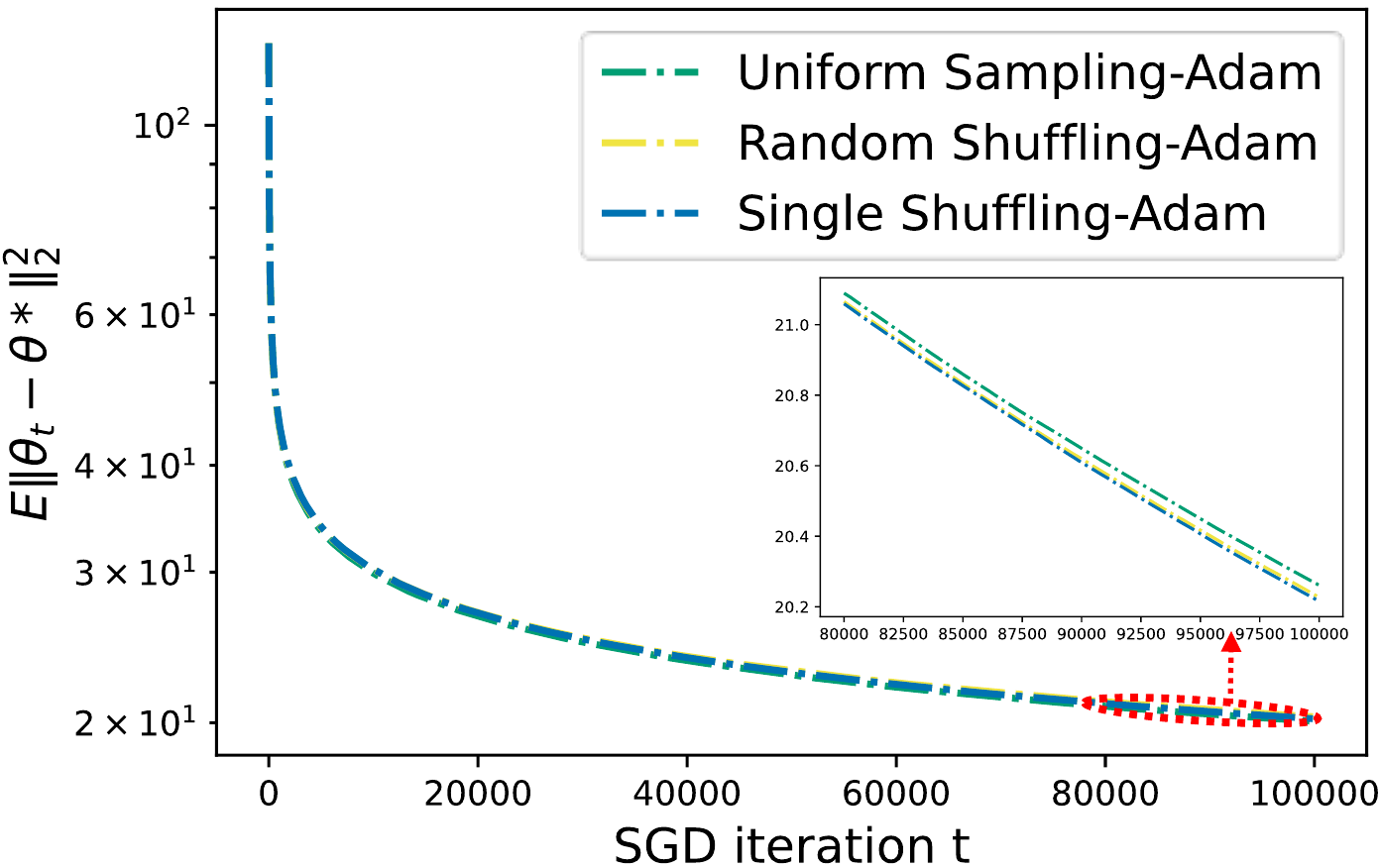}
         \caption{ADAM}
         \label{fig:9c}
     \end{subfigure}
     \caption{Performance comparison of shuffling methods and uniform sampling with non-convex objective function $\hat{g}(\theta)$ on the graph ``AS-733''.}
     \label{fig:9}
     \vspace{-2mm}
\end{figure}

\bibliographystyle{plain}
\bibliography{main.bib}

\end{document}